\documentclass{article}
\usepackage{iclr2027_conference,times}
\usepackage{amsmath,amsfonts,bm}
\usepackage{titletoc}
\usepackage{comment}

\def\eqref#1{(\ref{#1})}

\def\1{\bm{1}}

\DeclareMathAlphabet{\mathsfit}{\encodingdefault}{\sfdefault}{m}{sl}
\SetMathAlphabet{\mathsfit}{bold}{\encodingdefault}{\sfdefault}{bx}{n}

\DeclareMathOperator*{\argmax}{arg\,max}

\usepackage{amsmath,amssymb,amsthm,mathtools}
\usepackage{booktabs,multirow,longtable,array}
\usepackage{algorithm}
\usepackage{algpseudocode}
\usepackage{graphicx}
\usepackage{wrapfig}
\usepackage{xcolor}
\usepackage{tikz}
\usetikzlibrary{arrows.meta,calc,fit,positioning,shapes.geometric}
\usepackage{microtype}
\usepackage{placeins,float}
\usepackage{hyperref}
\usepackage{url}
\hypersetup{hidelinks,pdfauthor={Chuiyang Meng, Wenlu Yu, Ming Tang, Cheng Li},pdftitle={Social Circuits behind Multi-agent Echo Chambers}}
\usepackage{colortbl}
\usepackage{subcaption}

\definecolor{CGDshade}{HTML}{EDF5EF}

\newcommand{\scms}[2]{%
  \begingroup
  \renewcommand{\arraystretch}{0.85}%
  \begin{tabular}[c]{@{}c@{}}
    #1\\
    $\pm#2$
  \end{tabular}%
  \endgroup
}

\newtheorem{theorem}{Theorem}
\newtheorem{proposition}{Proposition}
\newtheorem{corollary}{Corollary}
\newtheorem{lemma}{Lemma}
\theoremstyle{definition}

\newcommand{\method}{\textsc{Social Circuits}}
\newcommand{\cgd}{\textsc{CGD}}

\title{\centering Social Circuits behind\\Multi-agent Echo Chambers}

\author{\begin{tabular}{c}
Chuiyang Meng$^{1}$ \quad Wenlu Yu$^{2}$ \quad Ming Tang$^{3}$ \quad Cheng Li$^{1}$ \\[3pt]
{\normalfont $^{1}$Simon Fraser University} \\
{\normalfont $^{2}$University of Alberta} \\
{\normalfont $^{3}$Southern University of Science and Technology} \\[3pt]
{\normalfont\href{mailto:chuiyang_meng@sfu.ca}{{\fontencoding{T1}\selectfont\texttt{chuiyang\_meng@sfu.ca}}} \quad \texttt{wenlu3@ualberta.ca}} \\
{\normalfont\texttt{tangm3@sustech.edu.cn} \quad \href{mailto:li_cheng@sfu.ca}{{\fontencoding{T1}\selectfont\texttt{li\_cheng@sfu.ca}}}}
\end{tabular}}

\iclrfinalcopy

\begin{document}
\maketitle
\lhead{} 

\begin{abstract}
Language-model agents exchange messages to combine evidence, but their communication can also create echo chambers that reinforce shared errors.
However, overall task performance does not explain how a message changes the receiving agent's internal activations and affects its decision.
In this work, we introduce \method{}, a framework for tracing message effects through receiver activations.
We compare the receiver's answers before and after changing a message.
Then, we restore selected activations recorded under the original message to determine how much of the message effect these activations reproduce.
Based on \method{}, we propose \emph{Circuit-Guided Deliberation} (CGD), which learns to select useful messages using receiver activation changes.
We establish when activation replacement preserves receiver decisions and bound the gap between CGD's task performance and the best achievable through message selection.
Experiments show that receiver activation changes explain the message effects and guide message selection that improves the task performance.
Across three models and four datasets, CGD achieves the highest or joint-highest average accuracy in our main comparisons while generating fewer tokens than multi-agent baselines.
\end{abstract}

\section{Introduction}

\begin{wrapfigure}{r}{0.55\textwidth} 
  \centering
  \includegraphics[width=0.55\textwidth]{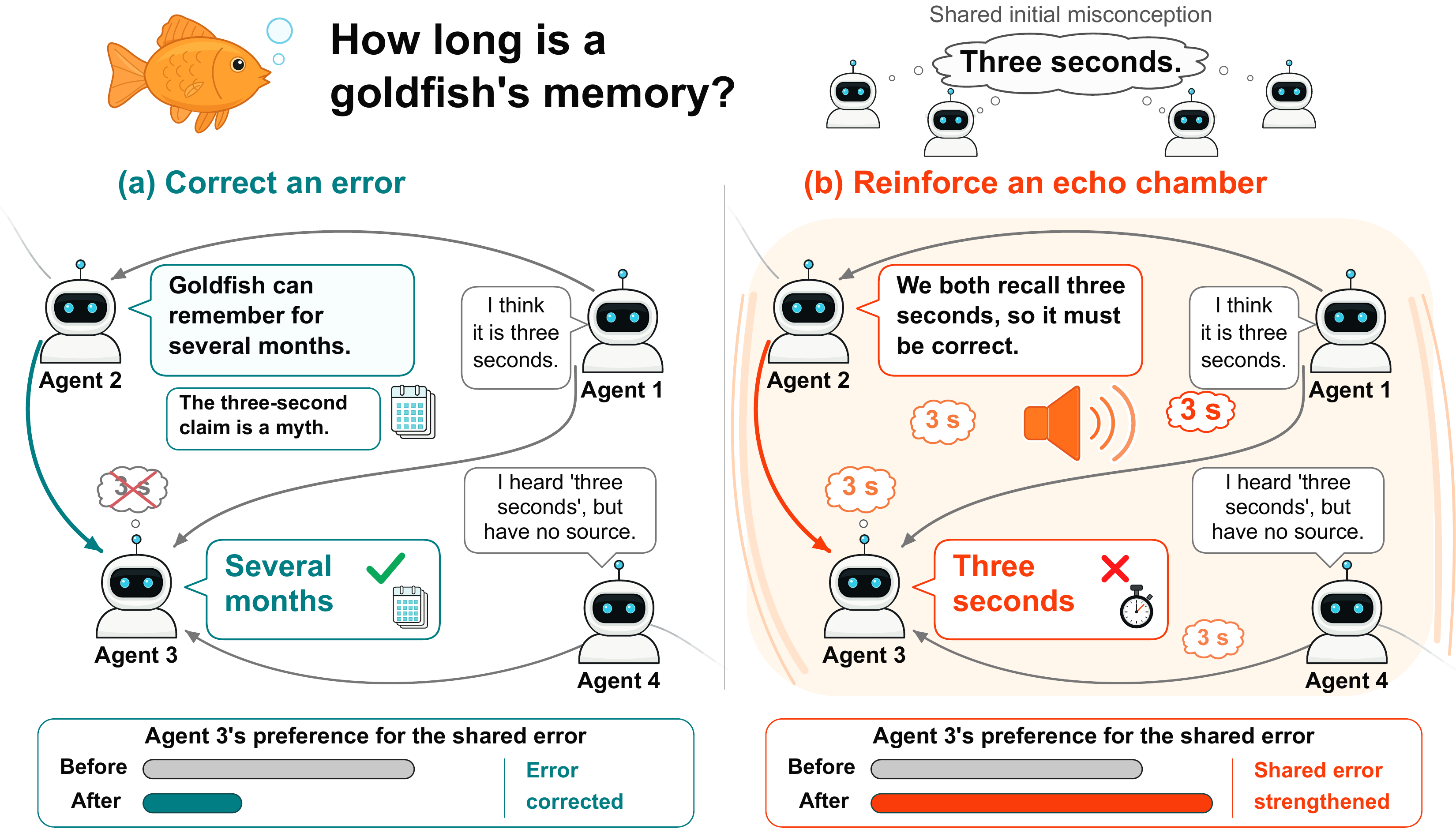} 
  \caption{{Comparison of the error correction and an echo chamber. True evidence corrects the receiver’s answer, whereas a message presenting the agents’ shared error as supporting evidence reinforces the echo chamber. }}
\label{fig:goldfish_example}
\end{wrapfigure}
Multi-agent systems enable language models to solve tasks collaboratively by exchanging messages that present evidence, reasoning, or proposed answers.
A message can correct a receiving agent's mistake, but it can also reinforce an error already shared by several agents.
We call the latter case an \emph{echo chamber}.
Fig.~\ref{fig:goldfish_example} illustrates that agreement can arise both when communication corrects errors and when it reinforces them.

Existing works show that debate can improve collective answers \citep{du2024multiagent}, while independent sampling and voting can achieve comparable performance under matched budgets \citep{smit2024mad,wang2023selfconsistency}.
However, agents can share errors before communication \citep{estornell2024multillm,kim2025correlated}.
Fig.~\ref{fig:motivating_example} shows that receiving more messages can either improve or degrade the task performance.
Comparing overall task performance alone does not explain how individual messages affect receivers \citep{lowe2019pitfalls}.
This motivates us to study the following question:

\textit{How do messages change a receiver's internal activations and decision, and can these activation changes help select messages that improve task performance?}

\par\noindent\begin{minipage}{\textwidth} \begin{wrapfigure}[16]{r}{0.38\textwidth} \centering \includegraphics[ width=\linewidth, trim=8bp 8bp 20bp 115bp, clip ]{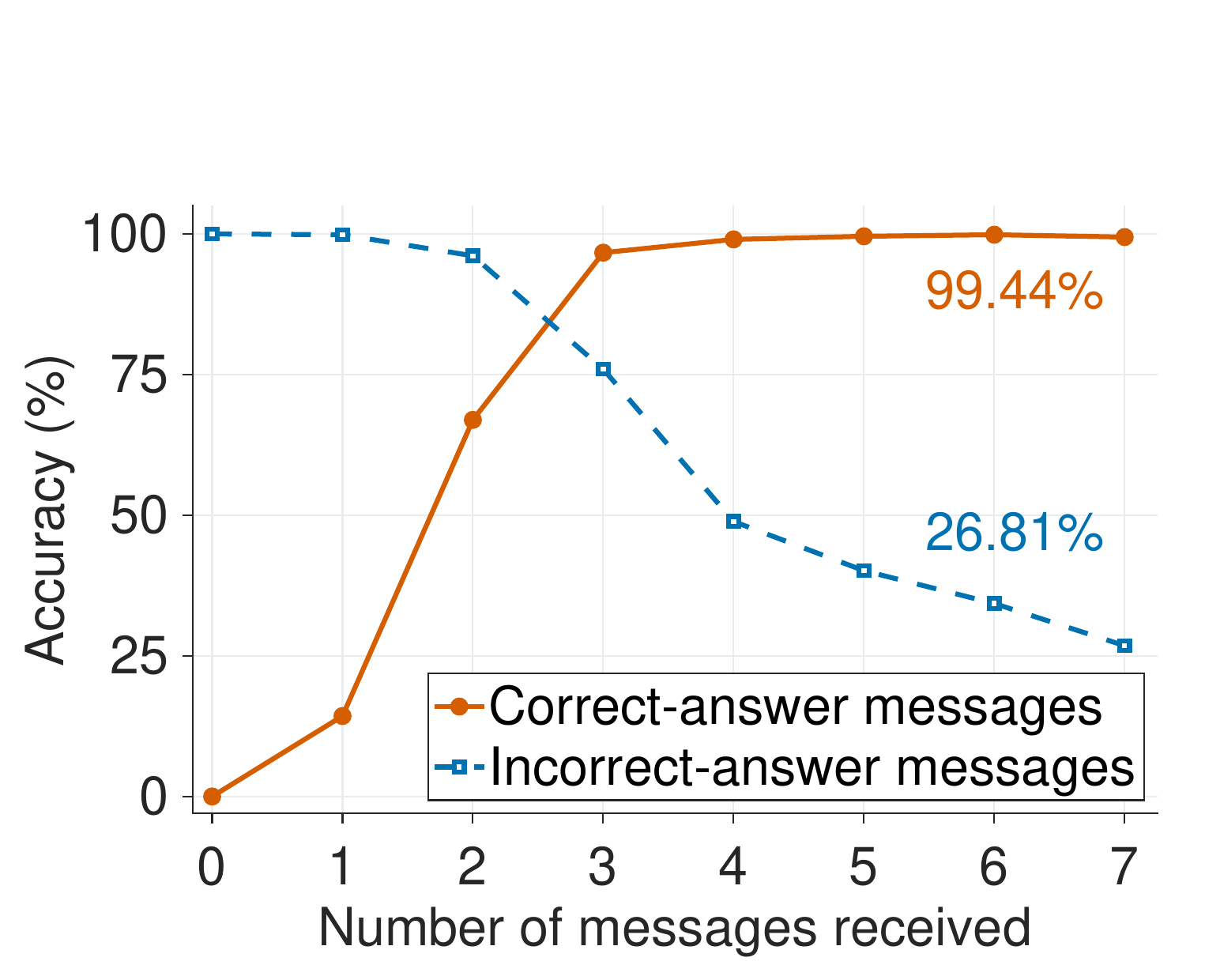} \caption{Receiver accuracy on Gemma-3-4B-it as more messages support the correct (orange) or incorrect (blue) answer.} \label{fig:motivating_example} \end{wrapfigure} 
To answer this question, we introduce \method{}, a framework for tracing message effects through receiver activations.
First, we compare the receiver's answers before and after removing a message or shuffling its tokens, with all other inputs fixed.
This \emph{matched rerun} measures the message's effect on answer scores and task performance.
Then, to trace this effect inside the receiver, we keep the message removed or shuffled and replace selected receiver activations with those recorded under the original message.
This \emph{activation replacement} measures the selected activations' contribution to the message-induced answer-score changes.

\medskip
Our observations show that receiver activation changes help distinguish how different messages affect the answers.
Based on these observations, we propose \emph{Circuit-Guided Deliberation} (\cgd{}), which uses these changes to learn which messages improve task performance.
CGD scores different sets of available messages, which are called \emph{message combinations}, by using sender activations and the receiver's activation changes.
The receiver then answers using the selected messages.
\par\end{minipage}\par

To sum up, our main contributions are as follows:

\begin{itemize}
\item We introduce \method{}, a framework for tracing how messages affect receiver decisions through internal activations.
It identifies receiver activation changes that reproduce part of the message's effect, whether the message corrects or reinforces a shared error.

\item Based on \method{}, we propose \emph{Circuit-Guided Deliberation} (\cgd{}).
By using the sender activations and the receiver's activation changes, \cgd{} learns to select message combinations that improve task performance.

\item Through theoretical analysis, we show the condition under which the activation replacement preserves receiver decisions without exactly reproducing answer scores.
We also bound CGD's expected performance gap against the optimal message combination and establish conditions under which minimizing its expected training loss closes this gap.

\item Experiments show that receiver activation changes help explain message effects and improve message selection.
Across three models and four datasets, \cgd{} achieves the highest or joint-highest average accuracy in our main comparisons while generating fewer tokens than multi-agent baselines.
On 2Wiki free generation, its accuracy exceeds the strongest baseline by 11.25\%--30.69\% under the same generation token limits.
\end{itemize}


\section{Related Work}

\paragraph{Multi-agent collaboration.}
Debate improves collective answers \citep{du2024multiagent}, while sampling and voting remain competitive \citep{wang2023selfconsistency,smit2024mad,choi2025debate}.
Agreement can reinforce shared errors \citep{kim2025correlated,estornell2024multillm}.
Other approaches model response dependencies \citep{ai2026beyond}, favor minority answers \citep{he2026minority}, or train with uncertainty \citep{tang2026variance}.
MAD-M2 masks unreliable memories from earlier rounds \citep{tian2026memorymasking}, while MADC orders roles by reasoning-path consistency \citep{zhang2026madc}.

\paragraph{Communication design and selection.}
G-Designer learns communication graphs \citep{zhang2025gdesigner}, while MOC combines multi-hop evidence under token constraints \citep{guan2026moc}.
SafeSieve prunes communication using semantic evaluation and performance feedback \citep{zhang2026safesieve}.
Shapley Message Value prunes communication through contribution estimates in cooperative multi-agent reinforcement learning \citep{ijcai2022p82}.

\paragraph{Causal analysis of communication.}
Text interventions reveal content effects \citep{lin2025isolated} and information propagation \citep{shen2025propagation}.
Comparisons with independent agents calibrate confidence \citep{huang2026cagecal}.
Restoring transmitted components identifies information affecting predictions \citep{zhang2026latent}.
Edge masking identifies compact communication subgraphs that preserve task performance \citep{li2026e2explainer}.

\paragraph{Internal mechanisms.}
Causal abstraction studies representations in causal models \citep{geiger2021causal,geiger2025causal}.
Mediation analysis traces input effects \citep{gultchin2021complex}, and sparse circuits identify behavior-relevant components \citep{marks2025sparse}.
Internal causal features predict output correctness \citep{huang2025internal}.
Agents communicate through activations \citep{ramesh2025activations} or state-change trajectories \citep{tang2025statedelta}.

\section{\method{}}
\label{sec:setup}

In this section, we introduce \method{}, a framework for tracing how a message from one agent affects another agent's answer.
Within \method{}, we first apply the \emph{message rerun} to measure the message's effect and \emph{activation replacement} to study whether selected activations in the receiving agent reproduce this effect.
Then, we evaluate whether the message corrects the receiving agent's answer or reinforces an error shared by the agents.
\method{} establishes a unified causal framework that connects the message exchange between agents, changes in receiver activations, and task performance.

We consider language-model agents that exchange text messages to solve a task.
Each agent can both send and receive messages.
For a given message, we refer to the agent sending it as the \emph{sender} and the agent receiving it as the \emph{receiver}.
We measure how the message affects the receiver's answer while keeping the task input and preceding conversation fixed.
Let $\mathcal O$ be a finite, nonempty set of candidate answers.
We denote the receiver's \emph{answer score} for the candidate answer $o\in\mathcal O$ by $F(o)$.
We use log probabilities as answer scores, calculated from the receiver's token probabilities conditioned on its input and preceding answer tokens.
Appendix~\ref{app:protocol-scoring} provides the detailed calculation of $F(o)$.
A higher score indicates a higher probability assigned to that answer.
The receiver selects $q=\argmax_{o\in\mathcal O}F(o)$.
We denote the corresponding \emph{task utility} as $u(q)$ to measure the performance under the task's evaluation rule, with higher values indicating better performance.

\subsection{Metrics for Measuring Message Effects}
We introduce metrics to characterize a message's effect on the receiver's answer scores and task performance.
We also quantify how much of the answer-score effect is reproduced by activation replacement.
We calculate these metrics using three runs as illustrated in Figure~\ref{fig:SC}: the \emph{original-message run}, the \emph{reference run}, and the \emph{activation-replacement run}.
We fix the task input, preceding conversation, and receiver model.
(i) In the \emph{original-message run}, the receiver reads the sender's original message, and we record its answer scores $F^{\rm msg}(o)$ and internal activations.
(ii) In the \emph{reference run}, we omit this message or shuffle its tokens and record the corresponding answer scores as $F^{\rm ref}(o)$.
By comparing these two runs, we characterize the effect of the original message against the reference message.

We use \emph{activation replacement} \citep{vig2020investigating,meng2022locating,zhang2024towards} to determine whether changing receiver activations alone can reproduce the message's effect on the answer scores.
For activation replacement, we select receiver layers at regular intervals and record the activation vectors of fixed tokens after the message.
Because the text of these tokens remains unchanged across runs, we can compare how the message changes their activation vectors.
The selected layers and token positions are provided in Appendix~\ref{app:protocol-replacement}.
(iii) In the \emph{activation-replacement run}, we keep the reference input unchanged and replace the selected receiver activation vectors with those recorded at the corresponding layers and tokens in the original-message run.
We record the receiver's answer scores as $F^{\rm act}(o)$.

\paragraph{Answer-score change.}
For a candidate answer $o\in\mathcal O$, we define the \emph{answer-score changes} caused by the original message and activation replacement as follows:
\begin{align}
\Delta F^{\rm msg}(o)
&=F^{\rm msg}(o)-F^{\rm ref}(o),\\
\Delta F^{\rm act}(o)
&=F^{\rm act}(o)-F^{\rm ref}(o).
\end{align}
In particular, $\Delta F^{\rm msg}(o)$ shows how much the original message affects the score of answer $o$.
$\Delta F^{\rm act}(o)$ shows how replacing the selected activations affects that score while keeping the reference message unchanged.
We compare $\Delta F^{\rm act}(o)$ with $\Delta F^{\rm msg}(o)$ to evaluate whether activation replacement reproduces the message's effect on the answer score.

\paragraph{Task-performance change.}
We measure changes in the receiver's task performance using the task utility $u(q)$ of its selected answer $q$.
For example, under the exact-match evaluation, $u(q)=1$ denotes a correct answer and $u(q)=0$ otherwise.
Let $q^{\rm ref}$, $q^{\rm msg}$, and $q^{\rm act}$ denote the answers selected in the reference, original-message, and activation-replacement runs, respectively.
We define the corresponding \emph{task-performance changes} as follows:
\begin{align}
\Delta u^{\rm msg}
&=u(q^{\rm msg})-u(q^{\rm ref}),\\
\Delta u^{\rm act}
&=u(q^{\rm act})-u(q^{\rm ref}).
\label{eq:decision-quality-change}
\end{align}
We compare $\Delta u^{\rm act}$ with $\Delta u^{\rm msg}$ to evaluate whether activation replacement reproduces the message's effect on the task performance.
Positive values indicate better performance than in the reference run, and negative values indicate worse performance.

\begin{figure}[t]
    \centering    \includegraphics[width=0.95\linewidth,trim=0bp 0bp 0bp 0bp,clip]{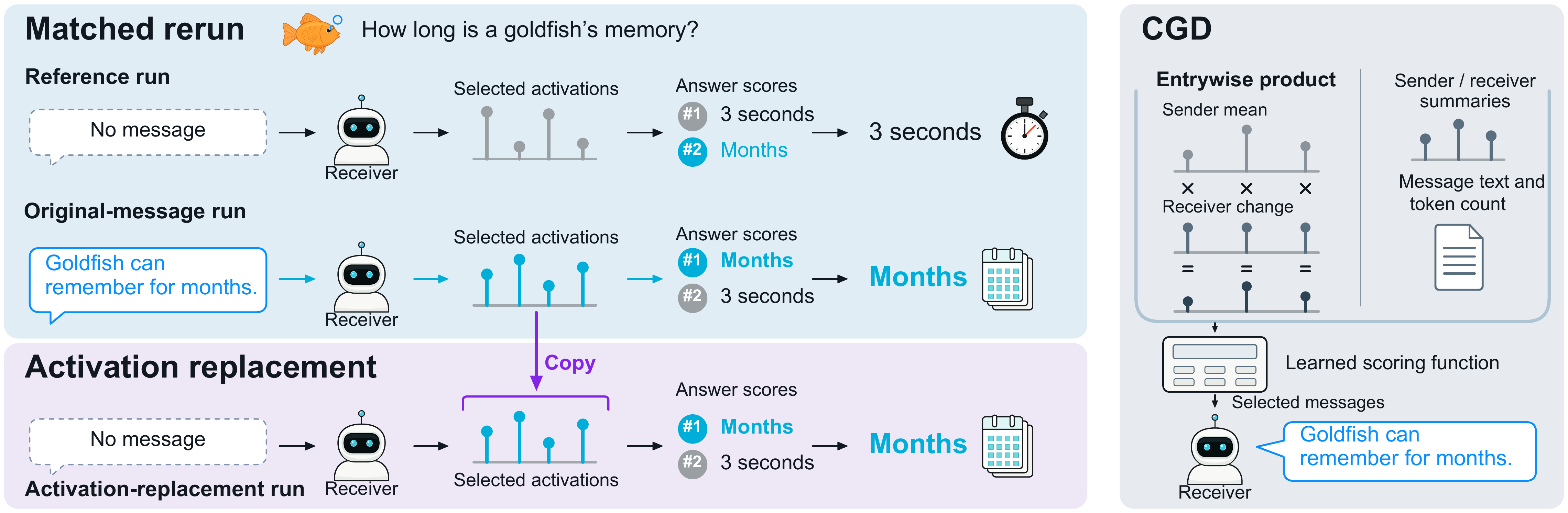}
    \caption{Illustration of \method{} and Circuit-Guided Deliberation (CGD). Matched reruns change the message, whereas activation replacement changes selected receiver activations while preserving the reference input. CGD uses sender summaries and receiver activation changes to select message combinations.}
    \label{fig:SC}
\end{figure}

\paragraph{Message-effect recovery.}
To measure whether a message supports the correct answer or a shared error, we compare the correct and incorrect answers specified by the task.
For correction measurements, $o_1\in\mathcal O$ denotes the correct answer and $o_2\in\mathcal O$ the incorrect answer.
For shared-error measurements, we reverse this order.
The same answers and order are used in all three runs.
We calculate the changes in their score difference as follows:
\begin{equation}
\begin{aligned}
\Delta^{\rm msg}
&=\Delta F^{\rm msg}(o_1)-\Delta F^{\rm msg}(o_2),\\
\Delta^{\rm act}
&=\Delta F^{\rm act}(o_1)-\Delta F^{\rm act}(o_2).
\end{aligned}
\label{eq:direct-path-effects}
\end{equation}
When $o_1$ is an incorrect answer which is favored by several agents, a positive value of $\Delta^{\rm msg}$ indicates reinforcement of their \emph{echo chamber}.
To quantify how much of the message effect is reproduced through receiver activations, we define the \emph{message-effect recovery ratio} as $\rho=
\frac{\mathbb E[\Delta^{\rm act}]}
{\mathbb E[\Delta^{\rm msg}]}$, where $\mathbb E$ averages over the same tasks and random seeds for message generation.

\subsection{Key Observations}
Based on the metrics we discussed above, we present the results in Fig.~\ref{fig:para_2} to show how messages affect receivers' answers and whether these effects can be traced through receiver activations.
We present two observations as follows.
\begin{figure}[t]
\centering
\begin{minipage}[t]{0.3\textwidth}
\centering
\includegraphics[width=\linewidth]{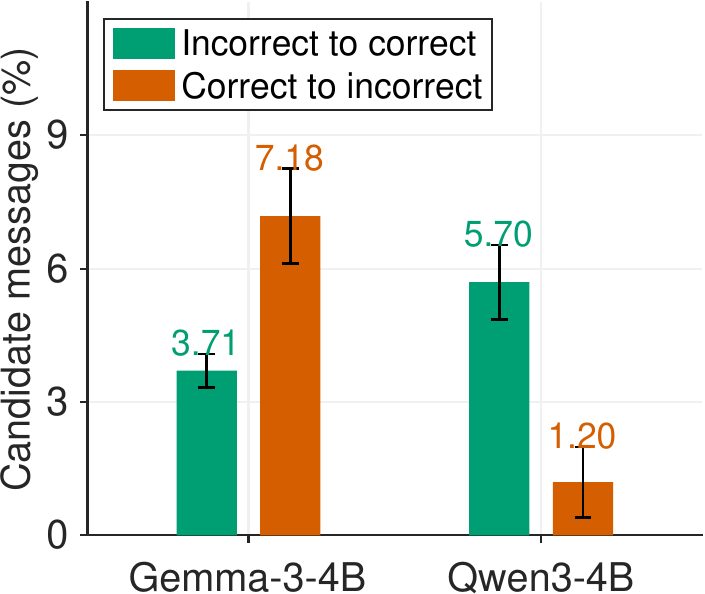}\\
{\footnotesize (a)}
\end{minipage}\hfill
\begin{minipage}[t]{0.3\textwidth}
\centering
\includegraphics[width=\linewidth]{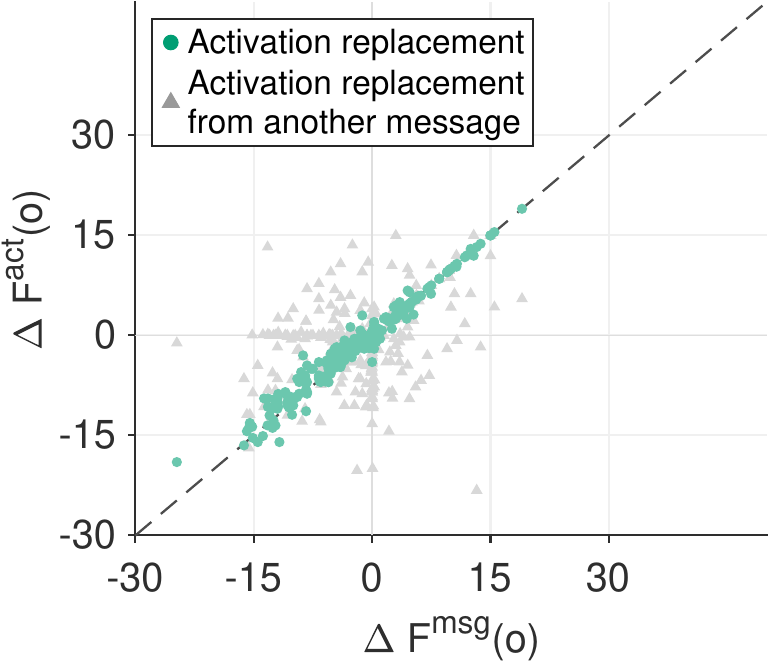}\\
{\footnotesize (b)}
\end{minipage}\hfill
\begin{minipage}[t]{0.3\textwidth}
\centering
\includegraphics[width=\linewidth]{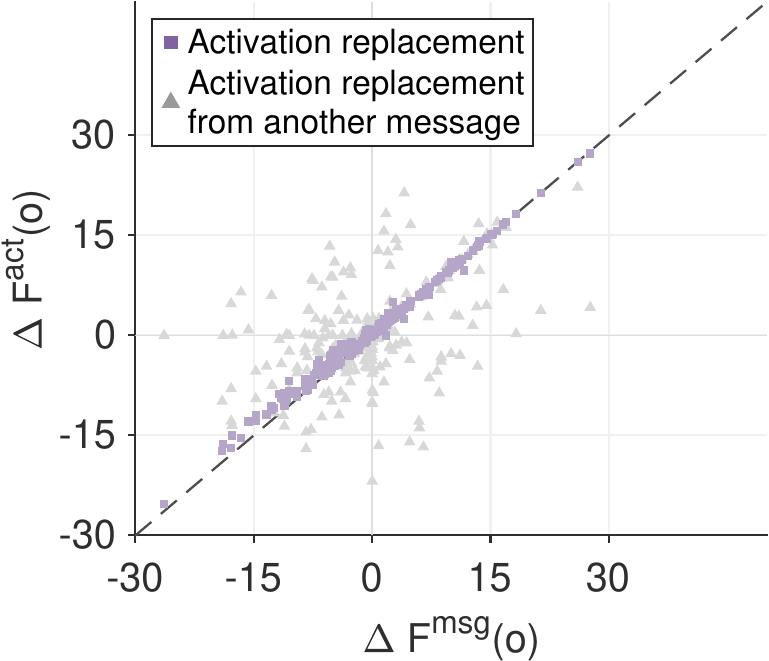}\\
{\footnotesize (c)}
\end{minipage}\\
\begin{minipage}[t]{0.3\textwidth}
\centering
\includegraphics[width=\linewidth]{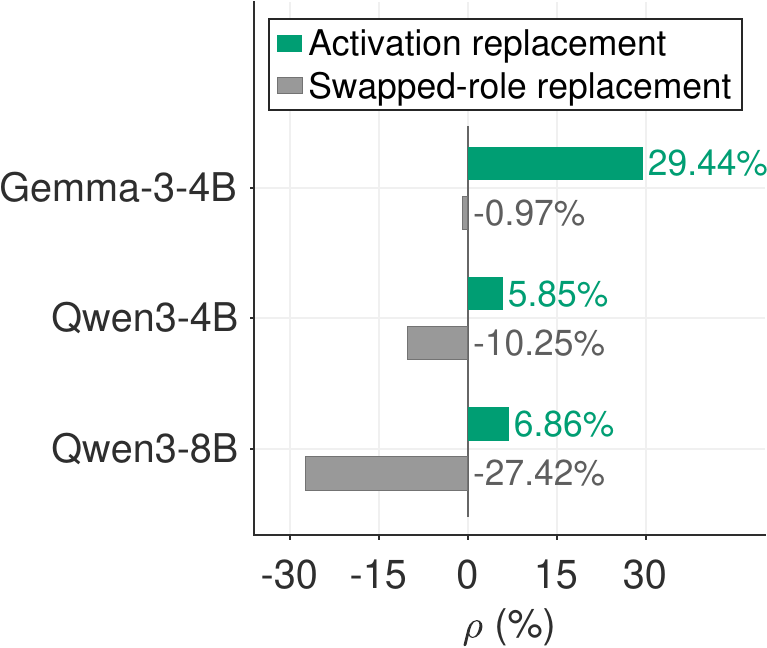}\\
{\footnotesize (d)}
\end{minipage}\hfill
\begin{minipage}[t]{0.3\textwidth}
\centering
\includegraphics[width=\linewidth]{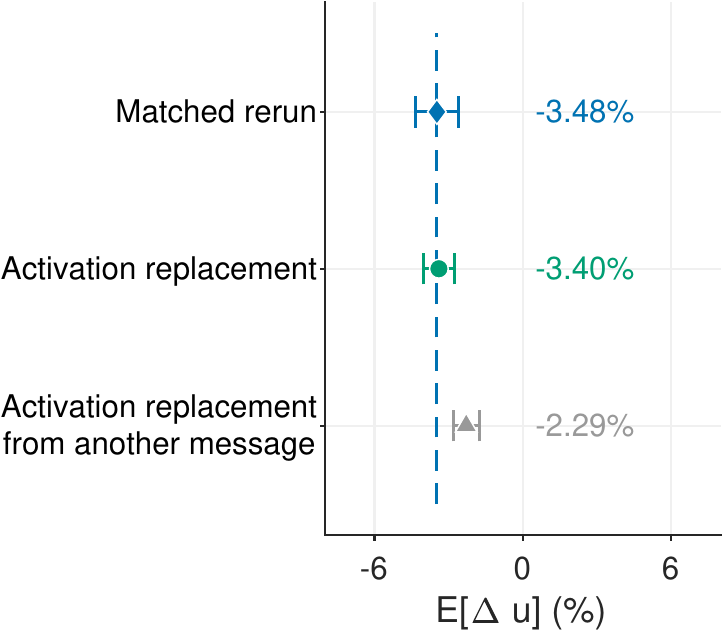}\\
{\footnotesize (e)}
\end{minipage}\hfill
\begin{minipage}[t]{0.3\textwidth}
\centering
\includegraphics[width=\linewidth]{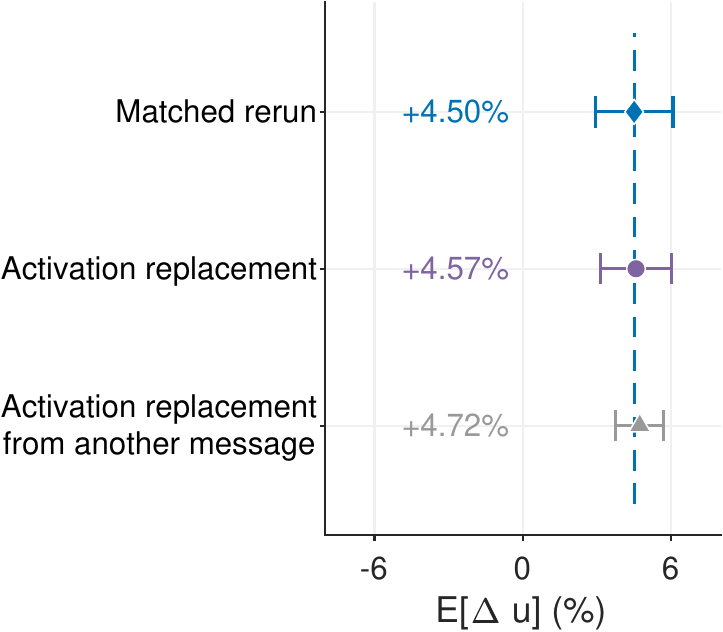}\\
{\footnotesize (f)}
\end{minipage}
\caption{{(a) Fractions of messages that correct or introduce errors for Gemma-3-4B and Qwen3-4B. (b) Answer-score changes for Gemma-3-4B. (c) Answer-score changes for Qwen3-4B. (d) Message-effect recovery $\rho$ across three models. (e) Average task-performance changes for Gemma-3-4B. (f) Average task-performance changes for Qwen3-4B. Diagonal dashed lines in (b) and (c) mark equal answer-score changes. 
Figs. (a)--(c), (e), and (f) use the same candidate messages within each model. Task performance is the accuracy of the higher-scoring answer, and error bars show standard deviations across random seeds.}}
\label{fig:para_2}
\end{figure}

\paragraph{Receiver activation changes are specific to each message.}
We present the fractions of messages that correct a receiver's answer or change a correct answer into an incorrect one in Fig.~\ref{fig:para_2}(a).
Results show that both changes occur in each model.
For example, in Qwen3-4B, 5.70\% of messages correct an answer, while 1.20\% introduce an error.
Thus, communication can improve average task performance even though some messages make the receiver's answer incorrect.
To trace these message effects through receiver activations, we present the answer-score changes in Figs.~\ref{fig:para_2}(b) and (c).
Results show that activation replacement closely reproduces the answer-score changes caused by the same message.
Using activation changes induced by another message increases the mean absolute error from 0.187 to 1.280 in Gemma-3-4B and from 0.121 to 1.039 in Qwen3-4B.
These results show that receiver activation changes help distinguish how different messages affect answer scores.

\paragraph{Activation replacement reproduces the average task performance changes.}
We further show that these receiver activation changes can reproduce message effects on receiver answers.
Fig.~\ref{fig:para_2}(d) presents the message-effect recovery across three models on a synthetic task with messages supporting correct answers or shared errors.
For comparison, we exchange sender and receiver roles in a separate run.
We apply the changes recorded in the original sender to the original receiver, which we call \emph{swapped-role replacement}.
Changes from the original sender or another message are rescaled to match the norm of the original receiver's changes.
Receiver activation replacement reproduces the direction of the average message effect across all three models, while the swapped-role replacement reverses this direction.
However, a change in answer scores does not always change the receiver's answer.
Therefore, we compare the task performance in Figs.~\ref{fig:para_2}(e) and (f).
Activation replacement closely matches the average performance decrease in Gemma-3-4B and the increase in Qwen3-4B.
For individual answers, Table~\ref{tab:activation-answer-recovery} in Appendix~\ref{sec:additional-activation} shows that activation replacement reproduces over 96\% of message-induced answer changes on 2Wiki in both models.
Using changes from another message recovers 41.07\% in Gemma-3-4B and 55.06\% in Qwen3-4B.
Thus, the receiver changes traced by \method{} can reproduce both message-induced error correction and error introduction.

\paragraph{Remark.}
The aforementioned observations demonstrate that \method{} uses receiver activation changes to trace how messages correct answers or introduce errors.
These changes distinguish the effects of different messages.
This motivates using receiver activation changes to learn which messages improve task performance.

\subsection{Theoretical Analysis}
\label{sec:sc-theory}

The aforementioned observations show that receiver activation replacement can reproduce answer changes caused by messages.
We establish sufficient conditions under which activation replacement preserves the receiver's answer and the message's effect on task performance.
Since the receiver selects the highest-scoring answer, we characterize these conditions through the score differences between candidate answers.
For a task with at least two candidate answers, we compare the score difference for each answer pair in the original-message and activation-replacement runs.
We define the \emph{recovery error} $\mathrm{err}$ as the largest absolute change in the answer score difference between the original-message and activation-replacement runs, as follows:
\begin{equation}
\mathrm{err}
=\max_{\substack{o_1,o_2\in\mathcal O\\o_1\ne o_2}}
\left|
\bigl[F^{\rm msg}(o_1)-F^{\rm msg}(o_2)\bigr]
-\bigl[F^{\rm act}(o_1)-F^{\rm act}(o_2)\bigr]
\right|.
\label{eq:activation-contrast-error}
\end{equation}
For activation replacement to change the receiver's answer from $q^{\rm msg}$, at least one answer $o\in\mathcal O\setminus\{q^{\rm msg}\}$ must satisfy $F^{\rm act}(o)\ge F^{\rm act}(q^{\rm msg})$.
We denote the smallest score difference between $q^{\rm msg}$ and the remaining answers in the original-message run by $\mathrm{gap}$, which is defined as follows:
\begin{equation}
\mathrm{gap}
=F^{\rm msg}(q^{\rm msg})
-\max_{o\in\mathcal O\setminus\{q^{\rm msg}\}}F^{\rm msg}(o).
\label{eq:message-answer-gap}
\end{equation}
The following theorem provides conditions for activation replacement to preserve the receiver's answer and the direction of the message's effect on the score difference between $o_1$ and $o_2$.

\begin{theorem}[Preserving answers and message effects]
\label{thm:sequential-social-circuit}
If $\mathrm{err}<\mathrm{gap}$, activation replacement preserves the answer and task-performance change of the original-message run, i.e., $q^{\rm act}=q^{\rm msg}$ and $\Delta u^{\rm act}=\Delta u^{\rm msg}$.
If $|\Delta^{\rm msg}|>\mathrm{err}$, activation replacement and the original message change the score difference between $o_1$ and $o_2$ in the same direction, i.e., $\Delta^{\rm act}\Delta^{\rm msg}>0$.
\end{theorem}

The proof is provided in Appendix~\ref{app:proof-sequential-social-circuit}.
Theorem~\ref{thm:sequential-social-circuit} shows that \method{} does not need to recover each answer score exactly to recover a message's effect on the task performance.
If activation replacement keeps the same highest-scoring answer, the receiver's task performance remains unchanged.
In addition, the original message and activation changes can both strengthen or both weaken a shared error.

\section{Circuit-Guided Deliberation}
\label{sec:method}
In this section, we propose \emph{Circuit-Guided Deliberation} (\cgd{}) to select messages for the receiver's final answer.
\method{} traces message effects through receiver activation changes.
We use these changes and sender information to learn which messages improve task performance.
We first describe the scoring function, followed by training, message selection, and theoretical guarantees.

\subsection{Scoring Function}
Each sender forms a candidate message.
Let $\mathcal C$ denote the set of these candidate messages.
Let $\mathcal A\subseteq\mathcal C$ denote a \emph{message combination}.
It specifies which messages the receiver uses for its final answer.
We denote the set of all combinations by $\mathcal S=\{\mathcal A\mid\mathcal A\subseteq\mathcal C\}$.
$\mathcal S$ includes the empty combination $\varnothing$, corresponding to a final answer without candidate messages.
For each sender, we record activations after it reads the input and its own evidence.
For the receiver, we record activations after it reads each message combination $\mathcal A\in\mathcal S$.
In both cases, we use the output vector at the last input token of each selected layer before output generation.

CGD uses the recorded activations to represent sender information and receiver changes for scoring messages.
In particular, to reduce the input dimension without learning additional parameters, we calculate $d$ weighted sums of each activation vector and average these $d$-dimensional vectors over the selected layers.
We refer to this average as an \emph{activation summary}.
The weights are randomly initialized, shared across agents using the same language model and its selected layers, and kept fixed during training.
Let $\mathbf{r}_{\mathcal A}$ denote the receiver's activation summary after reading $\mathcal A$, with $\mathbf{r}_{\varnothing}$ corresponding to no candidate messages.
For any nonempty set $\mathcal A$, let $\mathbf{s}_{\mathcal A}$ denote the average activation summary of the senders whose messages belong to $\mathcal A$.
We represent the receiver changes caused by the messages as $\mathbf r_{\mathcal A}-\mathbf r_{\varnothing}$.
We combine these changes with sender information using $\mathbf s_{\mathcal A}\odot(\mathbf r_{\mathcal A}-\mathbf r_{\varnothing})$, where $\odot$ denotes elementwise multiplication.
This product lets CGD score the same receiver change differently depending on sender information.
The scoring input $\mathbf x_{\mathcal A}$ includes the message text vector and token count, $\mathbf s_{\mathcal A}$, $\mathbf r_{\varnothing}$, and the aforementioned product.
To compare $\mathcal A$ with other message combinations, we also include the average receiver summary over the other nonempty combinations.
We introduce a shared scoring function $f_{\boldsymbol{\theta}}$ that calculates a weighted sum of $\mathbf{x}_{\mathcal A}$ with a bias.
We denote the learnable parameters as $\boldsymbol{\theta}$, which include weights and bias.
Adding the same bias to all nonempty combinations leaves their ranking unchanged.
It affects whether their highest score exceeds zero, which determines whether CGD uses messages.
The score for selecting a message combination $\mathcal A$ is defined as follows:
\begin{equation}
S(\mathcal A)=f_{\boldsymbol{\theta}}(\mathbf{x}_{\mathcal A})
\quad\text{for }\mathcal A\ne\varnothing,
\quad
S(\varnothing)=0.
\label{eq:cgd-ranking-score}
\end{equation}

\subsection{Message Selection}\label{sec:cgd-message-selection}
We train $\boldsymbol{\theta}$ to assign higher scores to message combinations that improve task performance.
For each training task, we run the receiver separately with each combination $\mathcal A\in\mathcal S$ and denote its answer as $q_{\mathcal A}$.
By using the task score $u$ defined in Section~\ref{sec:setup}, we define the gain from using $\mathcal A$ instead of using no messages as $V(\mathcal A)=u(q_{\mathcal A})-u(q_{\varnothing})$,
where $q_{\varnothing}$ is the receiver's answer without using any message.
To rank message combinations by their gains, we train $f_{\boldsymbol{\theta}}$ using the pairwise comparisons.
For each training task, we consider two combinations $\mathcal A,\mathcal B\in\mathcal S$ with $V(\mathcal A)>V(\mathcal B)$.
We denote its pairwise logistic loss \citep{burges2005learning} as
\begin{equation}
\ell(\mathcal A,\mathcal B)
=\bigl(V(\mathcal A)-V(\mathcal B)\bigr)
\log\!\left(1+\exp\!\left[S(\mathcal B)-S(\mathcal A)\right]\right).
\label{eq:cgd-pairwise-loss}
\end{equation}
If there is a larger difference in task scores between $\mathcal B$ and $\mathcal A$, this loss increases the penalty for ranking $\mathcal B$ above $\mathcal A$.
We update $\boldsymbol{\theta}$ by minimizing the summed pairwise loss and leaving the bias unpenalized.
During implementation, CGD keeps $\boldsymbol{\theta}$ fixed and selects the highest-scoring message combination as $\widehat{\mathcal A}=\argmax_{\mathcal A\in\mathcal S}S(\mathcal A)$.
The receiver then answers using the messages in $\widehat{\mathcal A}$.

\subsection{Message-Selection Guarantees}
\label{sec:cgd-correlated-selection}

Minimizing a pairwise loss does not always select the message combination with the highest expected task utility \citep{duchi2010consistency,dembczynski2012consistent}.
Hence, we bound CGD's loss in expected task utility and show when pairwise training selects an optimal combination.
In this analysis, we use $u(q)=1$ for task success and $0$ otherwise.
With the available messages, scoring inputs, and learned scoring function fixed, let $\mathcal P$ denote the joint distribution of receiver answers and task utilities across message combinations and activation-replacement runs.
We denote expectation and probability under this distribution by $\mathbb E_{\mathcal P}$ and $\Pr_{\mathcal P}$, respectively.
We compare CGD's selection $\widehat{\mathcal A}$ with an optimal combination $\mathcal A^\star\in\argmax_{\mathcal A\in\mathcal S}\mathbb E_{\mathcal P}[u(q_{\mathcal A})]$.
The expected task-utility difference between the optimal message
combination and CGD's selection is
$\mathbb E_{\mathcal P}[V(\mathcal A^\star)-V(\widehat{\mathcal A})]$.

If activation replacement preserves receiver answers, the task-utility differences between message combinations remain unchanged.
We replace the receiver's selected activations in the no-message run with those recorded after it reads $\mathcal A$.
We denote the receiver's answer after replacement by $q_{\mathcal A}^{\mathrm{act}}$.
The corresponding utility is $u^{\rm act}_{\mathcal A}=u(q^{\rm act}_{\mathcal A})$, with $q^{\rm act}_{\varnothing}=q_{\varnothing}$.
Thus, $\Pr_{\mathcal P}(q^{\rm act}_{\mathcal A}\ne q_{\mathcal A})$ is the probability that activation replacement and the run using $\mathcal A$ return different answers.
We denote the covariance between the task utilities of different message combinations by $\operatorname{Cov}_{\mathcal P}$.

Based on $\ell$ defined in Eq.~\eqref{eq:cgd-pairwise-loss}, we define the expected pairwise loss as $\mathcal L(S)=\mathbb E_{\mathcal P}[\sum_{\mathcal A,\mathcal B\in\mathcal S:\,V(\mathcal A)>V(\mathcal B)}\ell(\mathcal A,\mathcal B)]$.
Let $\mathcal L^\star$ denote the infimum of $\mathcal L(S)$ over all real-valued message-combination scores.
Then, we bound the expected task-utility difference between the optimal message combination and CGD's selection in the following theorem.

\begin{theorem}[CGD's Expected Task-Utility Difference] \label{thm:cgd-recovery-ranking} For $|\mathcal S|\ge2$, any joint distribution $\mathcal P$, and any finite learned scores, the expected task-utility difference between $\mathcal A^\star$ and $\widehat{\mathcal A}$ is bounded by $\mathbb E_{\mathcal P}[V(\mathcal A^\star)-V(\widehat{\mathcal A})] \le \sum_{\mathcal B\in\mathcal S\setminus\{\mathcal A^\star,\widehat{\mathcal A}\}} \max\{0,\operatorname{Cov}_{\mathcal P}(u^{\rm act}_{\widehat{\mathcal A}}-u^{\rm act}_{\mathcal A^\star},u^{\rm act}_{\mathcal B})\} + \sqrt{|\mathcal S|(\mathcal L(S)-\mathcal L^\star)} +4(|\mathcal S|-2)\max_{\mathcal A\in\mathcal S}\Pr_{\mathcal P}(q^{\rm act}_{\mathcal A}\ne q_{\mathcal A})$. \end{theorem}

Theorem~\ref{thm:cgd-recovery-ranking} bounds the expected
task-utility difference for CGD's learned scoring function.
If activation replacement preserves receiver answers and task utilities
after replacement have equal covariance across distinct combinations,
the first and third terms vanish.
Under these conditions, reducing $\mathcal L(S)-\mathcal L^\star$
tightens the bound on CGD's expected task-utility difference.
The proof is provided in Appendix~\ref{app:cgd-recovery-transfer}.

\section{Experiments}
\label{sec:experiments}

\begin{table}[t]
\centering
\caption{ Accuracy (\%, mean $\pm$ standard deviation over three seeds) and mean tokens per task for generated messages and final answers under the original generation settings. 
On 2Wiki A/B, CGD sends passages and responses, and ollaboration baselines send responses alone. }
\label{tab:four-dataset-accuracy-tokens}
\providecolor{CGDshade}{HTML}{EDF5EF}
\providecommand{\scms}[2]{}
\renewcommand{\scms}[2]{#1{\fontsize{5}{5}\selectfont$\pm#2$}}
\fontsize{7}{8}\selectfont
\setlength{\tabcolsep}{1pt}
\renewcommand{\arraystretch}{1.10}
\begin{tabular*}{\linewidth}{@{\extracolsep{\fill}}ll*{8}{c}>{\columncolor{CGDshade}[\tabcolsep][0pt]}c@{}}
\toprule
 & & \multicolumn{2}{c}{Single-agent} & \multicolumn{7}{c}{Multi-agent} \\
\cmidrule(lr){3-4}\cmidrule(lr){5-11}
Model & Metric & Single & SC6 & Majority & MAD & MAD-M2 & MADC & \shortstack{MOC} & SafeSieve & \textbf{CGD} \\
\midrule
\multicolumn{11}{l}{\textbf{2Wiki (A/B answers)}} \\
\multirow{2}{*}{Gemma-3-4B} & Accuracy $\uparrow$ & \scms{78.89}{0.48} & \scms{77.78}{1.46} & \scms{83.19}{0.64} & \scms{80.69}{1.05} & \scms{82.92}{1.50} & \scms{79.44}{1.27} & \scms{86.94}{0.96} & \scms{79.72}{3.07} & \scms{\textbf{89.72}}{1.46} \\
 & Tokens $\downarrow$ & 79 & 478 & 219 & 647 & 516 & 638 & 401 & 485 & 139 \\
\multirow{2}{*}{Qwen3-4B} & Accuracy $\uparrow$ & \scms{73.19}{1.73} & \scms{70.83}{2.32} & \scms{73.61}{1.68} & \scms{80.42}{2.08} & \scms{82.08}{1.44} & \scms{77.22}{1.27} & \scms{86.94}{0.64} & \scms{83.89}{1.05} & \scms{\textbf{89.86}}{1.68} \\
 & Tokens $\downarrow$ & 171 & 1025 & 468 & 978 & 1028 & 961 & 864 & 874 & 299 \\
\multirow{2}{*}{Qwen3-8B} & Accuracy $\uparrow$ & \scms{75.69}{0.48} & \scms{72.92}{0.42} & \scms{80.97}{2.55} & \scms{81.81}{2.71} & \scms{78.47}{2.77} & \scms{78.61}{0.24} & \scms{85.83}{0.42} & \scms{81.94}{1.97} & \scms{\textbf{87.36}}{0.64} \\
 & Tokens $\downarrow$ & 156 & 930 & 439 & 972 & 900 & 959 & 881 & 1004 & 286 \\
\midrule
\multicolumn{11}{l}{\textbf{GSM8K (free generation)}} \\
\multirow{2}{*}{Gemma-3-4B} & Accuracy $\uparrow$ & \scms{84.86}{1.20} & \scms{88.19}{1.05} & \scms{86.81}{0.64} & \scms{86.53}{1.73} & \scms{88.61}{1.34} & \scms{86.53}{1.73} & \scms{\textbf{88.75}}{1.50} & \scms{87.64}{0.24} & \scms{\textbf{88.75}}{0.48} \\
 & Tokens $\downarrow$ & 224 & 1329 & 678 & 1386 & 1248 & 1386 & 1211 & 1139 & 656 \\
\multirow{2}{*}{Qwen3-4B} & Accuracy $\uparrow$ & \scms{89.86}{0.87} & \scms{91.53}{0.24} & \scms{89.72}{0.64} & \scms{90.83}{0.72} & \scms{91.11}{1.68} & \scms{90.83}{0.72} & \scms{91.11}{0.64} & \scms{92.92}{0.42} & \scms{\textbf{93.33}}{1.10} \\
 & Tokens $\downarrow$ & 221 & 1320 & 656 & 1291 & 1226 & 1291 & 1144 & 1268 & 640 \\
\multirow{2}{*}{Qwen3-8B} & Accuracy $\uparrow$ & \scms{91.81}{1.46} & \scms{92.92}{0.83} & \scms{92.78}{1.05} & \scms{\textbf{93.89}}{1.73} & \scms{\textbf{93.89}}{1.05} & \scms{\textbf{93.89}}{1.73} & \scms{93.75}{0.42} & \scms{92.64}{0.48} & \scms{\textbf{93.89}}{0.64} \\
 & Tokens $\downarrow$ & 229 & 1383 & 688 & 1427 & 1325 & 1427 & 1212 & 1389 & 677 \\
\midrule
\multicolumn{11}{l}{\textbf{MATH-500 (A/B answers)}} \\
\multirow{2}{*}{Gemma-3-4B} & Accuracy $\uparrow$ & \scms{42.67}{0.76} & \scms{52.00}{1.00} & \scms{46.33}{1.04} & \scms{52.83}{0.58} & \scms{53.83}{2.02} & \scms{52.83}{0.58} & \scms{66.33}{0.58} & \scms{59.67}{0.58} & \scms{\textbf{68.83}}{1.76} \\
 & Tokens $\downarrow$ & 419 & 2507 & 1264 & 2429 & 2411 & 2429 & 2353 & 2203 & 897 \\
\multirow{2}{*}{Qwen3-4B} & Accuracy $\uparrow$ & \scms{55.83}{0.76} & \scms{60.00}{0.50} & \scms{57.50}{1.32} & \scms{63.50}{1.80} & \scms{62.67}{2.31} & \scms{63.50}{1.80} & \scms{69.50}{2.18} & \scms{65.33}{0.76} & \scms{\textbf{70.83}}{0.76} \\
 & Tokens $\downarrow$ & 385 & 2321 & 1168 & 2309 & 2279 & 2309 & 2143 & 2218 & 818 \\
\multirow{2}{*}{Qwen3-8B} & Accuracy $\uparrow$ & \scms{59.17}{0.76} & \scms{62.33}{0.58} & \scms{59.00}{1.32} & \scms{59.83}{1.04} & \scms{62.00}{1.32} & \scms{59.83}{1.26} & \scms{64.67}{1.26} & \scms{61.50}{0.50} & \scms{\textbf{65.83}}{1.61} \\
 & Tokens $\downarrow$ & 387 & 2333 & 1174 & 2363 & 2308 & 2365 & 2227 & 2318 & 912 \\
\midrule
\multicolumn{11}{l}{\textbf{TruthfulQA (A/B answers)}} \\
\multirow{2}{*}{Gemma-3-4B} & Accuracy $\uparrow$ & \scms{62.67}{1.76} & \scms{63.83}{1.26} & \scms{68.00}{1.32} & \scms{66.83}{2.75} & \scms{70.17}{0.29} & \scms{66.83}{2.75} & \scms{68.67}{1.04} & \scms{67.17}{1.26} & \scms{\textbf{70.67}}{1.04} \\
 & Tokens $\downarrow$ & 67 & 405 & 236 & 583 & 418 & 583 & 449 & 617 & 191 \\
\multirow{2}{*}{Qwen3-4B} & Accuracy $\uparrow$ & \scms{76.17}{1.53} & \scms{75.83}{0.76} & \scms{79.00}{1.80} & \scms{79.83}{1.44} & \scms{79.00}{1.00} & \scms{79.83}{1.44} & \scms{82.00}{1.32} & \scms{80.00}{1.50} & \scms{\textbf{83.33}}{1.76} \\
 & Tokens $\downarrow$ & 105 & 627 & 353 & 828 & 691 & 828 & 677 & 716 & 208 \\
\multirow{2}{*}{Qwen3-8B} & Accuracy $\uparrow$ & \scms{82.67}{0.58} & \scms{82.00}{0.50} & \scms{83.67}{1.04} & \scms{85.00}{2.29} & \scms{82.50}{1.50} & \scms{85.00}{2.29} & \scms{86.33}{2.08} & \scms{84.83}{0.58} & \scms{\textbf{87.00}}{1.50} \\
 & Tokens $\downarrow$ & 96 & 576 & 305 & 746 & 590 & 746 & 608 & 698 & 236 \\
\bottomrule
\end{tabular*}
\end{table}

\subsection{Experimental Setup}
\label{sec:experimental_setup}

\paragraph{Tasks and models.}
We use Gemma~3 4B-IT~\citep{gemmateam2025gemma3}, Qwen3-4B, and Qwen3-8B~\citep{qwenteam2025qwen3} on 2WikiMultiHopQA~\citep{ho2020twowiki}, GSM8K~\citep{cobbe2021training}, MATH-500~\citep{hendrycks2021math}, and TruthfulQA~\citep{lin2022truthfulqa}.
Each evaluation set uses 240 tasks for 2Wiki and GSM8K, and 200 for MATH-500 and TruthfulQA.
We compare accuracy in selecting between two supplied answers on 2Wiki, MATH-500, and TruthfulQA, and in freely generating answers on 2Wiki and GSM8K.

\paragraph{Message-selection baselines.}
\textbf{Text-only} uses message text, while \textbf{Additive} also uses sender and receiver summaries.
\textbf{Activation products} adds element-wise products of sender summaries with receiver summaries before and after communication.
For two-answer scoring, CGD and these baselines use the same messages.
Each approach selects a message combination, and we calculate its accuracy using the receiver's recorded answer for that combination.
For free generation, we train baselines with CGD’s loss. Each baseline selects the message combination with its highest score.

\paragraph{Collaboration baselines.}
We compare CGD with \textbf{Single}, \textbf{SC6}~\citep{wang2023selfconsistency}, \textbf{Majority vote}~\citep{choi2025debate}, \textbf{MAD}~\citep{du2024multiagent}, \textbf{MAD-M2}~\citep{tian2026memorymasking}, \textbf{MADC}~\citep{zhang2026madc}, \textbf{SafeSieve}~\citep{zhang2026safesieve}, and \textbf{MOC}~\citep{guan2026moc}.

\paragraph{Training and evaluation.}
We train CGD separately for two-answer scoring and free generation on 340 2Wiki tasks, with 100 validation tasks and separate evaluation sets.
We report the average accuracy and standard deviation over three seeds.
Generation cost counts all generated tokens, including the intermediate messages and final answer.
We provide training details in Appendix~\ref{app:protocol}.

\begingroup
\setlength{\intextsep}{6pt plus 2pt minus 2pt}
\setlength{\textfloatsep}{8pt plus 2pt minus 2pt}
\setlength{\floatsep}{6pt plus 2pt minus 2pt}
\subsection{Experimental Results}
\label{sec:result}

\paragraph{Accuracy and generation cost.}
We compare the accuracy and generated tokens in Table~\ref{tab:four-dataset-accuracy-tokens}.
Results show that CGD ranks first in all nine two-answer settings and ranks first or joint first in GSM8K free generation.
Across all settings, it generates fewer tokens than all multi-agent baselines.
For 2Wiki free generation, we limit each baseline’s total generated tokens to CGD’s recorded token count for the same task and seed.
Table~\ref{tab:cgd-efficiency} shows accuracy gains over the strongest baseline of 16.53\%, 11.25\%, and 30.69\% on Gemma-3-4B, Qwen3-4B, and Qwen3-8B, respectively.

\begin{table}[t]
\centering
\caption{{Free-generation accuracy (\%) on 240 2Wiki tasks under matched total generation budgets.}}
\label{tab:cgd-efficiency}
\renewcommand{\scms}[2]{#1{\fontsize{5}{5}\selectfont$\pm#2$}}
\fontsize{7}{8}\selectfont
\setlength{\tabcolsep}{2pt}
\renewcommand{\arraystretch}{1.12}
\begin{tabular*}{\linewidth}{@{\extracolsep{\fill}}l*{8}{c}>{\columncolor{CGDshade}[\tabcolsep][0pt]}c@{}}
\toprule
& \multicolumn{2}{c}{Single-agent} & \multicolumn{7}{c}{Multi-agent} \\
\cmidrule(lr){2-3}\cmidrule(lr){4-10}
Model & Single & SC6 & Majority & MAD & MAD-M2 & MADC & MOC & SafeSieve & \textbf{CGD} \\
\midrule
Gemma-3-4B & \scms{24.72}{1.05} & \scms{25.42}{1.67} & \scms{26.39}{2.77} & \scms{7.22}{1.88} & \scms{20.97}{2.10} & \scms{7.22}{1.88} & \scms{20.28}{0.48} & \scms{25.42}{4.23} & \scms{\textbf{42.92}}{0.83} \\
Qwen3-4B & \scms{21.94}{0.24} & \scms{22.22}{0.48} & \scms{25.97}{1.97} & \scms{23.33}{2.53} & \scms{25.14}{3.13} & \scms{23.33}{2.53} & \scms{29.31}{2.44} & \scms{40.56}{2.37} & \scms{\textbf{51.81}}{0.64} \\
Qwen3-8B & \scms{23.61}{2.10} & \scms{23.89}{1.05} & \scms{30.00}{1.67} & \scms{14.17}{0.83} & \scms{30.83}{0.83} & \scms{14.17}{0.83} & \scms{28.75}{5.42} & \scms{29.86}{5.50} & \scms{\textbf{61.53}}{0.48} \\
\bottomrule
\end{tabular*}
\end{table}

\paragraph{Effect of receiver activation changes.}
\begin{wrapfigure}{r}{0.4\textwidth} \centering \includegraphics[width=\linewidth]{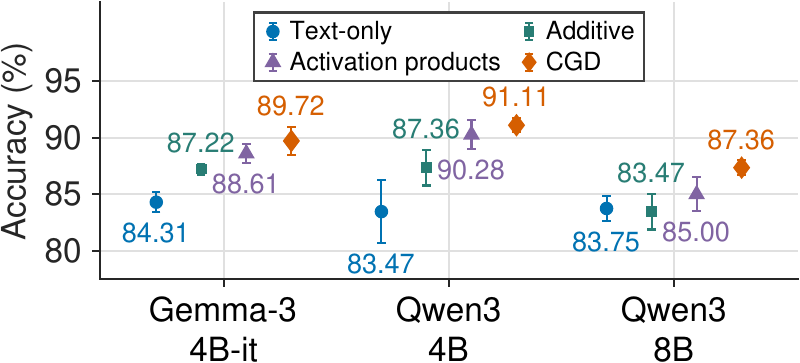} \caption{{Two-answer scoring accuracy on 2Wiki: 240 new tasks for the 4B models, and Table~1's tasks for Qwen3-8B.}} \label{fig:main_para3} \end{wrapfigure}
We present message selection comparisons in Fig.~\ref{fig:main_para3} using the same messages, recorded receiver answers, and training procedure. 
CGD ranks first across all three models and exceeds Activation products by 0.83\%--2.36\%. 
In addition, Table~\ref{tab:cgd-fixed-messages} in Appendix~\ref{sec:additional-selection} shows that CGD improves accuracy over sending both messages by 4.44\%--7.22\% across the three models. Table~\ref{tab:cgd-matched-products} in Appendix~\ref{sec:additional-selection} shows CGD's 0.42\%--1.11\% accuracy advantage over Activation products on Table~\ref{tab:four-dataset-accuracy-tokens}'s 2Wiki tasks with matched input scaling and weight penalty. These results support learning directly from receiver activation changes to select messages.

\paragraph{Effect of training set size.}
Figure~\ref{fig:additional_D3_3} shows that CGD exceeds Additive and Activation products across three models with 50, 100, 200, and 340 training tasks.
On Qwen3-8B, CGD reaches 85.00\% accuracy with 100 tasks, matching Activation products trained on 340 tasks.
It shows that CGD can achieve the same accuracy with fewer training tasks.

\begin{figure}[!h]
\centering
\begin{minipage}[t]{0.32\textwidth}
\centering
\includegraphics[width=\linewidth]{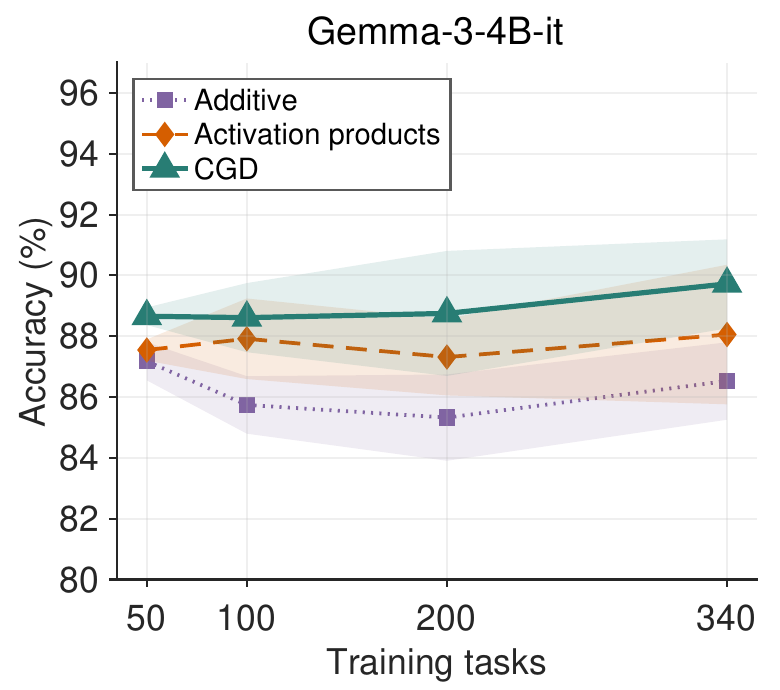}\\
{\footnotesize (a)}
\end{minipage}\hfill
\begin{minipage}[t]{0.32\textwidth}
\centering
\includegraphics[width=\linewidth]{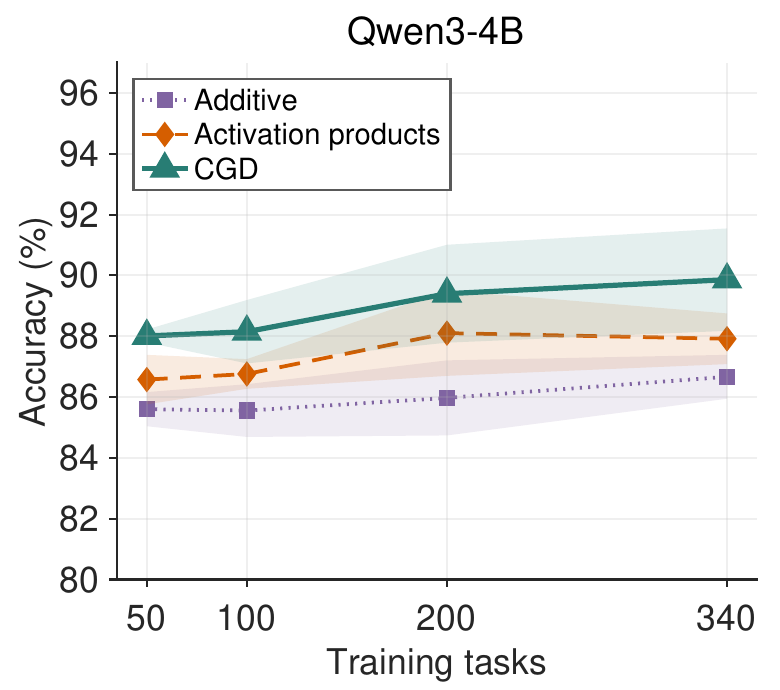}\\
{\footnotesize (b)}
\end{minipage}\hfill
\begin{minipage}[t]{0.32\textwidth}
\centering
\includegraphics[width=\linewidth]{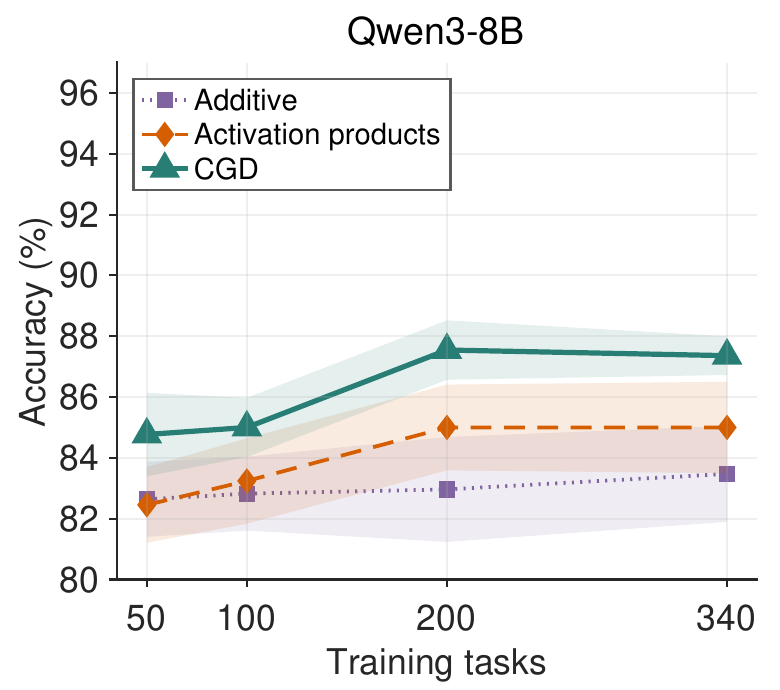}\\
{\footnotesize (c)}
\end{minipage}
\caption{{2Wiki learning curves. Lines and bands show mean accuracy and standard deviation.}}
\label{fig:additional_D3_3}
\end{figure}

\endgroup

\section{Conclusion} \label{sec:conclusion} In this work, we introduce \method{} to trace how messages correct or reinforce shared errors through receiver activation changes. Based on this framework, we propose \emph{Circuit-Guided Deliberation} to use sender information and activation changes to select messages that improve task performance. Our theory shows when activation replacement preserves receiver decisions and bounds CGD's loss in expected task utility. Across three models and four datasets, CGD achieves the highest or joint-highest average accuracy with fewer generated tokens than each multi-agent baseline, and improves the free-generation accuracy under the same generation limits. 

\section*{AI use statement}
We use generative AI tools to assist with the literature search, grammar correction, experiment deployments, and manuscript polishing. The authors 
reviewed these outputs against the cited sources and stored experiment records and take full responsibility for the paper.

\section*{Ethics statement}
This work studies how communication changes errors in collaborative AI systems.
Synthetic tasks use synthetic conflicting evidence, and the real-task
transformations are derived from existing benchmarks. No human subjects are
involved. Released artifacts follow the licenses of the models and datasets.

\section*{Reproducibility statement}
The anonymous repository of this project is \url{https://anonymous.4open.science/r/social-circuits-4416}. Appendix~\ref{app:algorithms} gives the
pseudocode, and Appendix~\ref{app:protocol} describes the models, data splits,
target quantities, baselines, and statistical analyses.
The code is also included in the supplementary material.

\bibliography{references}

\begin{thebibliography}{41}
\providecommand{\natexlab}[1]{#1}
\providecommand{\url}[1]{\texttt{#1}}
\expandafter\ifx\csname urlstyle\endcsname\relax
  \providecommand{\doi}[1]{doi: #1}\else
  \providecommand{\doi}{doi: \begingroup \urlstyle{rm}\Url}\fi

\bibitem[Ai et~al.(2026)Ai, Pan, Simchi-Levi, Tambe, and Xu]{ai2026beyond}
Rui Ai, Yuqi Pan, David Simchi-Levi, Milind Tambe, and Haifeng Xu.
\newblock Beyond majority voting: {LLM} aggregation by leveraging higher-order
  information.
\newblock In \emph{Proceedings of the International Conference on Machine
  Learning}, 2026.
\newblock URL \url{https://arxiv.org/abs/2510.01499}.

\bibitem[Burges et~al.(2005)Burges, Shaked, Renshaw, Lazier, Deeds, Hamilton,
  and Hullender]{burges2005learning}
Chris Burges, Tal Shaked, Erin Renshaw, Ari Lazier, Matt Deeds, Nicole
  Hamilton, and Greg Hullender.
\newblock Learning to rank using gradient descent.
\newblock In \emph{Proceedings of the International Conference on Machine
  Learning}, pp.\  89--96, 2005.
\newblock URL \url{https://doi.org/10.1145/1102351.1102363}.

\bibitem[Choi et~al.(2025)Choi, Zhu, and Li]{choi2025debate}
Hyeong~Kyu Choi, Xiaojin Zhu, and Sharon Li.
\newblock Debate or vote: Which yields better decisions in multi-agent large
  language models?
\newblock In \emph{Proceedings of the Conference on Neural Information
  Processing Systems}, 2025.
\newblock URL
  \url{https://proceedings.neurips.cc/paper_files/paper/2025/hash/934252acd87f254d5d4672fbde283bd2-Abstract-Conference.html}.

\bibitem[Cobbe et~al.(2021)Cobbe, Kosaraju, Bavarian, Chen, Jun, Kaiser,
  Plappert, Tworek, Hilton, Nakano, Hesse, and Schulman]{cobbe2021training}
Karl Cobbe, Vineet Kosaraju, Mohammad Bavarian, Mark Chen, Heewoo Jun, Lukasz
  Kaiser, Matthias Plappert, Jerry Tworek, Jacob Hilton, Reiichiro Nakano,
  Christopher Hesse, and John Schulman.
\newblock Training verifiers to solve math word problems.
\newblock arXiv preprint, 2021.
\newblock URL \url{https://arxiv.org/abs/2110.14168}.

\bibitem[Dembczy{\'n}ski et~al.(2012)Dembczy{\'n}ski, Kot{\l}owski, and
  H{\"u}llermeier]{dembczynski2012consistent}
Krzysztof Dembczy{\'n}ski, Wojciech Kot{\l}owski, and Eyke H{\"u}llermeier.
\newblock Consistent multilabel ranking through univariate loss minimization.
\newblock In \emph{Proceedings of the International Conference on Machine
  Learning}, 2012.
\newblock URL \url{https://icml.cc/2012/papers/661.pdf}.

\bibitem[Du et~al.(2024)Du, Li, Torralba, Tenenbaum, and
  Mordatch]{du2024multiagent}
Yilun Du, Shuang Li, Antonio Torralba, Joshua~B. Tenenbaum, and Igor Mordatch.
\newblock Improving factuality and reasoning in language models through
  multiagent debate.
\newblock In \emph{Proceedings of the International Conference on Machine
  Learning}, pp.\  11733--11763, 2024.
\newblock URL \url{https://proceedings.mlr.press/v235/du24e.html}.

\bibitem[Duchi et~al.(2010)Duchi, Mackey, and Jordan]{duchi2010consistency}
John~C. Duchi, Lester~W. Mackey, and Michael~I. Jordan.
\newblock On the consistency of ranking algorithms.
\newblock In \emph{Proceedings of the International Conference on Machine
  Learning}, pp.\  327--334, 2010.
\newblock URL \url{https://icml.cc/Conferences/2010/papers/421.pdf}.

\bibitem[Estornell \& Liu(2024)Estornell and Liu]{estornell2024multillm}
Andrew Estornell and Yang Liu.
\newblock Multi-{LLM} debate: Framework, principals, and interventions.
\newblock In \emph{Proceedings of the Conference on Neural Information
  Processing Systems}, 2024.
\newblock URL
  \url{https://proceedings.neurips.cc/paper_files/paper/2024/hash/32e07a110c6c6acf1afbf2bf82b614ad-Abstract-Conference.html}.

\bibitem[Geiger et~al.(2021)Geiger, Lu, Icard, and Potts]{geiger2021causal}
Atticus Geiger, Hanson Lu, Thomas Icard, and Christopher Potts.
\newblock Causal abstractions of neural networks.
\newblock In \emph{Proceedings of the Conference on Neural Information
  Processing Systems}, pp.\  9574--9586, 2021.
\newblock URL
  \url{https://proceedings.neurips.cc/paper/2021/hash/4f5c422f4d49a5a807eda27434231040-Abstract.html}.

\bibitem[Geiger et~al.(2025)Geiger, Ibeling, Zur, Chaudhary, Chauhan, Huang,
  Arora, Wu, Goodman, Potts, and Icard]{geiger2025causal}
Atticus Geiger, Duligur Ibeling, Amir Zur, Maheep Chaudhary, Sonakshi Chauhan,
  Jing Huang, Aryaman Arora, Zhengxuan Wu, Noah Goodman, Christopher Potts, and
  Thomas Icard.
\newblock Causal abstraction: A theoretical foundation for mechanistic
  interpretability.
\newblock \emph{Journal of Machine Learning Research}, 26\penalty0
  (83):\penalty0 1--64, 2025.
\newblock URL \url{https://jmlr.org/papers/v26/23-0058.html}.

\bibitem[{Gemma Team}(2025)]{gemmateam2025gemma3}
{Gemma Team}.
\newblock Gemma 3 technical report.
\newblock arXiv preprint, 2025.
\newblock URL \url{https://arxiv.org/abs/2503.19786}.

\bibitem[{Gemma Team}(2026)]{gemmateam2026gemma4}
{Gemma Team}.
\newblock Gemma 4 technical report.
\newblock arXiv preprint, 2026.
\newblock URL \url{https://arxiv.org/abs/2607.02770}.

\bibitem[Guan et~al.(2026)Guan, Wang, Lu, Wang, Yan, and Duan]{guan2026moc}
Yao Guan, Lin Wang, Zhihu Lu, Ziyi Wang, Wenzhu Yan, and Qiang Duan.
\newblock {MOC}: Multi-order communication in {LLM}-based multi-agent systems.
\newblock In \emph{Proceedings of the International Conference on Machine
  Learning}, 2026.
\newblock URL \url{https://arxiv.org/abs/2606.02359}.

\bibitem[Gultchin et~al.(2021)Gultchin, Watson, Kusner, and
  Silva]{gultchin2021complex}
Limor Gultchin, David Watson, Matt Kusner, and Ricardo Silva.
\newblock Operationalizing complex causes: A pragmatic view of mediation.
\newblock In \emph{Proceedings of the International Conference on Machine
  Learning}, pp.\  3875--3885, 2021.
\newblock URL \url{https://proceedings.mlr.press/v139/gultchin21a.html}.

\bibitem[He et~al.(2026)He, Chen, Yang, Qiao, Ju, Liu, Wen, and
  Liu]{he2026minority}
Chuan He, Zebin Chen, Zhengyi Yang, Shaobo Qiao, Mingchen Ju, Jiate Liu, Dong
  Wen, and Guanfeng Liu.
\newblock Minority sentinel: When to overturn majority voting in multi-agent
  {LLM} debates.
\newblock AgentSearch Workshop at the International ACM SIGIR Conference on
  Research and Development in Information Retrieval, 2026.
\newblock URL \url{https://arxiv.org/abs/2606.29270}.

\bibitem[Hendrycks et~al.(2021)Hendrycks, Burns, Kadavath, Arora, Basart, Tang,
  Song, and Steinhardt]{hendrycks2021math}
Dan Hendrycks, Collin Burns, Saurav Kadavath, Akul Arora, Steven Basart, Eric
  Tang, Dawn Song, and Jacob Steinhardt.
\newblock Measuring mathematical problem solving with the {MATH} dataset.
\newblock In \emph{Proceedings of the Neural Information Processing Systems
  Track on Datasets and Benchmarks}, 2021.
\newblock URL
  \url{https://datasets-benchmarks-proceedings.neurips.cc/paper/2021/hash/be83ab3ecd0db773eb2dc1b0a17836a1-Abstract-round2.html}.

\bibitem[Ho et~al.(2020)Ho, Duong~Nguyen, Sugawara, and Aizawa]{ho2020twowiki}
Xanh Ho, Anh-Khoa Duong~Nguyen, Saku Sugawara, and Akiko Aizawa.
\newblock Constructing a multi-hop {QA} dataset for comprehensive evaluation of
  reasoning steps.
\newblock In \emph{Proceedings of the International Conference on Computational
  Linguistics}, pp.\  6609--6625, 2020.
\newblock URL \url{https://aclanthology.org/2020.coling-main.580/}.

\bibitem[Huang et~al.(2026)Huang, Li, Li, Kwon, Yu, and
  Zhang]{huang2026cagecal}
Jiatan Huang, Mingchen Li, Ziming Li, Sunjae Kwon, Hong Yu, and Chuxu Zhang.
\newblock Counterfactual graph for multi-agent {LLM} calibration.
\newblock arXiv preprint, 2026.
\newblock URL \url{https://arxiv.org/abs/2605.30653}.

\bibitem[Huang et~al.(2025)Huang, Tao, Icard, Yang, and
  Potts]{huang2025internal}
Jing Huang, Junyi Tao, Thomas Icard, Diyi Yang, and Christopher Potts.
\newblock Internal causal mechanisms robustly predict language model
  out-of-distribution behaviors.
\newblock In \emph{Proceedings of the International Conference on Machine
  Learning}, pp.\  25791--25812, 2025.
\newblock URL \url{https://proceedings.mlr.press/v267/huang25af.html}.

\bibitem[Kim et~al.(2025)Kim, Garg, Peng, and Garg]{kim2025correlated}
Elliot~Myunghoon Kim, Avi Garg, Kenny Peng, and Nikhil Garg.
\newblock Correlated errors in large language models.
\newblock In \emph{Proceedings of the International Conference on Machine
  Learning}, pp.\  30038--30066, 2025.
\newblock URL \url{https://proceedings.mlr.press/v267/kim25e.html}.

\bibitem[Li et~al.(2026)Li, He, Ji, Wang, Liu, and Sun]{li2026e2explainer}
Junzhi Li, Peng He, Qirui Ji, Wei Wang, Lixiang Liu, and Chuxiong Sun.
\newblock Discovering efficient and explainable communication topologies for
  {LLM}-based multi-agent systems via causal inference.
\newblock arXiv preprint, 2026.
\newblock URL \url{https://arxiv.org/abs/2608.12921}.

\bibitem[Lin et~al.(2022)Lin, Hilton, and Evans]{lin2022truthfulqa}
Stephanie Lin, Jacob Hilton, and Owain Evans.
\newblock {TruthfulQA}: Measuring how models mimic human falsehoods.
\newblock In \emph{Proceedings of the Annual Meeting of the Association for
  Computational Linguistics}, pp.\  3214--3252, 2022.
\newblock URL \url{https://aclanthology.org/2022.acl-long.229/}.

\bibitem[Lin et~al.(2025)Lin, Morency, and Ben-Michael]{lin2025isolated}
Victoria Lin, Louis-Philippe Morency, and Eli Ben-Michael.
\newblock Isolated causal effects of natural language.
\newblock In \emph{Proceedings of the International Conference on Machine
  Learning}, pp.\  37919--37941, 2025.
\newblock URL \url{https://proceedings.mlr.press/v267/lin25k.html}.

\bibitem[Lowe et~al.(2019)Lowe, Foerster, Boureau, Pineau, and
  Dauphin]{lowe2019pitfalls}
Ryan Lowe, Jakob Foerster, Y-Lan Boureau, Joelle Pineau, and Yann Dauphin.
\newblock On the pitfalls of measuring emergent communication.
\newblock In \emph{Proceedings of the International Conference on Autonomous
  Agents and Multiagent Systems}, pp.\  693--701, 2019.
\newblock URL \url{https://ifaamas.org/Proceedings/aamas2019/pdfs/p693.pdf}.

\bibitem[Marks et~al.(2025)Marks, Rager, Michaud, Belinkov, Bau, and
  Mueller]{marks2025sparse}
Samuel Marks, Can Rager, Eric~J. Michaud, Yonatan Belinkov, David Bau, and
  Aaron Mueller.
\newblock Sparse feature circuits: Discovering and editing interpretable causal
  graphs in language models.
\newblock In \emph{Proceedings of the International Conference on Learning
  Representations}, 2025.
\newblock URL
  \url{https://proceedings.iclr.cc/paper_files/paper/2025/hash/3ba4d47a83e498c2b1a0868cba20f6de-Abstract-Conference.html}.

\bibitem[Meng et~al.(2022)Meng, Bau, Andonian, and Belinkov]{meng2022locating}
Kevin Meng, David Bau, Alex Andonian, and Yonatan Belinkov.
\newblock Locating and editing factual associations in {GPT}.
\newblock In \emph{Proceedings of the Conference on Neural Information
  Processing Systems}, 2022.
\newblock URL
  \url{https://papers.neurips.cc/paper_files/paper/2022/hash/6f1d43d5a82a37e89b0665b33bf3a182-Abstract-Conference.html}.

\bibitem[{Qwen Team}(2025)]{qwenteam2025qwen3}
{Qwen Team}.
\newblock Qwen3 technical report.
\newblock arXiv preprint, 2025.
\newblock URL \url{https://arxiv.org/abs/2505.09388}.

\bibitem[Ramesh \& Li(2025)Ramesh and Li]{ramesh2025activations}
Vignav Ramesh and Kenneth Li.
\newblock Communicating activations between language model agents.
\newblock In \emph{Proceedings of the International Conference on Machine
  Learning}, pp.\  51094--51116, 2025.
\newblock URL \url{https://proceedings.mlr.press/v267/ramesh25a.html}.

\bibitem[Shen et~al.(2025)Shen, Liu, Dai, Wang, Miao, Tan, Pan, and
  Wang]{shen2025propagation}
Xu~Shen, Yixin Liu, Yiwei Dai, Yili Wang, Rui Miao, Yue Tan, Shirui Pan, and
  Xin Wang.
\newblock Understanding the information propagation effects of communication
  topologies in {LLM}-based multi-agent systems.
\newblock In \emph{Proceedings of the Conference on Empirical Methods in
  Natural Language Processing}, pp.\  12347--12361, 2025.
\newblock URL \url{https://aclanthology.org/2025.emnlp-main.623/}.

\bibitem[Smit et~al.(2024)Smit, Grinsztajn, Duckworth, Barrett, and
  Pretorius]{smit2024mad}
Andries~Petrus Smit, Nathan Grinsztajn, Paul Duckworth, Thomas~D. Barrett, and
  Arnu Pretorius.
\newblock Should we be going {MAD}? a look at multi-agent debate strategies for
  {LLM}s.
\newblock In \emph{Proceedings of the International Conference on Machine
  Learning}, pp.\  45883--45905, 2024.
\newblock URL \url{https://proceedings.mlr.press/v235/smit24a.html}.

\bibitem[Tang et~al.(2026)Tang, Meng, Costa, Zhang, Ye, and
  Xi]{tang2026variance}
Luoxi Tang, Yuqiao Meng, Joseph Costa, Yingxue Zhang, Muchao Ye, and Zhaohan
  Xi.
\newblock The value of variance: Mitigating debate collapse in multi-agent
  systems via uncertainty-driven policy optimization.
\newblock In \emph{Proceedings of the International Conference on Machine
  Learning}, 2026.
\newblock URL \url{https://openreview.net/forum?id=Bc6c2OVWRh}.

\bibitem[Tang et~al.(2025)Tang, Su, Zhou, Liu, Zhang, Ma, and
  Ai]{tang2025statedelta}
Yichen Tang, Weihang Su, Yujia Zhou, Yiqun Liu, Min Zhang, Shaoping Ma, and
  Qingyao Ai.
\newblock Augmenting multi-agent communication with state delta trajectory.
\newblock In \emph{Proceedings of the Conference on Empirical Methods in
  Natural Language Processing}, pp.\  10219--10240, 2025.
\newblock URL \url{https://aclanthology.org/2025.emnlp-main.518/}.

\bibitem[Tian et~al.(2026)Tian, Feng, Zhao, Zhu, Yan, and
  Han]{tian2026memorymasking}
Hongduan Tian, Xiao Feng, Ziyuan Zhao, Xiangyu Zhu, Rolan Yan, and Bo~Han.
\newblock Multi-agent debate with memory masking.
\newblock In \emph{Proceedings of the International Conference on Learning
  Representations}, 2026.
\newblock URL \url{https://openreview.net/forum?id=EdTt8nMAMA}.

\bibitem[Vig et~al.(2020)Vig, Gehrmann, Belinkov, Qian, Nevo, Singer, and
  Shieber]{vig2020investigating}
Jesse Vig, Sebastian Gehrmann, Yonatan Belinkov, Sharon Qian, Daniel Nevo,
  Yaron Singer, and Stuart Shieber.
\newblock Investigating gender bias in language models using causal mediation
  analysis.
\newblock In \emph{Proceedings of the Conference on Neural Information
  Processing Systems}, 2020.
\newblock URL
  \url{https://papers.nips.cc/paper/2020/hash/92650b2e92217715fe312e6fa7b90d82-Abstract.html}.

\bibitem[Wang et~al.(2023)Wang, Wei, Schuurmans, Le, Chi, Narang, Chowdhery,
  and Zhou]{wang2023selfconsistency}
Xuezhi Wang, Jason Wei, Dale Schuurmans, Quoc~V. Le, Ed~H. Chi, Sharan Narang,
  Aakanksha Chowdhery, and Denny Zhou.
\newblock Self-consistency improves chain of thought reasoning in language
  models.
\newblock In \emph{Proceedings of the International Conference on Learning
  Representations}, 2023.
\newblock URL \url{https://openreview.net/forum?id=1PL1NIMMrw}.

\bibitem[Xue et~al.(2022)Xue, Yuan, Zhang, and Yu]{ijcai2022p82}
Di~Xue, Lei Yuan, Zongzhang Zhang, and Yang Yu.
\newblock Efficient multi-agent communication via {Shapley} message value.
\newblock In \emph{Proceedings of the International Joint Conference on
  Artificial Intelligence}, pp.\  578--584, 2022.
\newblock URL \url{https://doi.org/10.24963/ijcai.2022/82}.

\bibitem[Zhang \& Nanda(2024)Zhang and Nanda]{zhang2024towards}
Fred Zhang and Neel Nanda.
\newblock Towards best practices of activation patching in language models:
  Metrics and methods.
\newblock In \emph{Proceedings of the International Conference on Learning
  Representations}, 2024.
\newblock URL
  \url{https://proceedings.iclr.cc/paper_files/paper/2024/hash/06a52a54c8ee03cd86771136bc91eb1f-Abstract-Conference.html}.

\bibitem[Zhang et~al.(2025)Zhang, Yue, Sun, Wan, Yu, Fang, Wang, Chen, and
  Cheng]{zhang2025gdesigner}
Guibin Zhang, Yanwei Yue, Xiangguo Sun, Guancheng Wan, Miao Yu, Junfeng Fang,
  Kun Wang, Tianlong Chen, and Dawei Cheng.
\newblock {G-Designer}: Architecting multi-agent communication topologies via
  graph neural networks.
\newblock In \emph{Proceedings of the International Conference on Machine
  Learning}, pp.\  76678--76692, 2025.
\newblock URL \url{https://proceedings.mlr.press/v267/zhang25cu.html}.

\bibitem[Zhang \& Emu(2026)Zhang and Emu]{zhang2026latent}
Huixiang Zhang and Mahzabeen Emu.
\newblock Do latent channels actually communicate? a causal audit of latent
  multi-agent {LLM}.
\newblock arXiv preprint, 2026.
\newblock URL \url{https://arxiv.org/abs/2607.26773}.

\bibitem[Zhang et~al.(2026{\natexlab{a}})Zhang, Liu, Zheng, Liang, and
  Wang]{zhang2026madc}
Qian Zhang, Jinyi Liu, Yan Zheng, Hebin Liang, and Lanjun Wang.
\newblock Key decision-makers in multi-agent debates: Who holds the power?
\newblock In \emph{Proceedings of the AAAI Conference on Artificial
  Intelligence}, pp.\  29883--29891, 2026{\natexlab{a}}.
\newblock URL \url{https://ojs.aaai.org/index.php/AAAI/article/view/40235}.

\bibitem[Zhang et~al.(2026{\natexlab{b}})Zhang, Zhao, Wang, Chen, Zhang, Zhang,
  Wang, and Wen]{zhang2026safesieve}
Ruijia Zhang, Xinyan Zhao, Ruixiang Wang, Sigen Chen, Guibin Zhang, An~Zhang,
  Kun Wang, and Qingsong Wen.
\newblock {SafeSieve}: From heuristics to experience in progressive pruning for
  {LLM}-based multi-agent communication.
\newblock In \emph{Proceedings of the AAAI Conference on Artificial
  Intelligence}, pp.\  29892--29900, 2026{\natexlab{b}}.
\newblock URL \url{https://ojs.aaai.org/index.php/AAAI/article/view/40236}.

\end{thebibliography}
\bibliographystyle{iclr2027_conference}

\clearpage
\appendix
\startcontents[appendices]
\section*{Appendix Contents}
{\small
\printcontents[appendices]{}{1}[2]{}
}
\clearpage

\section{Notations}
\label{app:notation}

{
\renewcommand{\arraystretch}{1.12}
\begin{longtable}{@{}p{0.29\linewidth}>{\raggedright\arraybackslash}p{0.67\linewidth}@{}}
\caption{List of key notations.}\label{tab:notation}\\
\toprule
Symbol & Meaning\\
\midrule
\endfirsthead
\multicolumn{2}{l}{\tablename~\thetable\ (continued)}\\
\toprule
Symbol & Meaning\\
\midrule
\endhead
\bottomrule
\endfoot
$\mathcal A$ & Message combination, a subset of $\mathcal C$.\\
$\widehat{\mathcal A},\mathcal A^\star$ & Combinations with the highest CGD score and expected task utility, respectively.\\
$\mathcal B$ & Another message combination in $\mathcal S$.\\
$\mathcal C$ & Candidate messages for one receiver.\\
$d$ & Activation-summary dimension.\\
$e$ & Incorrect answer favored by multiple agents before delivery.\\
$\mathrm{err}$ & Largest absolute error in an answer-pair score difference between original-message and replacement runs.\\
$f_{\boldsymbol{\theta}},F(o)$ & CGD scoring function and receiver's score for answer $o$.\\
$g$ & Candidate answer satisfying $u(g)>u(e)$.\\
$\mathrm{gap}$ & Original-message score of the selected answer minus the highest other answer score.\\
$\ell(\mathcal A,\mathcal B)$ & Pairwise training loss.\\
$\mathcal L(S),\mathcal L^\star$ & Expected summed pairwise loss without the weight penalty, and its lowest value approached over free scores with $S(\varnothing)=0$.\\
$o,\mathcal O$ & Candidate answer and the set of candidate answers.\\
$q,q_{\mathcal A}$ & Receiver's output, and its output after reading $\mathcal A$.\\
$q^{\rm act}_{\mathcal A}$ & Output after replacing selected no-message activations with those from the run of $\mathcal A$.\\
$\mathbf r_{\mathcal A}$ & Receiver activation summary in $\mathbb R^d$ after reading $\mathcal A$.\\
$\mathbf s_{\mathcal A}$ & Mean sender activation summary in $\mathbb R^d$ for nonempty $\mathcal A$.\\
$S(\mathcal A),\mathcal S$ & CGD score and the set of all message combinations, including $\varnothing$; $S(\varnothing)=0$.\\
$u(q),u^{\rm act}_{\mathcal A}$ & Task utility of $q$ and of $q^{\rm act}_{\mathcal A}$; binary task success in the CGD theory.\\
$V(\mathcal A)$ & Utility gain over sending no message: $u(q_{\mathcal A})-u(q_{\varnothing})$.\\
$\mathbf x_{\mathcal A}$ & CGD scoring input for nonempty $\mathcal A$.\\
$z$ & Group of matched reference, original-message, and replacement runs.\\
$\Delta F^{\rm msg}(o),\Delta F^{\rm act}(o)$ & Message and replacement effects on the score of $o$.\\
$\Delta^{\rm msg}(z),\Delta^{\rm act}(z)$ & Message and replacement effects on an answer-pair score difference.\\
$\Delta u^{\rm msg},\Delta u^{\rm act}$ & Message and replacement effects on task utility.\\
$\boldsymbol{\theta}$ & Learned weights and bias of the scoring function.\\
$\rho$ & Recovery ratio $\mathbb E[\Delta^{\rm act}]/\mathbb E[\Delta^{\rm msg}]$, with a nonzero denominator.\\
\end{longtable}
}

The $\Delta$ terms measure changes from the same reference run. In the CGD theory, expectations,
probabilities, and covariances are based on the candidate set and all scoring inputs.


\section{Proofs}
\label{sec:proofs}

We first prove Proposition~\ref{prop:sc-full-recovery} and
Theorem~\ref{thm:sequential-social-circuit} for Social Circuits.
We then bound the expected task-utility difference under pairwise training
and use this bound to prove Theorem~\ref{thm:cgd-recovery-ranking}.

\subsection{Proofs for Social Circuits}
\label{app:proof-sequential-social-circuit}

\begin{proposition}[Complete recovery through receiver activations]
\label{prop:sc-full-recovery}
We consider original and reference messages that occupy the same token
positions. Each layer takes the preceding layer's activations as input,
and each token uses only its own and earlier positions.
When scoring each $o\in\mathcal O$, restoring the original-message
activations at one layer's output for all message tokens and all later
tokens used to score $o$ leads to $F^{\rm act}(o)=F^{\rm msg}(o)$.
\end{proposition}

\begin{proof}[Proof of Proposition~\ref{prop:sc-full-recovery}]
We score the same answer $o\in\mathcal O$ at the same token positions in
both runs. We keep the tokens before the message unchanged. Since these
tokens cannot use later positions, their activations remain unchanged
at every layer.

At the selected layer, we restore the original-message activations at
all message positions and all later positions used to score $o$.
With the unchanged earlier positions, these activations form
the same layer output as in the original-message run. We therefore
obtain the same outputs from each remaining layer and the answer-scoring
function. Hence,
\begin{equation}
F^{\rm act}(o)=F^{\rm msg}(o)
\qquad\text{for every }o\in\mathcal O.
\label{eq:sc-full-score-recovery}
\end{equation}
This completes the proof of Proposition~\ref{prop:sc-full-recovery}.
\end{proof}

By subtracting the same reference scores, we obtain
$\Delta F^{\rm act}(o)=\Delta F^{\rm msg}(o)$.
With the same rule for ties, we also have $q^{\rm act}=q^{\rm msg}$ and
$\Delta u^{\rm act}=\Delta u^{\rm msg}$.
We next allow errors in the recovered score differences.

\begin{proof}[Proof of Theorem~\ref{thm:sequential-social-circuit}]
We first assume $\mathrm{err}<\mathrm{gap}$.
For every $o\in\mathcal O\setminus\{q^{\rm msg}\}$, we have
\begin{equation}
F^{\rm act}(q^{\rm msg})-F^{\rm act}(o)
\overset{(\text{a})}{\geq}
F^{\rm msg}(q^{\rm msg})-F^{\rm msg}(o)-\mathrm{err}
\overset{(\text{b})}{\geq}
\mathrm{gap}-\mathrm{err}>0,
\label{eq:sc-recovered-answer-order}
\end{equation}
where inequality (a) follows from the definition of $\mathrm{err}$,
and inequality (b) follows from the definition of $\mathrm{gap}$.
Thus $q^{\rm msg}$ remains the highest-scoring answer, so
$q^{\rm act}=q^{\rm msg}$. By substituting into the task-performance
change, we have
\begin{equation}
\Delta u^{\rm act}
=u(q^{\rm act})-u(q^{\rm ref})
=u(q^{\rm msg})-u(q^{\rm ref})
=\Delta u^{\rm msg}.
\label{eq:sc-recovered-utility}
\end{equation}

We next consider the message and activation effects for the same
answer pair $o_1,o_2$. By cancelling the common reference scores, we have
\begin{equation}
\left|\Delta^{\rm act}-\Delta^{\rm msg}\right|
=\left|
\left[F^{\rm act}(o_1)-F^{\rm act}(o_2)\right]
-\left[F^{\rm msg}(o_1)-F^{\rm msg}(o_2)\right]
\right|\leq\mathrm{err}.
\label{eq:sc-effect-recovery-error}
\end{equation}
We obtain the inequality from the definition of $\mathrm{err}$.
We now consider the two possible signs of $\Delta^{\rm msg}$:
\begin{align}
\Delta^{\rm msg}>\mathrm{err}
&\quad\Longrightarrow\quad
\Delta^{\rm act}\geq\Delta^{\rm msg}-\mathrm{err}>0,
\nonumber\\
\Delta^{\rm msg}<-\mathrm{err}
&\quad\Longrightarrow\quad
\Delta^{\rm act}\leq\Delta^{\rm msg}+\mathrm{err}<0.
\label{eq:sc-effect-sign}
\end{align}
Thus $|\Delta^{\rm msg}|>\mathrm{err}$ implies
$\Delta^{\rm act}\Delta^{\rm msg}>0$.
This completes the proof of Theorem~\ref{thm:sequential-social-circuit}.
\end{proof}

By Theorem~\ref{thm:sequential-social-circuit}, different answers require
$\mathrm{err}\geq\mathrm{gap}$. Taking probabilities over matched
reruns, we have
\begin{equation}
\Pr(q^{\rm act}\ne q^{\rm msg})
\leq\Pr(\mathrm{err}\geq\mathrm{gap}).
\label{eq:activation-decision-bound}
\end{equation}

\paragraph{Average recovery and individual decisions.}
We compare the message-effect recovery ratio $\rho$ with answer agreement
using two equally likely tasks with $\mathcal O=\{o_1,o_2\}$.
The entries below are the score of $o_1$ minus the score of $o_2$:
\begin{center}
\begin{tabular}{lrrr}
 &Reference message&Original message&Activation replacement\\ \hline
First task&$-2$&$-1$&$1$\\
Second task&$-2$&$1$&$-1$
\end{tabular}
\end{center}
In this example, the score difference equals the reference value $-2$
plus two terms determined by receiver activations. Both terms are zero in the reference
run. In the original-message run, they are $(3,-2)$ for the first
task and $(1,2)$ for the second. After restoring the activations that determine the first term, we obtain the activation-replacement scores shown in the table.
Thus the message effects are $1,3$ and the activation effects are $3,1$, so
\begin{equation}
\rho=\frac{(3+1)/2}{(1+3)/2}=1,
\qquad
\Pr(q^{\rm act}\ne q^{\rm msg})=1.
\label{eq:sc-average-recovery-counterexample}
\end{equation}
We therefore recover the average score effect without recovering either answer.

\subsection{Expected Task-Utility Difference under Pairwise Training}
\label{app:cgd-correlated-selection}

We use binary task utilities as shown in the main text.
We fix the candidate messages, scoring inputs, and learned scores.
All expectations, probabilities, and covariances in this subsection
and the next are under the resulting distribution $\mathcal P$.
We omit their subscript $\mathcal P$ to keep the formulas compact.
We analyze the expected pairwise loss $\mathcal L(S)$ defined in the
main text. Its infimum $\mathcal L^\star$ is taken over all real-valued
combination scores with $S(\varnothing)=0$.
Subtracting a common constant from all scores leaves the loss unchanged,
so this normalization does not change the infimum.
Minimizing the shared scoring function's training objective does not by
itself imply $\mathcal L(S)=\mathcal L^\star$.

\begin{lemma}[Expected task-utility difference under dependent task utilities]
\label{thm:cgd-correlated-selection}
For $|\mathcal S|\geq2$, any joint distribution of binary task utilities,
and any finite scores, we have
\begin{align}
\mathbb E\left[V(\mathcal A^\star)-V(\widehat{\mathcal A})\right]
\leq&\:\sum_{\mathcal B\in\mathcal S\setminus
\{\mathcal A^\star,\widehat{\mathcal A}\}}
\max\left\{0,\operatorname{Cov}\left(
u(q_{\widehat{\mathcal A}})-u(q_{\mathcal A^\star}),
u(q_{\mathcal B})\right)\right\}
\nonumber\\
&+\sqrt{|\mathcal S|\left(\mathcal L(S)-\mathcal L^\star\right)}.
\label{eq:cgd-dependent-utility-bound}
\end{align}
\end{lemma}

\begin{proof}
We consider $\mathcal A^\star\ne\widehat{\mathcal A}$. Otherwise, the expected task-utility difference is zero.
By cancelling the common no-message utility in $V$, we have
\begin{equation}
V(\mathcal A)-V(\mathcal B)
=u(q_{\mathcal A})-u(q_{\mathcal B}).
\label{eq:cgd-utility-cancellation}
\end{equation}
With binary utilities, each pair contributes to the loss only when
its two utilities differ. By differentiating these logistic terms, we have
\begin{align}
2\frac{\partial\mathcal L(S)}{\partial S(\mathcal A)}
=&\:\sum_{\mathcal B\in\mathcal S\setminus\{\mathcal A\}}
\Pr\left(u(q_{\mathcal A})\ne u(q_{\mathcal B})\right)
\frac{\exp\left(S(\mathcal A)-S(\mathcal B)\right)-1}
{\exp\left(S(\mathcal A)-S(\mathcal B)\right)+1}\nonumber\\
&-|\mathcal S|\,\mathbb E\left[u(q_{\mathcal A})\right]
+\sum_{\mathcal B\in\mathcal S}\mathbb E\left[u(q_{\mathcal B})\right].
\label{eq:cgd-pairwise-gradient}
\end{align}
Here we differentiate with respect to each combination score, temporarily
allowing $S(\varnothing)$ to vary. We can restore $S(\varnothing)=0$
by subtracting its value from all scores without changing the loss.

We add the same amount to $S(\mathcal A^\star)$ as we subtract from $S(\widehat{\mathcal A})$, keeping all other scores fixed.
The difference $S(\mathcal A^\star)-S(\widehat{\mathcal A})$ changes by twice the added amount.
For each remaining combination $\mathcal B$, the differences $S(\mathcal A^\star)-S(\mathcal B)$ and $S(\widehat{\mathcal A})-S(\mathcal B)$ change by equal amounts in opposite directions.
Each logistic loss term has second derivative at most $1/4$ with respect to its score difference, and each pair's probability of different task utilities is at most one.
Hence, the second derivative of $\mathcal L(S)$ with respect to the amount added to $S(\mathcal A^\star)$ is at most
\begin{equation}
\frac{2^2}{4}+\frac{2(|\mathcal S|-2)}4
=\frac{|\mathcal S|}{2}.
\label{eq:cgd-direction-curvature}
\end{equation}
We choose the amount added to $S(\mathcal A^\star)$ and subtracted from $S(\widehat{\mathcal A})$ as
\begin{equation}
\frac{2}{|\mathcal S|}
\left(\frac{\partial\mathcal L(S)}{\partial S(\widehat{\mathcal A})}
-\frac{\partial\mathcal L(S)}{\partial S(\mathcal A^\star)}\right).
\label{eq:cgd-score-adjustment}
\end{equation}
By applying the second-order Taylor bound and Eq.~\eqref{eq:cgd-direction-curvature} to these two score adjustments, we reduce the loss by at least the squared difference between the derivatives in Eq.~\eqref{eq:cgd-score-adjustment}, divided by $|\mathcal S|$.
Since the loss cannot decrease below $\mathcal L^\star$, we obtain
\begin{align}
\frac{1}{|\mathcal S|}
\left(\frac{\partial\mathcal L(S)}{\partial S(\widehat{\mathcal A})}
-\frac{\partial\mathcal L(S)}{\partial S(\mathcal A^\star)}\right)^2
\leq&\:\mathcal L(S)-\mathcal L^\star,
\label{eq:cgd-gradient-square-bound}\\
\left|\frac{\partial\mathcal L(S)}{\partial S(\widehat{\mathcal A})}
-\frac{\partial\mathcal L(S)}{\partial S(\mathcal A^\star)}\right|
\leq&\:\sqrt{|\mathcal S|\left(\mathcal L(S)-\mathcal L^\star\right)}.
\label{eq:cgd-gradient-difference-bound}
\end{align}

We now use the selection rule
$S(\widehat{\mathcal A})\geq S(\mathcal B)$ for all $\mathcal B\in\mathcal S$.
The fraction in Eq.~\eqref{eq:cgd-pairwise-gradient} increases with
the score difference and lies in $[-1,1]$.
For each $\mathcal B\notin\{\mathcal A^\star,\widehat{\mathcal A}\}$,
the fraction for $\widehat{\mathcal A}$ is nonnegative and no smaller
than the fraction for $\mathcal A^\star$.
We therefore have
\begin{align}
&\Pr\left(u(q_{\mathcal A^\star})\ne u(q_{\mathcal B})\right)
\frac{\exp\left(S(\mathcal A^\star)-S(\mathcal B)\right)-1}
{\exp\left(S(\mathcal A^\star)-S(\mathcal B)\right)+1}\nonumber\\
&-\Pr\left(u(q_{\widehat{\mathcal A}})\ne u(q_{\mathcal B})\right)
\frac{\exp\left(S(\widehat{\mathcal A})-S(\mathcal B)\right)-1}
{\exp\left(S(\widehat{\mathcal A})-S(\mathcal B)\right)+1}\nonumber\\
\leq&\:\max\left\{0,
\Pr\left(u(q_{\mathcal A^\star})\ne u(q_{\mathcal B})\right)
-\Pr\left(u(q_{\widehat{\mathcal A}})\ne u(q_{\mathcal B})\right)
\right\}.
\label{eq:cgd-pair-term-comparison}
\end{align}
For the direct comparison between $\mathcal A^\star$ and
$\widehat{\mathcal A}$, the corresponding difference is nonpositive
because $S(\mathcal A^\star)\leq S(\widehat{\mathcal A})$.
By subtracting their derivatives in Eq.~\eqref{eq:cgd-pairwise-gradient}
and applying Eqs.~\eqref{eq:cgd-gradient-difference-bound}
and~\eqref{eq:cgd-pair-term-comparison}, we obtain
\begin{align}
|\mathcal S|\,\mathbb E\left[V(\mathcal A^\star)-V(\widehat{\mathcal A})\right]
\leq&\:\sum_{\mathcal B\in\mathcal S\setminus
\{\mathcal A^\star,\widehat{\mathcal A}\}}
\max\left\{0,
\Pr\left(u(q_{\mathcal A^\star})\ne u(q_{\mathcal B})\right)
-\Pr\left(u(q_{\widehat{\mathcal A}})\ne u(q_{\mathcal B})\right)
\right\}\nonumber\\
&+2\sqrt{|\mathcal S|\left(\mathcal L(S)-\mathcal L^\star\right)}.
\label{eq:cgd-pairwise-comparison}
\end{align}

Finally, we express the probabilities of different task utilities using covariances.
For binary utilities, we have
\begin{align}
\Pr\left(u(q_{\mathcal A})\ne u(q_{\mathcal B})\right)
=&\:\mathbb E\left[u(q_{\mathcal A})\right]+\mathbb E\left[u(q_{\mathcal B})\right]
-2\mathbb E\left[u(q_{\mathcal A})u(q_{\mathcal B})\right]\nonumber\\
=&\:\mathbb E\left[u(q_{\mathcal A})\right]+\mathbb E\left[u(q_{\mathcal B})\right]
-2\mathbb E\left[u(q_{\mathcal A})\right]\mathbb E\left[u(q_{\mathcal B})\right]
-2\operatorname{Cov}\left(u(q_{\mathcal A}),u(q_{\mathcal B})\right).
\label{eq:cgd-disagreement-covariance}
\end{align}
By applying Eq.~\eqref{eq:cgd-disagreement-covariance} to
$\mathcal A^\star$ and $\widehat{\mathcal A}$ and subtracting, we have
\begin{align}
&\Pr\left(u(q_{\mathcal A^\star})\ne u(q_{\mathcal B})\right)
-\Pr\left(u(q_{\widehat{\mathcal A}})\ne u(q_{\mathcal B})\right)
\nonumber\\
=&\:\mathbb E\left[V(\mathcal A^\star)-V(\widehat{\mathcal A})\right]
\left(1-2\mathbb E\left[u(q_{\mathcal B})\right]\right)
+2\operatorname{Cov}\left(
u(q_{\widehat{\mathcal A}})-u(q_{\mathcal A^\star}),u(q_{\mathcal B})\right)
\nonumber\\
\leq&\:\mathbb E\left[V(\mathcal A^\star)-V(\widehat{\mathcal A})\right]
+2\max\left\{0,\operatorname{Cov}\left(
u(q_{\widehat{\mathcal A}})-u(q_{\mathcal A^\star}),u(q_{\mathcal B})\right)\right\}.
\label{eq:cgd-covariance-comparison}
\end{align}
We obtain the inequality because $\mathcal A^\star$ maximizes expected
utility and $1-2\mathbb E\left[u(q_{\mathcal B})\right]\leq1$.
Its right-hand side is nonnegative, so it also bounds the maximum
with zero in Eq.~\eqref{eq:cgd-pairwise-comparison}.
By substitution and moving the $|\mathcal S|-2$ copies of the expected
utility difference to the left, we have
\begin{align}
2\mathbb E\left[V(\mathcal A^\star)-V(\widehat{\mathcal A})\right]
\leq&\:2\sum_{\mathcal B\in\mathcal S\setminus
\{\mathcal A^\star,\widehat{\mathcal A}\}}
\max\left\{0,\operatorname{Cov}\left(
u(q_{\widehat{\mathcal A}})-u(q_{\mathcal A^\star}),u(q_{\mathcal B})\right)\right\}
\nonumber\\
&+2\sqrt{|\mathcal S|\left(\mathcal L(S)-\mathcal L^\star\right)}.
\label{eq:cgd-dependent-bound-final-step}
\end{align}
We divide both sides by two to complete the proof of
Lemma~\ref{thm:cgd-correlated-selection}.
\end{proof}

\paragraph{The comparison used in the bound.}
\label{app:cgd-selection-reference}
With the scoring inputs fixed, $\mathcal A^\star$ maximizes expected utility. Selecting the best combination after observing all task utilities defines a different comparison.
By adding and subtracting $V(\mathcal A^\star)$, we have
\begin{align}
&\mathbb E\left[\max_{\mathcal A\in\mathcal S}V(\mathcal A)
-V(\widehat{\mathcal A})\right]\nonumber
=\:\mathbb E\left[\max_{\mathcal A\in\mathcal S}V(\mathcal A)
-V(\mathcal A^\star)\right]
+\mathbb E\left[V(\mathcal A^\star)-V(\widehat{\mathcal A})\right].
\label{eq:cgd-selection-reference-decomposition}
\end{align}
The lemma bounds the second term. The first term also contributes when
we compare with the best combination after observing all utilities.

\subsection{CGD's Expected Task-Utility Difference}
\label{app:cgd-recovery-transfer}

\begin{proof}[Proof of Theorem~\ref{thm:cgd-recovery-ranking}]
We bound the change in the covariances in
Lemma~\ref{thm:cgd-correlated-selection} when we replace receiver
activations. We keep the scoring inputs, scores, and expected pairwise
loss $\mathcal L(S)$ unchanged. Since equal answers have equal binary utilities,
we have
\begin{equation}
\mathbb E\left|u(q_{\mathcal A})-u^{\rm act}_{\mathcal A}\right|
\leq\Pr(q^{\rm act}_{\mathcal A}\ne q_{\mathcal A})
\leq\max_{\mathcal B\in\mathcal S}
\Pr(q^{\rm act}_{\mathcal B}\ne q_{\mathcal B}).
\label{eq:cgd-utility-disagreement}
\end{equation}
By adding and subtracting
$\operatorname{Cov}(u^{\rm act}_{\mathcal A},u(q_{\mathcal B}))$, we have
\begin{align}
&\operatorname{Cov}\left(u(q_{\mathcal A}),u(q_{\mathcal B})\right)
-\operatorname{Cov}\left(u^{\rm act}_{\mathcal A},u^{\rm act}_{\mathcal B}\right)
\nonumber\\
=&\:\operatorname{Cov}\left(
u(q_{\mathcal A})-u^{\rm act}_{\mathcal A},u(q_{\mathcal B})\right)
+\operatorname{Cov}\left(
u^{\rm act}_{\mathcal A},u(q_{\mathcal B})-u^{\rm act}_{\mathcal B}\right)
\nonumber\\
=&\:\mathbb E\left[
\left(u(q_{\mathcal A})-u^{\rm act}_{\mathcal A}\right)
\left(u(q_{\mathcal B})-\mathbb E\left[u(q_{\mathcal B})\right]\right)\right]
+\mathbb E\left[
\left(u^{\rm act}_{\mathcal A}-\mathbb E\left[u^{\rm act}_{\mathcal A}\right]\right)
\left(u(q_{\mathcal B})-u^{\rm act}_{\mathcal B}\right)\right].
\label{eq:cgd-covariance-expansion}
\end{align}
The absolute difference between each binary utility and its expectation is at most one.
By applying the triangle inequality and
Eq.~\eqref{eq:cgd-utility-disagreement}, we obtain
\begin{align}
\left|\operatorname{Cov}\left(u(q_{\mathcal A}),u(q_{\mathcal B})\right)
-\operatorname{Cov}\left(u^{\rm act}_{\mathcal A},u^{\rm act}_{\mathcal B}\right)\right|
\nonumber
\leq&\:\mathbb E\left|u(q_{\mathcal A})-u^{\rm act}_{\mathcal A}\right|
+\mathbb E\left|u(q_{\mathcal B})-u^{\rm act}_{\mathcal B}\right|\nonumber\\
\leq&\:2\max_{\mathcal A\in\mathcal S}
\Pr(q^{\rm act}_{\mathcal A}\ne q_{\mathcal A}).
\label{eq:cgd-covariance-recovery-bound}
\end{align}
We apply Eq.~\eqref{eq:cgd-covariance-recovery-bound} separately with $\mathcal A=\widehat{\mathcal A}$ and $\mathcal A=\mathcal A^\star$, while keeping $\mathcal B$ fixed.
By adding these two bounds and applying the triangle inequality, we obtain
\begin{align}
&\max\left\{0,\operatorname{Cov}\left(
u(q_{\widehat{\mathcal A}})-u(q_{\mathcal A^\star}),u(q_{\mathcal B})\right)\right\}
\nonumber\\
\leq&\:\max\left\{0,\operatorname{Cov}\left(
u^{\rm act}_{\widehat{\mathcal A}}-u^{\rm act}_{\mathcal A^\star},
u^{\rm act}_{\mathcal B}\right)\right\}
+4\max_{\mathcal A\in\mathcal S}
\Pr(q^{\rm act}_{\mathcal A}\ne q_{\mathcal A}).
\label{eq:cgd-recovered-covariance-comparison}
\end{align}
For $\mathcal A^\star\ne\widehat{\mathcal A}$, we sum over the
$|\mathcal S|-2$ remaining combinations and substitute into
Lemma~\ref{thm:cgd-correlated-selection}:
\begin{align}
\mathbb E\left[V(\mathcal A^\star)-V(\widehat{\mathcal A})\right]
\leq&\:\sum_{\mathcal B\in\mathcal S\setminus
\{\mathcal A^\star,\widehat{\mathcal A}\}}
\max\left\{0,\operatorname{Cov}\left(
u^{\rm act}_{\widehat{\mathcal A}}-u^{\rm act}_{\mathcal A^\star},
u^{\rm act}_{\mathcal B}\right)\right\}
\nonumber\\
&+\sqrt{|\mathcal S|\left(\mathcal L(S)-\mathcal L^\star\right)}
+4(|\mathcal S|-2)\max_{\mathcal A\in\mathcal S}
\Pr(q^{\rm act}_{\mathcal A}\ne q_{\mathcal A}).
\label{eq:cgd-recovery-transfer-final}
\end{align}
For $\mathcal A^\star=\widehat{\mathcal A}$, the left-hand side is zero
and the inequality holds directly.
This completes the proof of Theorem~\ref{thm:cgd-recovery-ranking}.
\end{proof}

When the receiver selects the highest-scoring candidate answer, we apply
Theorem~\ref{thm:sequential-social-circuit} to each message combination.
The probability that activation replacement and the run using $\mathcal A$ return different answers is at most
$\Pr(\mathrm{err}\geq\mathrm{gap})$, using the no-message run as the
reference. If replacement preserves every answer, the last term in
Eq.~\eqref{eq:cgd-recovery-transfer-final} is zero.
Proposition~\ref{prop:sc-full-recovery} ensures this property when its
conditions on token positions and restored activations hold.

\begin{corollary}[Optimal message selection under equal covariance]
\label{cor:cgd-local-characterization}
If the task utilities of different message combinations have equal
covariance and finite scores satisfy $\mathcal L(S)=\mathcal L^\star$, then
$\widehat{\mathcal A}\in\argmax_{\mathcal A\in\mathcal S}
\mathbb E_{\mathcal P}[u(q_{\mathcal A})]$.
\end{corollary}

\begin{proof}[Proof of Corollary~\ref{cor:cgd-local-characterization}]
For each $\mathcal B\notin\{\mathcal A^\star,\widehat{\mathcal A}\}$,
equal covariance implies
\begin{align}
&\operatorname{Cov}\left(
u(q_{\widehat{\mathcal A}})-u(q_{\mathcal A^\star}),u(q_{\mathcal B})\right)
\nonumber\\
=&\:\operatorname{Cov}\left(u(q_{\widehat{\mathcal A}}),u(q_{\mathcal B})\right)
-\operatorname{Cov}\left(u(q_{\mathcal A^\star}),u(q_{\mathcal B})\right)=0.
\label{eq:cgd-equal-covariance-cancellation}
\end{align}
By substituting this equality and $\mathcal L(S)=\mathcal L^\star$
into Lemma~\ref{thm:cgd-correlated-selection}, we have
\begin{equation}
0\leq\mathbb E\left[V(\mathcal A^\star)-V(\widehat{\mathcal A})\right]
\leq\sqrt{|\mathcal S|\left(\mathcal L(S)-\mathcal L^\star\right)}=0.
\label{eq:cgd-optimal-selection}
\end{equation}
Thus CGD selects a combination with the highest expected utility.
\end{proof}

\paragraph{Unequal covariances.}
We construct an example with unequal covariances in which minimizing the pairwise loss does not maximize the expected utility.
We take two candidate messages, with $\mathcal A$ and $\mathcal B$
each containing one message. Thus $\mathcal C=\mathcal A\cup\mathcal B$
and $\mathcal S=\{\varnothing,\mathcal A,\mathcal B,\mathcal C\}$.
We consider that the binary utilities satisfy
\begin{equation}
\begin{gathered}
u(q_{\varnothing})=u(q_{\mathcal A})=1-u(q_{\mathcal B})
\leq u(q_{\mathcal C}),\\
0<\mathbb E\left[u(q_{\varnothing})\right]<\frac29,
\qquad
\mathbb E\left[u(q_{\mathcal C})\right]=1-2\mathbb E\left[u(q_{\varnothing})\right].
\end{gathered}
\label{eq:cgd-unequal-covariance-family}
\end{equation}
We assign probabilities $\mathbb E\left[u(q_{\varnothing})\right]$,
$2\mathbb E\left[u(q_{\varnothing})\right]$, and
$1-3\mathbb E\left[u(q_{\varnothing})\right]$ to the utility pairs
$(u(q_{\varnothing}),u(q_{\mathcal C}))=(1,1)$, $(0,0)$, and $(0,1)$,
respectively. These probabilities are positive and sum to one.
By Eq.~\eqref{eq:cgd-unequal-covariance-family}, we have
\begin{equation}
\mathbb E\left[u(q_{\mathcal B})\right]
=1-\mathbb E\left[u(q_{\varnothing})\right]
>\mathbb E\left[u(q_{\mathcal C})\right]
>\mathbb E\left[u(q_{\mathcal A})\right]
=\mathbb E\left[u(q_{\varnothing})\right].
\label{eq:cgd-unequal-utility-order}
\end{equation}
Thus $\mathcal B$ has the highest expected utility throughout this family.

With $S(\varnothing)=0$, exactly one set of finite scores minimizes the loss.
For each pair involving $\mathcal B$, either utility can exceed the other
with positive probability. Each pair loss is strictly convex in its score
difference and tends to infinity when this difference tends to either
infinity. With $S(\varnothing)=0$, any nonzero score change affects at
least one of these three differences, and any unbounded score sequence
makes at least one difference unbounded. Hence the continuous total loss
is minimized by a unique set of finite scores.

Since $u(q_{\mathcal A})=u(q_{\varnothing})$, exchanging their scores
leaves the loss unchanged. We can then subtract the new no-message
score from every score to restore $S(\varnothing)=0$.
By uniqueness of the minimum, we obtain
\begin{equation}
S(\mathcal A)=S(\varnothing)=0.
\label{eq:cgd-equal-zero-scores}
\end{equation}
At this minimum, we set the derivatives with respect to
$S(\mathcal B)$ and $S(\mathcal C)$ to zero.
By substituting the joint probabilities and rearranging, we have
\begin{align}
\frac{2}{1+\exp S(\mathcal B)}
=&\:\frac{3\mathbb E\left[u(q_{\varnothing})\right]}
{1+\exp\left(S(\mathcal C)-S(\mathcal B)\right)},
\label{eq:cgd-minimum-score-b}\\
\frac{2\left(1-3\mathbb E\left[u(q_{\varnothing})\right]\right)}
{1+\exp S(\mathcal C)}
=&\:\frac{3\mathbb E\left[u(q_{\varnothing})\right]}
{1+\exp\left(S(\mathcal B)-S(\mathcal C)\right)}
-\mathbb E\left[u(q_{\varnothing})\right].
\label{eq:cgd-minimum-score-c}
\end{align}
The right-hand side of Eq.~\eqref{eq:cgd-minimum-score-b} is less than
$3\mathbb E\left[u(q_{\varnothing})\right]<1$, so $S(\mathcal B)>0$.

We assume $S(\mathcal C)\leq S(\mathcal B)$ and derive a contradiction.
By substituting $\exp\left(S(\mathcal C)-S(\mathcal B)\right)\leq1$ into
Eq.~\eqref{eq:cgd-minimum-score-b}, we have
\begin{equation}
\frac{1}{1+\exp S(\mathcal B)}
\geq\frac{3\mathbb E\left[u(q_{\varnothing})\right]}4.
\label{eq:cgd-score-b-lower-bound}
\end{equation}
We combine this inequality with Eq.~\eqref{eq:cgd-minimum-score-c}:
\begin{align}
\frac{3\mathbb E\left[u(q_{\varnothing})\right]}2
\left(1-3\mathbb E\left[u(q_{\varnothing})\right]\right)
\overset{(\text{a})}{\leq}&\:
\frac{2\left(1-3\mathbb E\left[u(q_{\varnothing})\right]\right)}
{1+\exp S(\mathcal C)}\nonumber\\
=&\:\frac{3\mathbb E\left[u(q_{\varnothing})\right]}
{1+\exp\left(S(\mathcal B)-S(\mathcal C)\right)}
-\mathbb E\left[u(q_{\varnothing})\right]\nonumber\\
\overset{(\text{b})}{\leq}&\:
\frac{\mathbb E\left[u(q_{\varnothing})\right]}2.
\label{eq:cgd-score-order-contradiction}
\end{align}
Inequality (a) follows from Eq.~\eqref{eq:cgd-score-b-lower-bound}
and $1+\exp S(\mathcal C)\leq1+\exp S(\mathcal B)$.
Inequality (b) follows from
$\exp\left(S(\mathcal B)-S(\mathcal C)\right)\geq1$.
By dividing by $\mathbb E\left[u(q_{\varnothing})\right]/2>0$ and rearranging,
we obtain $\mathbb E\left[u(q_{\varnothing})\right]\geq2/9$, which contradicts
Eq.~\eqref{eq:cgd-unequal-covariance-family}.
Hence,
\begin{equation}
S(\mathcal C)>S(\mathcal B)>S(\mathcal A)=S(\varnothing)=0.
\label{eq:cgd-unequal-score-order}
\end{equation}
CGD therefore selects $\mathcal C$ at the loss minimum, although
$\mathcal B$ has higher expected utility. The difference is
\begin{equation}
\mathbb E\left[V(\mathcal B)-V(\mathcal C)\right]
=\mathbb E\left[u(q_{\varnothing})\right]>0.
\label{eq:cgd-unequal-selection-loss}
\end{equation}
Finally, the covariances in this family satisfy
\begin{equation}
\operatorname{Cov}\left(u(q_{\varnothing}),u(q_{\mathcal A})\right)
=-\operatorname{Cov}\left(u(q_{\varnothing}),u(q_{\mathcal B})\right)
=\mathbb E\left[u(q_{\varnothing})\right]
\left(1-\mathbb E\left[u(q_{\varnothing})\right]\right)>0.
\label{eq:cgd-unequal-covariances}
\end{equation}
Thus, unequal covariance can lead to a positive expected task-utility
difference even when finite scores minimize the pairwise loss.

\section{Algorithm}
\label{app:algorithms}
\begingroup
\raggedbottom

Algorithm~\ref{alg:social-circuit-confirmation} compares a message's
effect with the effect reproduced by receiver activation replacement.
Algorithm~\ref{alg:cgd-summary} uses the sender information and
receiver changes to learn which message combination improves
task performance.

\begin{algorithm}[h]
\caption{\method{}: message and activation effects}
\label{alg:social-circuit-confirmation}
\renewcommand{\algorithmicrequire}{\textbf{Input:}}
\renewcommand{\algorithmicensure}{\textbf{Output:}}
\setlength{\fboxsep}{2pt}
\newcommand{\scphase}[2]{%
\par\noindent\colorbox[gray]{0.95}{%
\begin{minipage}{\dimexpr\linewidth-2\fboxsep\relax}
\begin{algorithmic}[1]
\setcounter{ALG@line}{#1}
#2
\end{algorithmic}
\end{minipage}}\par\vspace{3pt}}

\begin{algorithmic}[1]
\Require Task and receiver; original and reference messages; selected layers and post-message prompt positions; candidate answers $\mathcal O$; task utility $u$.
\end{algorithmic}
\scphase{0}{
\Statex \textbf{Matched reruns}
\State Fix the task, earlier messages, receiver model, and generation seed.
\State Run the receiver with the reference message; record $F^{\rm ref}(o)$ for each $o\in\mathcal O$.
\State Run the receiver with the original message; record $F^{\rm msg}(o)$ and the activations at the selected layers and positions.
}
\scphase{3}{
\Statex \textbf{Activation replacement}
\State Repeat the reference run, replacing the selected activations with their original-message values; record $F^{\rm act}(o)$.
}
\scphase{4}{
\Statex \textbf{Effect calculation}
\State Select $q^{\rm ref}$, $q^{\rm msg}$, and $q^{\rm act}$ by maximizing their respective answer scores, using the same tie rule.
\State Calculate the message and activation effects on each answer score:
\Statex \hspace{\algorithmicindent}$\begin{aligned}
\Delta F^{\rm msg}(o)&=F^{\rm msg}(o)-F^{\rm ref}(o),\\
\Delta F^{\rm act}(o)&=F^{\rm act}(o)-F^{\rm ref}(o).
\end{aligned}$
\State Calculate the corresponding changes in task utility:
\Statex \hspace{\algorithmicindent}$\begin{aligned}
\Delta u^{\rm msg}&=u(q^{\rm msg})-u(q^{\rm ref}),\\
\Delta u^{\rm act}&=u(q^{\rm act})-u(q^{\rm ref}).
\end{aligned}$
}
\begin{algorithmic}[1]
\Ensure $\Delta F^{\rm msg}$, $\Delta F^{\rm act}$, $\Delta u^{\rm msg}$, and $\Delta u^{\rm act}$.
\end{algorithmic}
\end{algorithm}

\begin{algorithm}[h]
\caption{Circuit-Guided Deliberation (\cgd{})}
\label{alg:cgd-summary}
\renewcommand{\algorithmicrequire}{\textbf{Input:}}
\renewcommand{\algorithmicensure}{\textbf{Output:}}
\setlength{\fboxsep}{2pt}
\newcommand{\scphase}[2]{%
\par\noindent\colorbox[gray]{0.95}{%
\begin{minipage}{\dimexpr\linewidth-2\fboxsep\relax}
\begin{algorithmic}[1]
\setcounter{ALG@line}{#1}
#2
\end{algorithmic}
\end{minipage}}\par\vspace{3pt}}
\begin{algorithmic}[1]
\Require Training tasks with task utility $u$; a new task; sending and receiving agents.
\end{algorithmic}
\scphase{0}{
\Statex \textbf{Training}
\ForAll{training tasks}
  \State \label{alg:cgd-input-start} Record sender activation summaries; generate candidate messages $\mathcal C$ once.
  \State Form $\mathcal S=\{\mathcal A\mid\mathcal A\subseteq\mathcal C\}$, including $\varnothing$.
  \ForAll{$\mathcal A\in\mathcal S$}
    \State Run the receiver with $\mathcal A$, keeping the task and earlier messages fixed; record $\mathbf r_{\mathcal A}$ before answering and save $q_{\mathcal A}$.
  \EndFor
  \State \label{alg:cgd-input-end} Construct $\mathbf s_{\mathcal A}$ and the scoring input $\mathbf x_{\mathcal A}$ for nonempty combinations.
  \State Set $V(\mathcal A)=u(q_{\mathcal A})-u(q_{\varnothing})$ for every $\mathcal A\in\mathcal S$.
\EndFor
\State Keep the language models fixed; fit the shared $f_{\boldsymbol\theta}$ using the gain-weighted pairwise loss in Eq.~\eqref{eq:cgd-pairwise-loss}, with $S(\varnothing)=0$ and a squared-weight penalty that leaves the bias unpenalized.
}
\scphase{10}{
\Statex \textbf{Selection for a new task}
\State Record sender activation summaries; generate candidate messages $\mathcal C$ once.
\State Form $\mathcal S=\{\mathcal A\mid\mathcal A\subseteq\mathcal C\}$, including $\varnothing$.
\ForAll{$\mathcal A\in\mathcal S$}
  \State Keep the task and earlier messages fixed. Read $\mathcal A$ with the receiver and record $\mathbf r_{\mathcal A}$ before generating any answer.
\EndFor
\State Construct $\mathbf s_{\mathcal A}$ and the scoring input $\mathbf x_{\mathcal A}$ for nonempty combinations.
\State Set $S(\varnothing)=0$ and $S(\mathcal A)=f_{\boldsymbol\theta}(\mathbf x_{\mathcal A})$ for nonempty $\mathcal A\in\mathcal S$.
\State Select $\widehat{\mathcal A}\in\arg\max_{\mathcal A\in\mathcal S}S(\mathcal A)$.
\State If scores tie, select $\varnothing$ when available; otherwise use the fixed combination order.
\State Obtain the receiver's answer $q_{\widehat{\mathcal A}}$ using $\widehat{\mathcal A}$, as specified in Appendix~\ref{app:protocol}.
}
\begin{algorithmic}[1]
\Ensure Selected combination $\widehat{\mathcal A}$ and receiver answer $q_{\widehat{\mathcal A}}$.
\end{algorithmic}
\end{algorithm}

\FloatBarrier
\endgroup

\section{Experimental protocol details}
\label{app:protocol}

\subsection{Tasks and settings}
\label{app:protocol-settings}
We use the tasks and models in Section~\ref{sec:experimental_setup} and the synthetic Correlated-Trap task, where several agents share a passage supporting an incorrect answer and another receives evidence for the correct answer.
Correlated-Trap comparisons also include Gemma~4 E4B-IT~\citep{gemmateam2026gemma4}.
On this task, \textbf{Fixed sender} selects the agent given different evidence.
Table~\ref{tab:protocol-data} distinguishes the evaluation sets.
For matched interventions, we keep the task input, model, earlier messages,
and generation seed unchanged. We use seeds 0/1/2 for CGD,
10/11/12 for joint replacements, and
20/21/22 for mean changes.

\begin{table}[t]
\centering
\caption{Evaluation sets. Counts are per experiment.}
\label{tab:protocol-data}
\small
\begin{tabular}{@{}>{\raggedright\arraybackslash}p{0.48\linewidth}>{\raggedright\arraybackslash}p{0.23\linewidth}>{\raggedright\arraybackslash}p{0.23\linewidth}@{}}
\toprule
Comparison & Task & Questions\\
\midrule
Collaboration approaches & Correlated-Trap & 600\\
Message content and repetition & Correlated-Trap & 240 per direction\\
Activation replacement & Correlated-Trap & 240 per direction\\
Message-specific replacement & 2Wiki & 240, two candidates each\\
Two-option message selection & 2Wiki & 240 per test set\\
Two-answer scoring & MATH-500 & 200\\
Two-answer scoring & TruthfulQA & 200\\
Free-generation performance & 2Wiki, GSM8K & 240 per dataset\\
\bottomrule
\end{tabular}
\end{table}
\paragraph{Agent inputs.}
All agents receive the same task description. On 2Wiki, each agent reads its assigned passages.
For two-answer scoring, we assign supporting passages to the three agents in turn.
Each remaining passage goes to the agent with the fewest passages. Ties follow the agent order.
For free generation, we assign all passages in their original order, cycling through the agents.
GSM8K uses the same problem without extra passages.
MATH-500 and TruthfulQA use the same question and answer options.
Agent instructions specify independent reasoning, verification, or final answering.

\begin{table}[t]
\centering
\caption{Information used by each approach.}
\label{tab:agent-inputs}
\small
\begin{tabular}{@{}>{\raggedright\arraybackslash}p{0.19\linewidth}>{\raggedright\arraybackslash}p{0.77\linewidth}@{}}
\toprule
Approach & Information\\
\midrule
Single, SC6 & The designated receiver's assigned input, without messages from other agents.\\
Majority vote & Independent answers from three agents, each using its own assigned input.\\
MAD, MADC & Each agent's assigned input and the responses exchanged according to the communication procedure.\\
MAD-M2 & Each agent's assigned input and the responses retained by memory masking.\\
MOC & The receiver's assigned input, a summary of the two senders' messages, and an updated sender response.\\
SafeSieve & Each agent's assigned input and messages received through the learned communication graph.\\
CGD & The receiver's assigned input and the selected sender messages. Message-selection baselines use the same messages and recorded receiver answers.\\
\bottomrule
\end{tabular}
\end{table}

In the 2Wiki two-answer experiments, CGD messages include the sender's passages and response.
The collaboration baselines exchange generated responses.
In free generation, both include the sender's passages in their initial messages.

\paragraph{Answer options.}
For two-answer scoring, each task has one correct answer and one incorrect answer.
For 2Wiki, we use tasks whose correct answer is a passage title. 
We select the incorrect answer alphabetically from the other titles, excluding those of passages used as supporting evidence.
For MATH-500, we use tasks with integer answers. 
We construct the incorrect answer by adding 1 to a nonnegative correct answer or subtracting 1 from a negative correct answer.
For TruthfulQA, we use the dataset's correct answer and its first listed incorrect answer.
We alternate the correct answer between answer A and B across tasks.

\paragraph{MOC implementation.} We adapt MOC~\citep{guan2026moc} to six model calls with one summary and a fixed communication order. After three initial responses, we summarize the two senders' messages. One sender updates its answer using this summary, and the receiver reads the summary and updated answer before generating its final answer. This fixed sequence replaces the original five-summary merger, graph-based execution, and separate final decision node. \subsection{Message and activation comparisons}
\label{app:protocol-replacement}
For message-content comparisons on Gemma-3-4B, we use \textbf{No message}, \textbf{Original message}, \textbf{Facts and reasoning only}, \textbf{Suggested answer only}, \textbf{Unrelated message}, and \textbf{Opposite-answer message}.
For repetition, we deliver one to seven copies, paraphrases, or independently
stated pieces of evidence. We evaluate correction and echo chambers separately.

For joint replacement, we shuffle message positions in a single cycle,
preserving tokens and length. In the shuffled-message run, we restore the
original-message activations at fixed prompt tokens after the message,
including the chat footer but excluding the message and answer prefix.
We apply these replacements at six layers together, at 15\%, 30\%, 45\%, 60\%, 75\%, and 90\%
of decoder depth. For Swapped-role change, we use the opposite communication
direction on the same question and seed, match the change's norm at each layer,
and add it to the shuffled-message activations of the receiver being studied.

For layerwise comparisons, we edit activations at the last prompt token at one
layer per run. We add the Original-minus-Empty change to Empty and subtract
it from Original. We pair questions with the same template, correct label, and
independent-agent position. The control uses the paired question's change from
the same role, seed, and layer, matched in norm. For mean changes, we estimate directions and
magnitudes on 120 separate questions and apply them across six layers
on 240 new questions. Mean swapped-role change uses the reversed communication direction on these separate questions.

For 2Wiki replacement, both senders and the receiver read the same summary
of the initial answers. We add each candidate separately and replace activations
at the last token of the common answer prefix across six layers, using no
candidate message as the reference. Other-message change uses the other candidate's
change, matched in norm at each layer. We calculate both MAEs where the
candidate messages are nonempty and distinct: 422/428/418 samples for Gemma
and 444/446/444 for Qwen across seeds. Effects use the correct answer's log
probability normalized over the two options.

\subsection{CGD training and comparisons}
\label{app:protocol-cgd}
We use 340 2Wiki training questions and 100 validation questions, with seed 0.
Two-option and free-generation experiments collect separate outputs and train
separate scoring functions. We exclude evaluation questions from both sets.
Two-option evaluation uses the original 240 questions for all three models
and another 240 for the two 4B models, without retraining.
Free generation uses separate sets of 240 questions each from 2Wiki and
GSM8K, without GSM8K training.

We set $d=64$ and use zero-based layers 5, 10, 15, 20, 25, and 30 for Gemma,
and 5, 10, 16, 21, 26, and 32 for Qwen, with projection seed 2026090301.
We read the last common answer-prefix token for two-option tasks in CGD and the last
prompt token for free generation.
For Text-only, we convert the message tokens and their positions into a 64-dimensional vector.
We normalize this vector to unit length and use the message token count as an additional input.
CGD and the learned message-selection baselines use the same training
tasks and task utilities within each experiment.
We fit their scoring functions using the pairwise loss in
Section~\ref{sec:cgd-message-selection} and L-BFGS-B. We standardize inputs using the training data.
The two-option experiment uses weight penalty 1.
Free generation uses exact-match success on development answers and selects
the penalty from $\{0.01,0.1,1,10,100\}$ by validation accuracy;
all three CGD models select 100.

We keep CGD unchanged and adjust Activation products to use the same scaling and penalty.
Both approaches use the same messages and training data.
We report the comparison in Table~\ref{tab:cgd-matched-products}.

Learning curves keep penalty 1 and use the same nested subsets of 50, 100,
200, and 340 questions across selectors, with three draws below 340.

\subsection{Scoring and cost}
\label{app:protocol-scoring}

\paragraph{Answer scores.}
For complete-answer scoring, we calculate
\begin{equation}
F(o)=\sum_{\ell=1}^{L}\log p(y_\ell\mid P,y_{<\ell}),
\label{eq:complete-answer-score}
\end{equation}
where $P$ is the prompt, $o=(y_1,\ldots,y_L)$ is the answer,
$L$ is its token count, $y_{<\ell}$ denotes the preceding answer tokens,
and $p$ is the receiver's probability for each token.
For 2Wiki activation comparisons, we calculate the ratio of the
correct answer's probability to the sum of both answer probabilities
and use its log as the answer score.
Single-token scoring uses the first tokens at which the candidate
answers are different.
For $\rho$, the message and activation effects use the same reference
input, as specified in Appendix~\ref{app:protocol-replacement}.

\paragraph{Task performance.}
All approaches use the same tasks within each comparison.
Accuracy changes use No message as the reference.
Two-answer scoring selects between supplied answers.
Free generation uses no supplied candidate answers.
For free generation, we use exact match with any provided correct
answer for 2Wiki and compare the final answer number with the correct
number for GSM8K.
The code provides the generation settings.
We use limits of 256 generated tokens per model call for 2Wiki
two-answer scoring and 512 for 2Wiki free generation, GSM8K,
MATH-500, and TruthfulQA.

\paragraph{Two-answer comparisons.}
On 2Wiki, CGD selects the higher-scoring answer for its selected
message combination.
SC6 and Majority vote use majority voting.
If voting does not select a single answer, we select A.
The other collaboration baselines use the receiver's final A/B answer.
If the receiver does not provide A or B, we select the higher-scoring
answer.
In Table~\ref{tab:2wiki-answer-scoring}, we also use the higher-scoring
A/B answer for all 6 baselines that do not use voting.
We keep each task, message, and prompt unchanged.
If the scores are equal, we select A.

\paragraph{Generation cost.}
For free generation and the MATH-500 and TruthfulQA A/B comparisons,
the receiver generates one final answer using the selected message
combination.
Generation cost counts generated tokens in the sender messages
and the final answer when generated.
We generate answers for the other combinations after message selection
and record these outputs separately for message-selection comparisons.
CGD also scores answers and records activations in 2 sender and
4 receiver runs.
Generated token counts do not include the calculation for these
scores and activations.

\paragraph{Token limits.}
For the 2Wiki comparison and the separate 555-run GSM8K comparison,
each baseline uses CGD's generated token count as its total limit
for the same task and seed.
Single, SC6, and Majority vote generate answers in order.
After each answer, we subtract its token count from the remaining limit.

For the other baselines, before each message-generation run,
we subtract the tokens already generated and CGD's final-answer
 token count from the total limit.
We use the ratio of this remaining count to the number of remaining
message-generation runs as the limit for that run.
The final answer can use all tokens remaining under the total limit.
We use a limit of 512 tokens per model call and include all generated
answers in the accuracy calculation, complete or not.

\clearpage
\makeatletter
\setlength{\@fptop}{0pt}
\makeatother
\section{Additional Results} 
We provide additional results on message effects and receiver activation replacement. We compare CGD with message-selection baselines using the same candidate messages and report learning curves. We then compare collaboration approaches in task performance and generated-token use. Examples illustrate how CGD combines useful messages and excludes misleading information.\par  \subsection{How additional messages correct or reinforce errors}\label{sec:additional-messages}

\begin{figure}[!t]
    \centering
    \begin{minipage}[b]{0.48\textwidth}
        \centering
        \includegraphics[width=\textwidth]{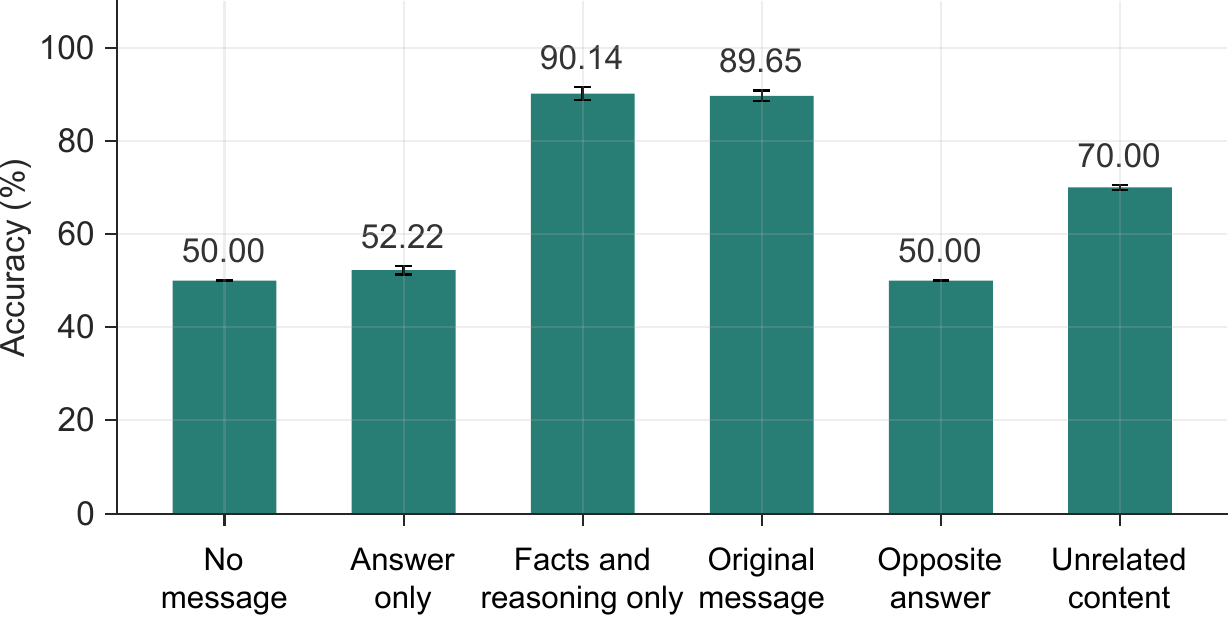} 
        \caption{Accuracy across six message-content conditions on Gemma-3-4B-it. Bars and error bars show the mean and standard deviation.}
        \label{fig:additional-content-accuracy}
    \end{minipage}
    \hfill 
    \begin{minipage}[b]{0.48\textwidth}
        \centering
        \includegraphics[width=\textwidth]{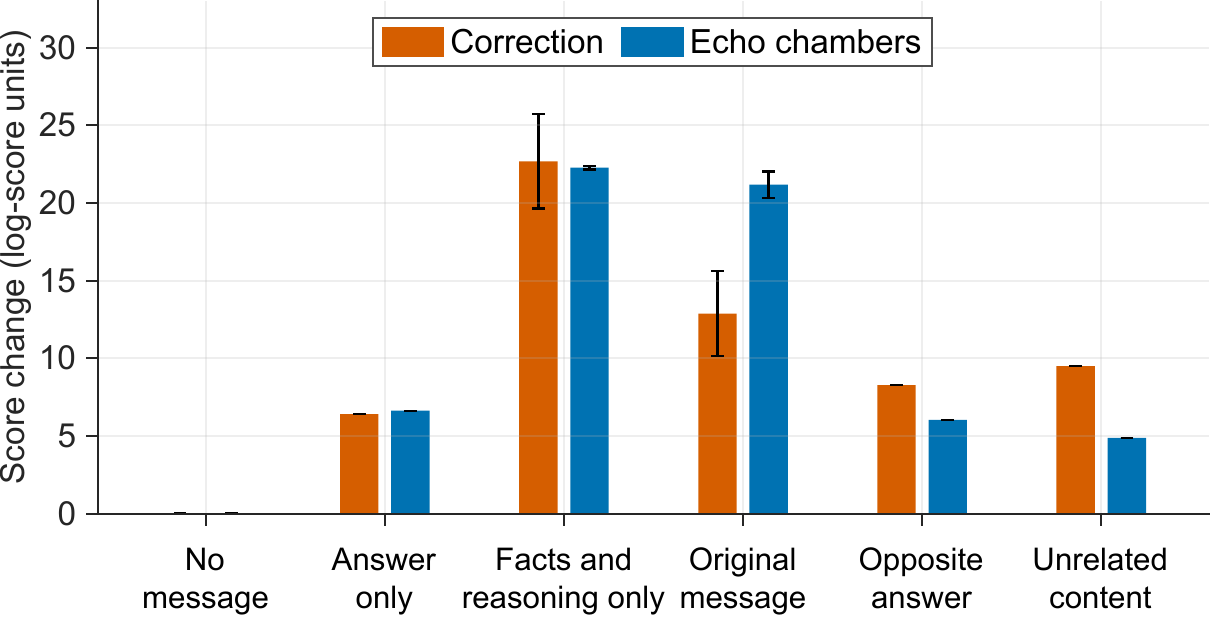} 
        \caption{Message effects on Gemma-3-4B-it. Error bars show the standard deviation across seeds.}
        \label{fig:additional-content-effects}
    \end{minipage}
\end{figure}

\paragraph{Message content.}
To identify which parts of a message affect the receiver, we present six message-content conditions in Fig.~\ref{fig:additional-content-accuracy}.
Facts and reasoning alone achieve an accuracy of $90.14\pm1.42\%$, close to the original message's $89.65\pm1.15\%$ and above $52.22\pm0.87\%$ for the suggested answer alone.
We further separate runs by whether the receiver answers correctly before receiving a message in Fig.~\ref{fig:para_1}.
When its initial answer is incorrect, facts and reasoning preserve most of the original message's corrective effect, whereas suggested answers alone have little effect and unrelated messages correct fewer errors.
When its initial answer is correct, the original message and its facts and reasoning introduce errors, whereas suggested answers and unrelated messages leave accuracy unchanged.
Fig.~\ref{fig:additional-content-effects} shows larger changes in answer scores with facts and reasoning than with the suggested answer alone, both when correcting mistakes and when reinforcing a shared error.
These comparisons show that both correction and the reinforcement of shared errors depend mainly on the supplied facts and reasoning, rather than on the suggested answer alone.

\begin{figure}[!t] \centering \begin{minipage}[t]{0.45\textwidth} \centering \includegraphics[width=\linewidth,trim=8bp 40bp 10bp 78bp,clip]{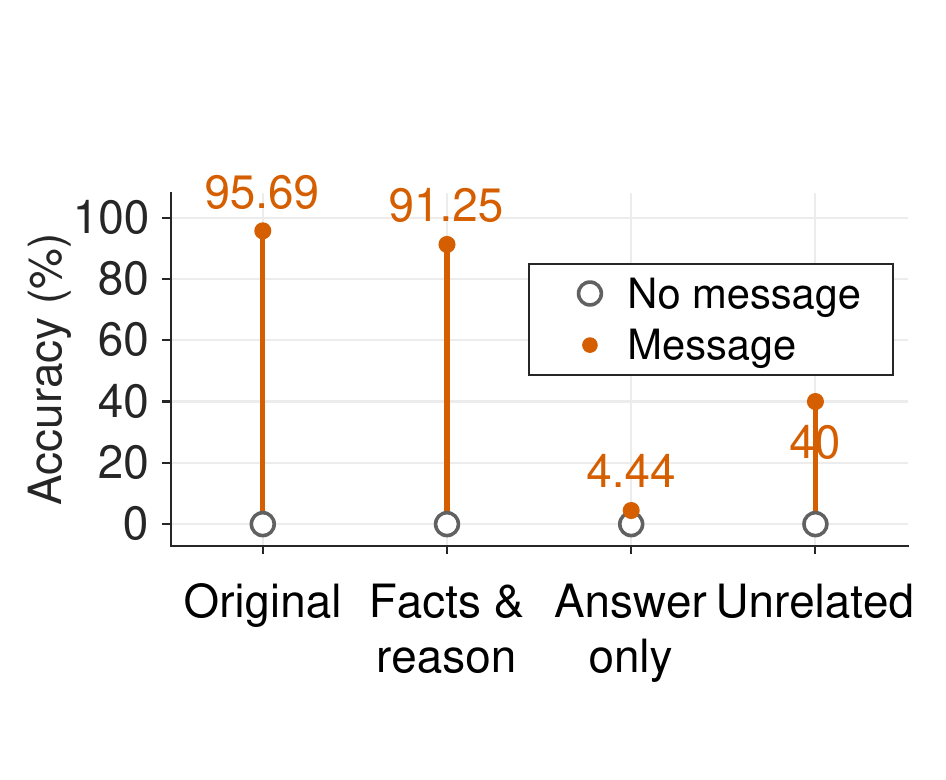}\\ {\footnotesize (a)} \end{minipage}\hfill \begin{minipage}[t]{0.45\textwidth} \centering \includegraphics[width=\linewidth,trim=8bp 40bp 10bp 78bp,clip]{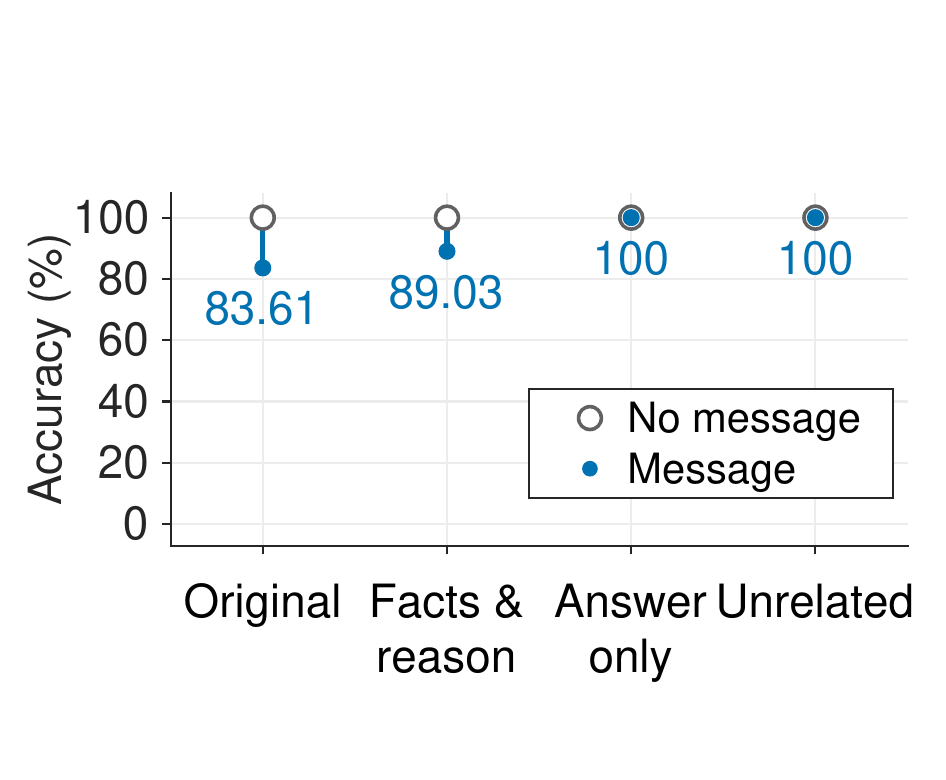}\\ {\footnotesize (b)} \end{minipage} \setlength{\abovecaptionskip}{4pt} \caption{Receiver accuracy when the answer without a message is (a) incorrect or (b) correct.} \label{fig:para_1} \end{figure}

\FloatBarrier

\paragraph{Additional messages.}
We study how message effects change as one receiver receives
more messages.
Fig.~\ref{fig:additional_para1} compares repeating the same
message, rephrasing its content, and providing different evidence
from different senders.
At seven messages, all three conditions correct more than 99\%
of initially incorrect answers.
Their effects on initially correct answers differ:
repeated messages introduce errors in 6.25\%,
compared with 73.19\% for rephrased messages and 95.14\% for different evidence.
Thus, similar correction rates may have different rates of switching to a shared incorrect answer.
These results show how additional communication can reinforce a shared error even when the same communication strategy is
effective at correcting mistakes.

\begin{figure}[!t]
    \centering
    \begin{minipage}[t]{0.4\textwidth}
        \centering
        \includegraphics[width=\linewidth]{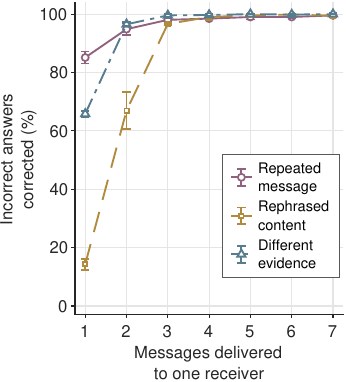}\par
        {\footnotesize (a)\par}
    \end{minipage}
    \begin{minipage}[t]{0.4\textwidth}
        \centering
        \includegraphics[width=\linewidth]{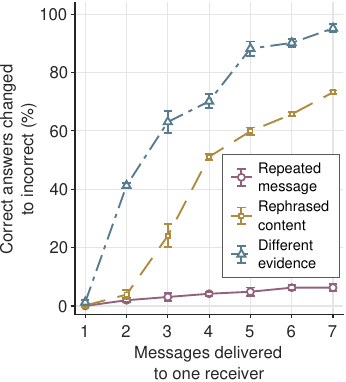}\par
        {\footnotesize (b)\par}
    \end{minipage}
    \caption{Effects of additional messages: (a) correcting an incorrect answer; (b) switching from a correct answer to a shared incorrect answer. }
    \label{fig:additional_para1}
\end{figure}

\FloatBarrier
\begin{figure}[!t]
\centering
\begin{minipage}[t]{0.49\linewidth}
\centering
\includegraphics[width=\linewidth]{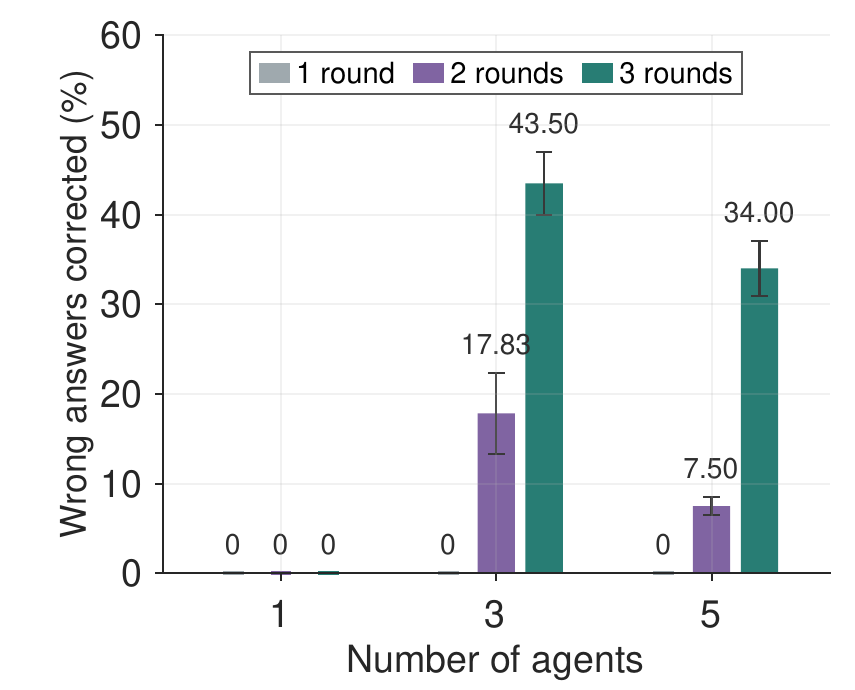}\\
{\footnotesize (a)}
\end{minipage}\hfill
\begin{minipage}[t]{0.49\linewidth}
\centering
\includegraphics[width=\linewidth]{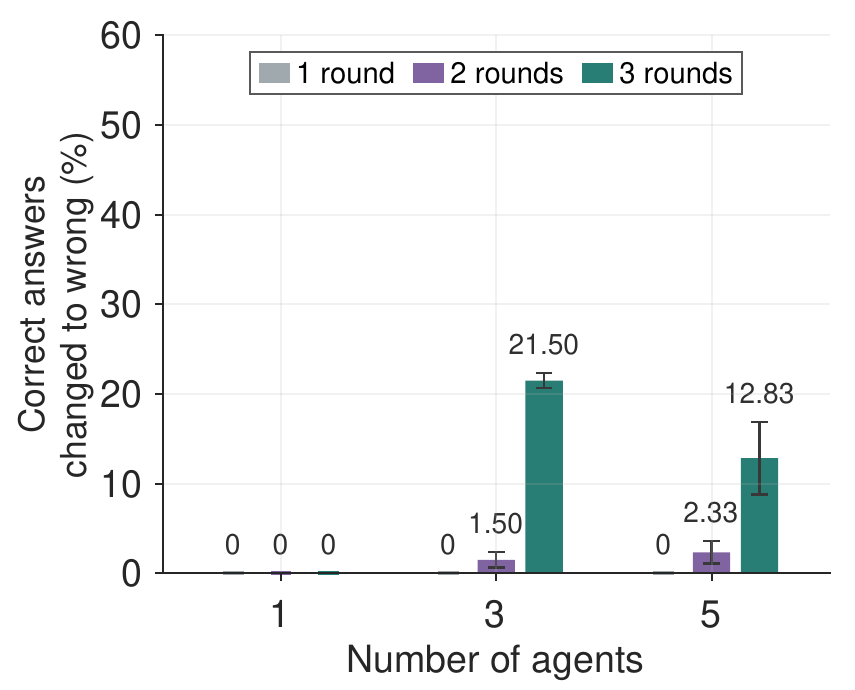}\\
{\footnotesize (b)}
\end{minipage}
\caption{Gemma-3-4B with different agent counts and rounds. (a) Wrong answers corrected. (b) Correct answers changed to wrong answers. The first round uses independent answers, and the one-agent scenario uses self-review in later rounds.}
\label{fig:additional_team_scale}
\end{figure}

\subsection{Agent number and communication structure}\label{sec:additional-team}
\paragraph{Agent number and communication rounds.}
To study how agent number and communication rounds affect error correction and echo chambers, we present the Gemma-3-4B results on Correlated-Trap in Fig.~\ref{fig:additional_team_scale}.
We use the same 200 tasks per direction under the MAD protocol.
For the three- and five-agent groups, one agent receives independent evidence, while the remaining agents share correlated evidence.
The first round consists of independent answers, followed by message exchange in later rounds.
For three agents, increasing the number of rounds from two to three raises error correction from 17.83\% to 43.50\%, but also raises the fraction of correct answers changed to wrong answers from 1.50\% to 21.50\%.
Within two rounds, increasing the number of agents from three to five reduces error correction to 7.50\%.
These results show that additional rounds can improve correction while strengthening echo chambers, and adding agents that share correlated evidence can reduce correction.
This supports measuring the message effects separately in \method{}.

\paragraph{Communication Topology.}
We compare three communication topologies with Qwen3-8B and two rounds in Fig.~\ref{fig:additional_team_topology}.
In the full topology, each agent receives messages from all other agents.
In a ring topology, each agent receives its predecessor's message.
In a star topology, a central agent exchanges messages with the other agents.
With three agents, the full topology achieves the highest accuracy of 73.25\%, compared with 67.17\% for the ring and 69.00\% for the star.
With five agents, the ring and star achieve 58.33\% and 57.00\%, exceeding the 51.58\% of the full topology.
These results show that higher communication density does not consistently improve task performance.
The effect of increasing the number of agents depends on which messages reach each receiver.
\begin{figure}[!t]
\centering
\includegraphics[width=0.5\linewidth]{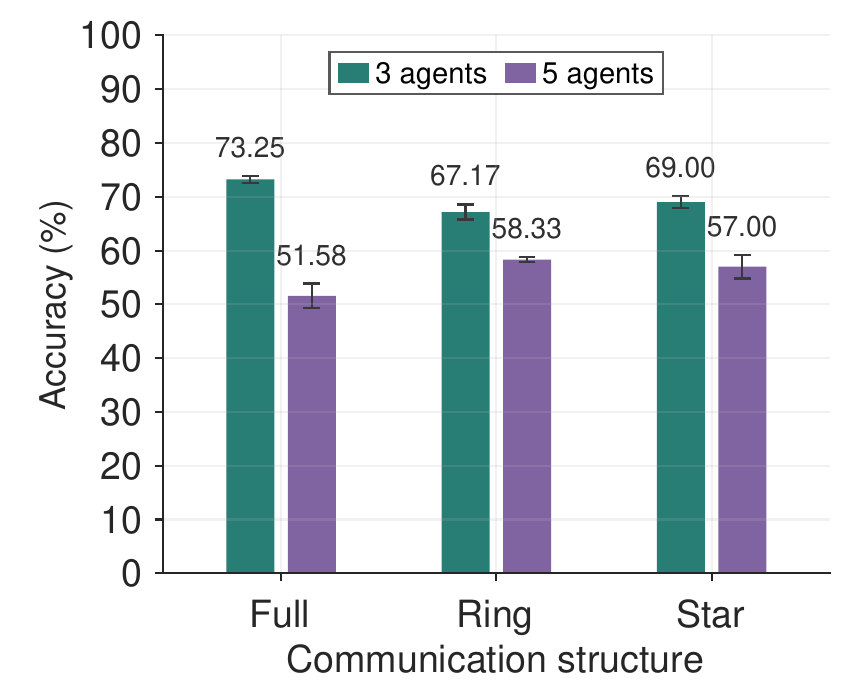}
\caption{Qwen3-8B receiver accuracy with different communication topologies and agent numbers. }
\label{fig:additional_team_topology}
\end{figure}
\FloatBarrier
\subsection{Receiver activation replacement}\label{sec:additional-activation} 
\paragraph{Activation replacement across models.}
We present message-effect recovery across three models in Fig.~\ref{fig:additional_D2_1}.
The single-token score uses the first token where the candidate answers differ, while the complete-answer score sums the token log probabilities of each answer.
Results show that receiver activation changes achieve higher recovery than Swapped-role changes across all three models under both scoring approaches.
Recovery is positive for both correction and echo chambers in Gemma-3-4B and Qwen3-8B, but only for echo chambers in Qwen3-4B.
These comparisons show that the sender--receiver direction matters when tracing message effects, and that recovery differs between correcting errors and reinforcing shared errors.

\begin{figure}[!t]
    \centering
    \includegraphics[width=0.8\linewidth]{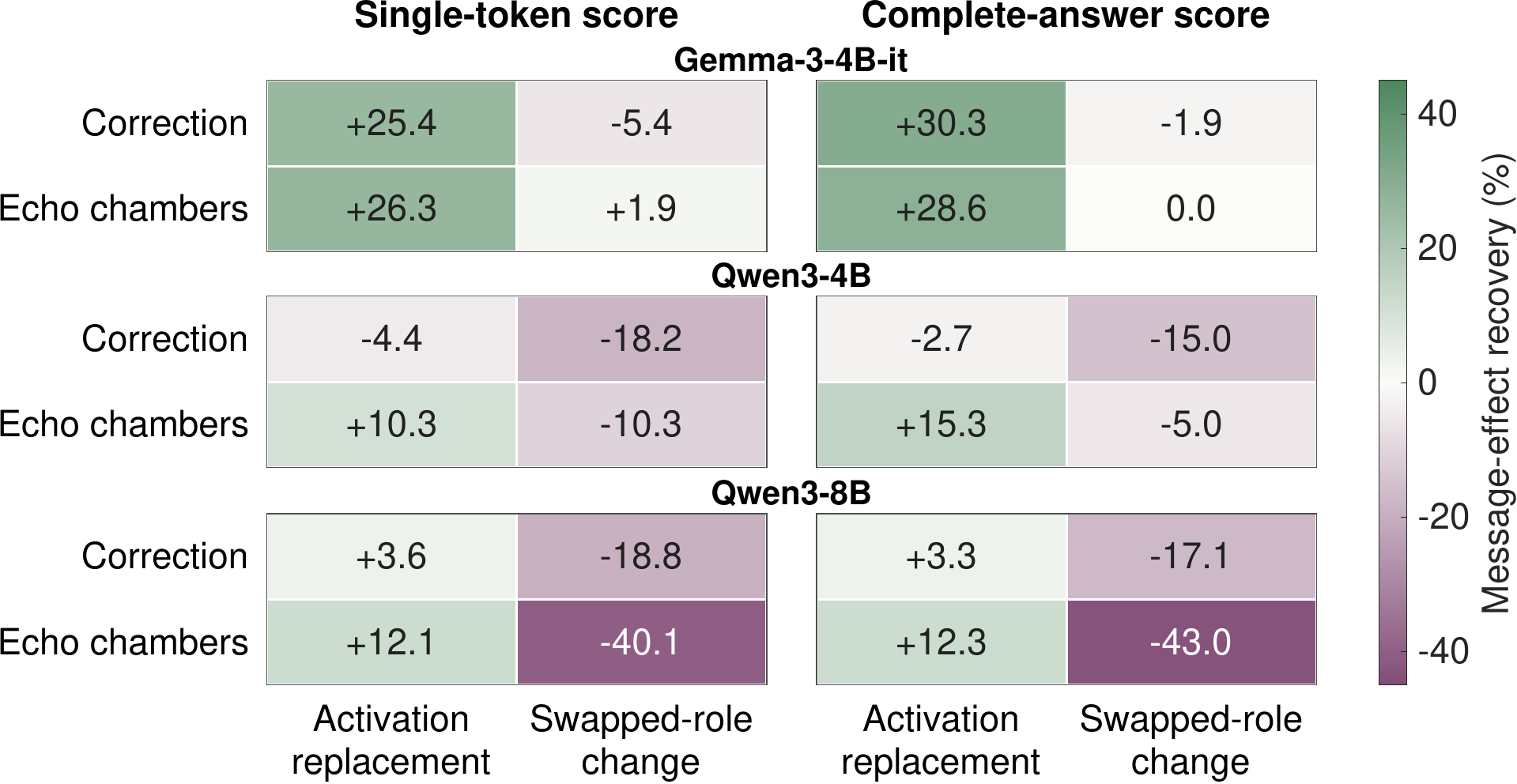}
    \caption{Message-effect recovery across three models. 
    Values are 100 times the mean replacement effect divided by the mean message effect, using 240 tasks per direction, averaged across seeds. 
    Positive values follow the original message effect, while negative values reverse it.}
    \label{fig:additional_D2_1}
\end{figure}

\paragraph{Activation changes from other tasks.}
We compare receiver activation changes from the current task with those from another task on Gemma-3-4B-it in Fig.~\ref{fig:additional_D2_2}(a), replacing activations at one layer per run.
Results show that changes from the current task recover more of the message effect on average, with the largest difference at layer 10.
Both replacements achieve the highest recovery at layer 15 and reverse the message effect at some other layers.

We further apply activation changes formed from the average directions and magnitudes on other tasks to new tasks in Fig.~\ref{fig:additional_D2_2}(b).
Both the original and reversed communication directions produce positive effects in correction and echo chambers.
The original direction produces a larger effect for echo chambers, whereas the reversed direction produces a larger effect for correction.
\begin{figure}[!t]
\centering
\begin{minipage}[t]{0.48\textwidth}
\centering
\includegraphics[width=\linewidth]{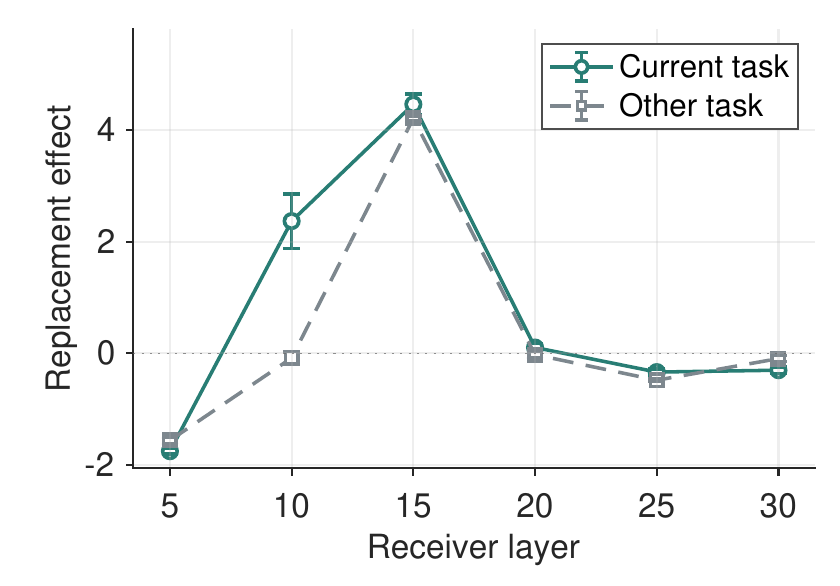}\par
{\footnotesize (a)\par}
\end{minipage}\hfill
\begin{minipage}[t]{0.48\textwidth}
\centering
\includegraphics[width=\linewidth]{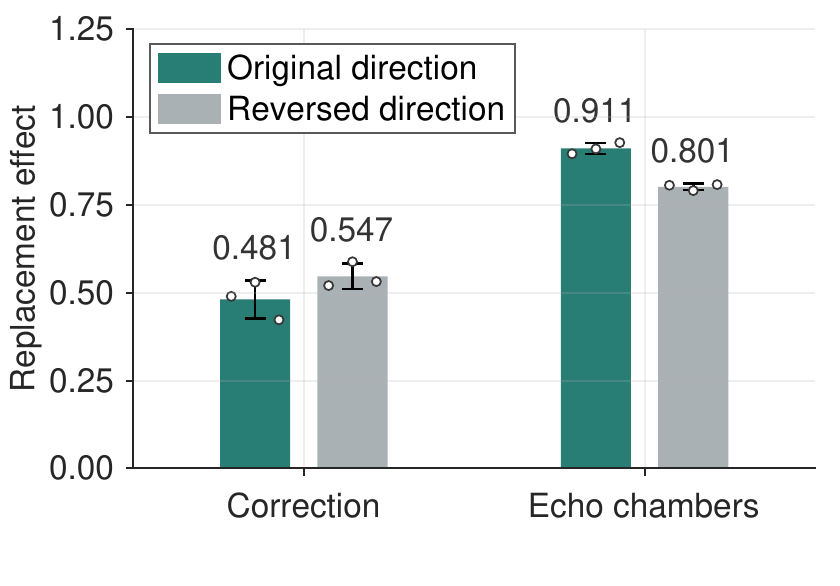}\par
{\footnotesize (b)\par}
\end{minipage}
\caption{Receiver activation replacement on Gemma-3-4B-it: (a) current-task and other-task changes, with one layer replaced per run; (b) average changes from other tasks, applied across six layers on new tasks. Values use complete-answer scores.
Circles in (b) show results of different seeds.}
\label{fig:additional_D2_2}
\end{figure}

\paragraph{Message-specific replacement.}
We compare receiver activation changes induced by the original message with those induced by another candidate message for the same 2Wiki task in Fig.~\ref{fig:additional_D2_3}.
We calculate the mean absolute error between the activation-replacement effect and the original message effect on the correct answer's score, using the same samples for both comparisons.
The original message's activation changes achieve lower error on both models: 0.187 versus 1.280 on Gemma-3-4B and 0.121 versus 1.039 on Qwen3-4B.
These results show that receiver activation changes distinguish the effects of different messages on the same task.

\begin{figure}[!t]
\centering
\includegraphics[width=0.5\linewidth]{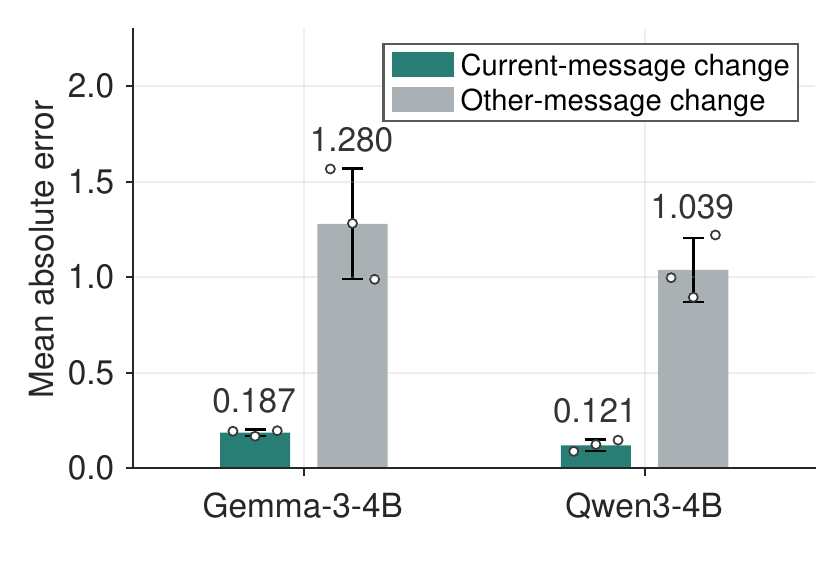}
\caption{Message-specific activation replacement on 2Wiki. Mean absolute error between message and replacement effects on the correct answer's score, evaluated on matched samples.}
\label{fig:additional_D2_3}
\end{figure}

\paragraph{Receiver answers after activation replacement.}
We present answer agreement and recovery in Table~\ref{tab:activation-answer-recovery}.
Answer agreement measures the proportion of runs where activation replacement selects the same answer as the original-message run.
For answer recovery, we calculate this proportion among runs where the original and reference messages lead to different answers.
On 2Wiki, the original message's activation changes achieve answer recovery of 97.81\% on Gemma-3-4B and 96.73\% on Qwen3-4B, compared with 41.07\% and 55.06\% using another message's changes.
On Correlated-Trap, activation replacement achieves higher answer recovery than Swapped-role changes across all three models.
The recovery rates range from 1.33\% to 16.39\% for activation replacement and from 0.00\% to 0.27\% for Swapped-role changes.

\begin{table}[!t]
\centering
\caption{Receiver answers after activation replacement, reported as mean $\pm$ sample standard deviation. Agreement uses the same matched message records for both replacements: 1,268 for Gemma and 1,334 for Qwen on 2Wiki, and 1,440 per model on Correlated-Trap. Recovery uses only records where the message changes the answer, including both corrections and changes to incorrect answers. Bold marks the higher mean recovery within each model and dataset.}
\label{tab:activation-answer-recovery}
\small
\setlength{\tabcolsep}{5pt}
\renewcommand{\arraystretch}{1.05}
\begin{tabular}{llcc}
\toprule
Model & Activation changes & \shortstack{Answer agreement\\(\%)} & \shortstack{Answer recovery\\(\%)} \\
\midrule
\multicolumn{4}{l}{\textbf{(a) 2Wiki}} \\
Gemma-3-4B & Original message & $99.61\pm0.27$ & $\mathbf{97.81}\pm0.23$ \\ 
& Another message (same task) & $88.89\pm1.88$ & $41.07\pm13.20$ \\
Qwen3-4B & Original message & $99.78\pm0.00$ & $\mathbf{96.73}\pm0.22$ \\ 
& Another message (same task) & $94.23\pm1.48$ & $55.06\pm11.67$ \\
\midrule
\multicolumn{4}{l}{\textbf{(b) Correlated-Trap}} \\
Gemma-3-4B & Original message & $23.68\pm2.68$ & $\mathbf{16.39}\pm1.67$ \\ 
& Swapped-role changes & $6.74\pm1.15$ & $0.23\pm0.00$ \\
Qwen3-4B & Original message & $18.47\pm1.07$ & $\mathbf{1.33}\pm0.92$ \\ 
& Swapped-role changes & $14.51\pm1.05$ & $0.27\pm0.27$ \\
Qwen3-8B & Original message & $40.90\pm1.89$ & $\mathbf{4.88}\pm1.54$ \\ 
& Swapped-role changes & $30.49\pm1.73$ & $0.00\pm0.00$ \\
\bottomrule
\end{tabular}
\end{table}

\FloatBarrier
\subsection{CGD message selection}\label{sec:additional-selection} \paragraph{Selection with the same candidate messages.}
We compare CGD with the message-selection baselines in Fig.~\ref{fig:additional_D3_1}, using the same candidate messages and recorded receiver answers.
Each model is evaluated on 240 2Wiki tasks, with answers scored over the same two options.
Results show that CGD achieves higher average accuracy than all three baselines on all three models.
When compared with Activation products, CGD improves accuracy by 1.67\%, 1.94\%, and 2.36\% on Gemma-3-4B-it, Qwen3-4B, and Qwen3-8B, respectively.
The paired results in Fig.~\ref{fig:additional_D3_1}(b) show that CGD corrects more baseline errors than it introduces in every comparison.
The comparison with Activation products shows that learning directly from receiver activation changes improves message selection.

\begin{figure}[!t]
\centering
\begin{minipage}[t]{0.48\textwidth}
\centering
\includegraphics[width=\linewidth]{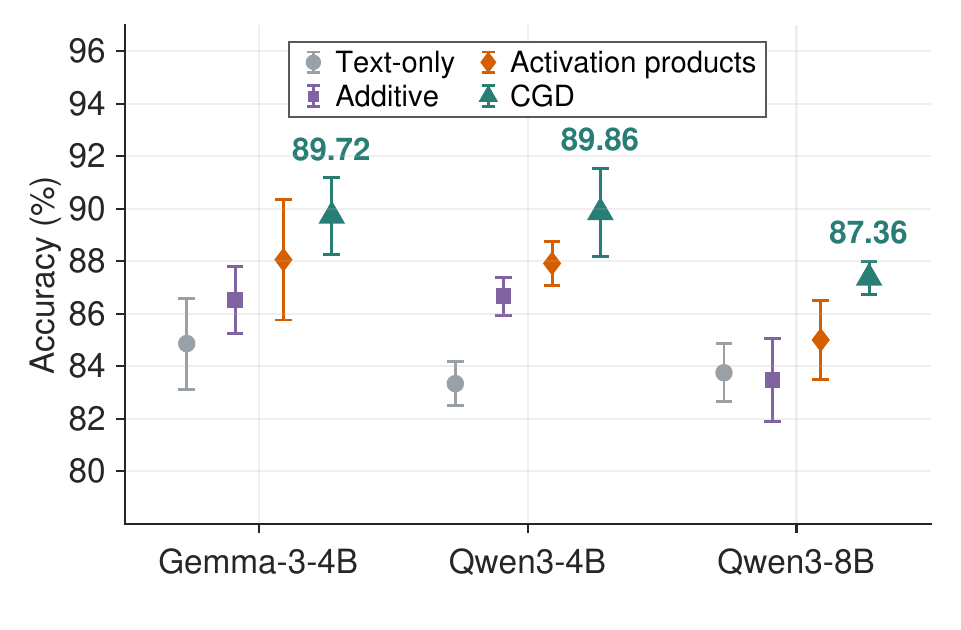}\par
{\footnotesize (a)\par}
\end{minipage}\hfill
\begin{minipage}[t]{0.48\textwidth}
\centering
\includegraphics[width=\linewidth]{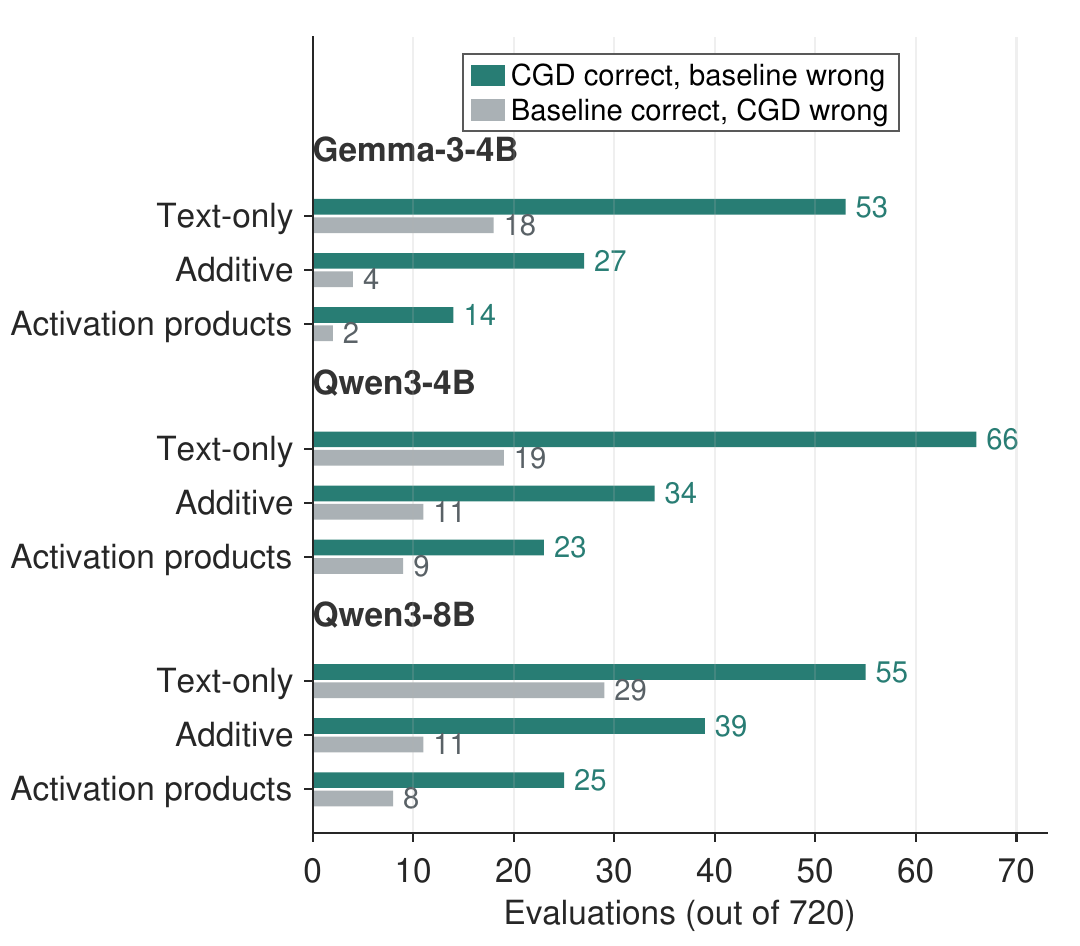}\par
{\footnotesize (b)\par}
\end{minipage}
\caption{Message selection on 2Wiki with the same candidate messages and two-option scoring. (a) Accuracy across baselines. (b) Cases where CGD is correct and the baseline is wrong, or vice versa.}
\label{fig:additional_D3_1}
\end{figure}

\paragraph{Fixed message strategies.}
We compare CGD with sending no message, either sender's message, or both messages in Table~\ref{tab:cgd-fixed-messages}.
All strategies use the same messages and receiver settings.
CGD achieves the highest average accuracy across all three models.
Compared with sending both messages, it corrects 50, 65, and 60 errors while introducing 18, 13, and 27 errors on Gemma-3-4B, Qwen3-4B, and Qwen3-8B, respectively.
Thus, selecting messages improves accuracy over directly sending all messages.

\begin{table}[!t]
\centering
\caption{Two-answer accuracy (\%) on 240 new 2Wiki tasks for the 4B models and the original 240 tasks for Qwen3-8B, reported as mean $\pm$ sample standard deviation. CGD achieves the highest average accuracy across all three models.}
\label{tab:cgd-fixed-messages}
\small
\renewcommand{\scms}[2]{#1{\scriptsize$\pm#2$}}
\setlength{\tabcolsep}{4pt}
\renewcommand{\arraystretch}{1.12}
\begin{tabular*}{\linewidth}{@{\extracolsep{\fill}}lccc@{}}
\toprule
Strategy & Gemma-3-4B & Qwen3-4B & Qwen3-8B \\
\midrule
No message & \scms{84.17}{0.00} & \scms{78.75}{0.00} & \scms{81.67}{0.00} \\
Sender 1 only & \scms{86.11}{1.27} & \scms{81.67}{1.10} & \scms{84.58}{2.17} \\
Sender 2 only & \scms{82.50}{1.91} & \scms{79.58}{1.44} & \scms{77.50}{0.83} \\
Both messages & \scms{85.28}{1.34} & \scms{83.89}{1.34} & \scms{82.78}{1.97} \\
\midrule
\rowcolor{CGDshade}
\textbf{CGD} & \scms{\textbf{89.72}}{1.20} & \scms{\textbf{91.11}}{0.64} & \scms{\textbf{87.36}}{0.64} \\
\bottomrule
\end{tabular*}
\end{table}

\paragraph{Activation products.}
We match the input scaling and weight penalty of Activation products to CGD, as described in Appendix~\ref{app:protocol-cgd}.
Table~\ref{tab:cgd-matched-products} shows that CGD retains higher average accuracy across all three models.
On the same tasks and seeds, CGD answers 3, 8, and 3 more cases correctly on Gemma-3-4B, Qwen3-4B, and Qwen3-8B, respectively.
These results support learning directly from receiver activation changes to select messages.
\begin{table}[!t]
\centering
\caption{Two-answer accuracy (\%) on the original 240 2Wiki tasks for each model, before and after matching the scaling and penalty to CGD. Entries show the mean and sample standard deviation across three seeds.}
\label{tab:cgd-matched-products}
\small
\renewcommand{\scms}[2]{#1{\scriptsize$\pm#2$}}
\setlength{\tabcolsep}{4pt}
\renewcommand{\arraystretch}{1.12}
\begin{tabular*}{\linewidth}{@{\extracolsep{\fill}}lccc@{}}
\toprule
Approach & Gemma-3-4B & Qwen3-4B & Qwen3-8B \\
\midrule
Activation products & \scms{88.06}{2.29} & \scms{87.92}{0.83} & \scms{85.00}{1.50} \\
Activation products (adjusted) & \scms{89.31}{1.05} & \scms{88.75}{1.50} & \scms{86.94}{1.27} \\
\midrule
\rowcolor{CGDshade}
\textbf{CGD} & \scms{\textbf{89.72}}{1.46} & \scms{\textbf{89.86}}{1.68} & \scms{\textbf{87.36}}{0.64} \\
\bottomrule
\end{tabular*}
\end{table}

\paragraph{Evaluation on new 2Wiki tasks.}
We evaluate CGD on another 240 2Wiki tasks without retraining and present the results in Fig.~\ref{fig:additional_D3_2}.
Results show that CGD achieves the highest average accuracy among all baseline schemes on both 4B models.
On Gemma-3-4B-it and Qwen3-4B, it improves over Additive by 2.50 \% and 3.75\%, and over Activation products by 1.11\% and 0.83\%, respectively.
These results show that CGD maintains its accuracy advantage on new 2Wiki tasks without retraining.

\begin{figure}[!t]
\centering
\includegraphics[width=0.55\linewidth]{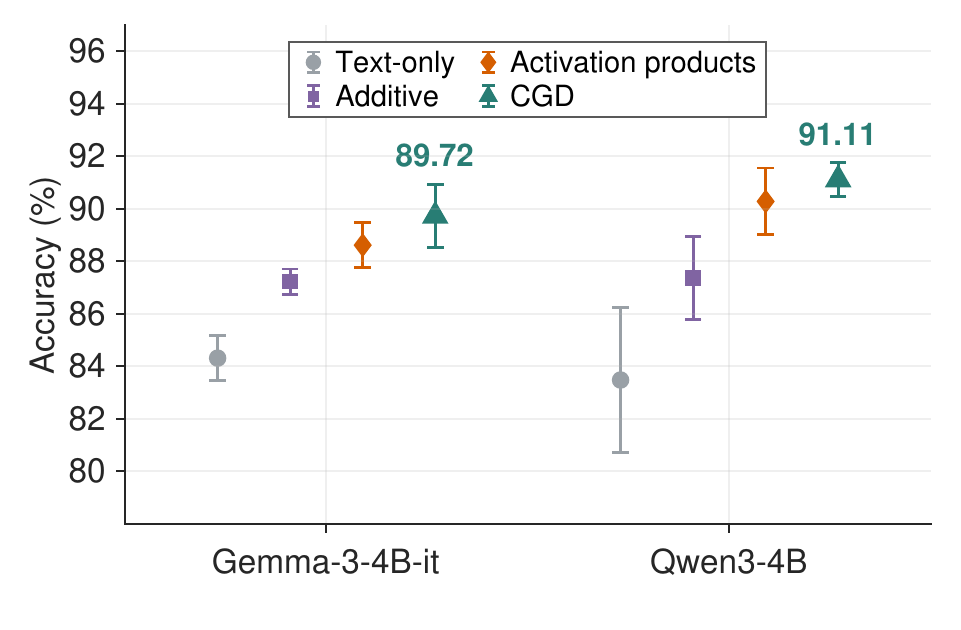}
\caption{Message selection on new 2Wiki tasks without retraining. Points and error bars show average accuracy and standard deviation.}
\label{fig:additional_D3_2}
\end{figure}

\FloatBarrier
\subsection{Free-generation task performance}\label{sec:additional-performance}
\paragraph{Receiver-summary inputs.}
We compare message-selection inputs for free generation on 2Wiki in Table~\ref{tab:additional-receiver-inputs}.
For each nonempty message combination, Receiver-only uses the receiver's activation summary without candidate messages and the average of its activation summaries for the other nonempty message combinations.
Text + receiver also includes the text of the messages in the combination being scored.
Results show that CGD achieves 61.53\% accuracy on Qwen3-8B, exceeding all five input baselines and improving on the strongest, Receiver-only, by 1.11\%.
On Gemma-3-4B, CGD exceeds Receiver-only and Text + receiver by 2.22\% and 4.44\%, while Activation products have the highest average accuracy.
Thus, CGD improves on receiver-summary inputs alone on both models.

\begin{table}[!t]
\centering
\caption{Free-generation accuracy (\%) on 2Wiki, reported as mean $\pm$ sample standard deviation. The Qwen3-8B CGD entry uses a separate run with the same evaluation tasks, seeds, and scoring. Bold marks the highest mean in each column.}
\label{tab:additional-receiver-inputs}
\small
\setlength{\tabcolsep}{12pt}
\renewcommand{\arraystretch}{1.05}
\begin{tabular}{lcc}
\toprule
Approach & Gemma-3-4B & Qwen3-8B \\
\midrule
Text-only & $38.33\pm1.67$ & $56.67\pm1.10$ \\
Receiver-only & $40.69\pm1.58$ & $60.42\pm0.42$ \\
Text + receiver & $38.47\pm1.05$ & $59.17\pm0.72$ \\
Additive & $42.78\pm1.27$ & $53.61\pm1.05$ \\
Activation products & $\mathbf{43.33}\pm1.91$ & $57.36\pm0.24$ \\
\midrule
CGD & $42.92\pm0.83$ & $\mathbf{61.53}\pm0.48$ \\
\bottomrule
\end{tabular}
\end{table}
\paragraph{Comparison with collaboration approaches.}
We present GSM8K free-generation accuracy with Qwen3-8B in Fig.~\ref{fig:additional_D4_1}.
To compare accuracy on the same tasks and seeds, we use the runs completed by CGD and all baselines when the token-limited comparison was set up.
We calculate accuracy over these 555 paired runs across 240 GSM8K tasks.
We limit each baseline's total generated tokens to the number generated by CGD for the same task and seed.
For approaches that generate messages before answering, we first reserve enough tokens to match the length of CGD’s final answer. 
We divide the remaining tokens among the message-generation steps. Each step stops at its token limit.
Results show that CGD achieves higher free-generation accuracy than all eight baselines when all approaches have the same limit on generated tokens.

\paragraph{Paired answer comparison.}
To explain the accuracy differences, we present paired results from the same runs in Fig.~\ref{fig:additional_D4_2}.
CGD is correct and Majority vote is wrong on 11 runs, whereas Majority vote is correct and CGD is wrong on 8 runs.
Under the same token limit, runs where only CGD is correct outnumber runs where only the baseline is correct for each of the eight baselines.
\begin{figure}[!t]
\centering
\begin{minipage}[t]{0.48\textwidth}
\centering
\includegraphics[width=0.8\linewidth]{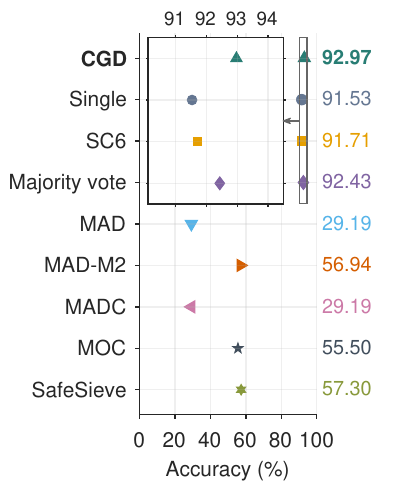}
\caption{Comparison of the free-generation accuracy on Qwen3-8B and GSM8K, calculated over 555 paired runs across 240 tasks, with 195, 150, and 210 runs for seeds 0, 1, and 2, respectively.}
\label{fig:additional_D4_1}
\end{minipage}\hfill
\begin{minipage}[t]{0.48\textwidth}
\centering
\includegraphics[width=0.8\linewidth]{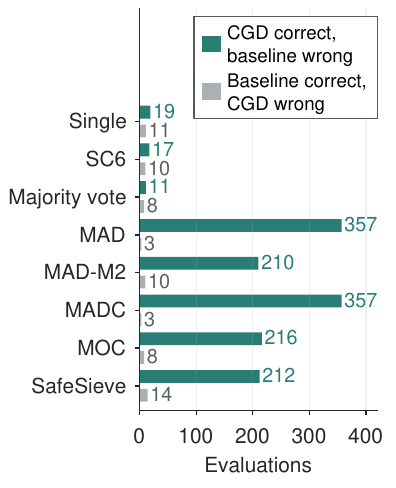}
\caption{Paired free-generation answers on GSM8K with Qwen3-8B, using the same 555 runs and token limits as Fig.~\ref{fig:additional_D4_1}. Bars count runs where CGD is correct and the baseline is wrong, and vice versa.}
\label{fig:additional_D4_2}
\end{minipage}
\end{figure}

\begin{table}[!t]
\centering
\caption{2Wiki accuracy (\%) using A/B answer scores. We report the mean and sample standard deviation across three seeds.}
\label{tab:2wiki-answer-scoring}
\small
\renewcommand{\scms}[2]{#1{\scriptsize$\pm#2$}}
\setlength{\tabcolsep}{5pt}
\renewcommand{\arraystretch}{1.12}
\begin{tabular*}{\linewidth}{@{\extracolsep{\fill}}lccc@{}}
\toprule
Approaches & Gemma-3-4B & Qwen3-4B & Qwen3-8B \\
\midrule
Single & \scms{78.75}{0.00} & \scms{75.83}{0.00} & \scms{75.42}{0.00} \\
MAD & \scms{81.11}{1.05} & \scms{81.39}{1.73} & \scms{81.94}{2.93} \\
MAD-M2 & \scms{82.92}{1.50} & \scms{81.53}{1.68} & \scms{78.47}{2.77} \\
MADC & \scms{79.72}{1.46} & \scms{79.03}{1.34} & \scms{81.53}{1.27} \\
MOC & \scms{86.81}{1.20} & \scms{86.94}{0.87} & \scms{85.69}{0.24} \\
SafeSieve & \scms{80.00}{2.60} & \scms{84.17}{0.72} & \scms{81.94}{2.10} \\
\midrule
\rowcolor{CGDshade}
\textbf{CGD} & \scms{\textbf{89.72}}{1.46} & \scms{\textbf{89.86}}{1.68} & \scms{\textbf{87.36}}{0.64} \\
\bottomrule
\end{tabular*}
\end{table}
\FloatBarrier
\begingroup
\setlength{\textfloatsep}{2pt plus 1pt minus 1pt}
\begin{figure}[!t]
\centering
\begin{tabular}{@{}c@{\hspace{0.02\linewidth}}c@{}}
\includegraphics[width=0.49\linewidth]{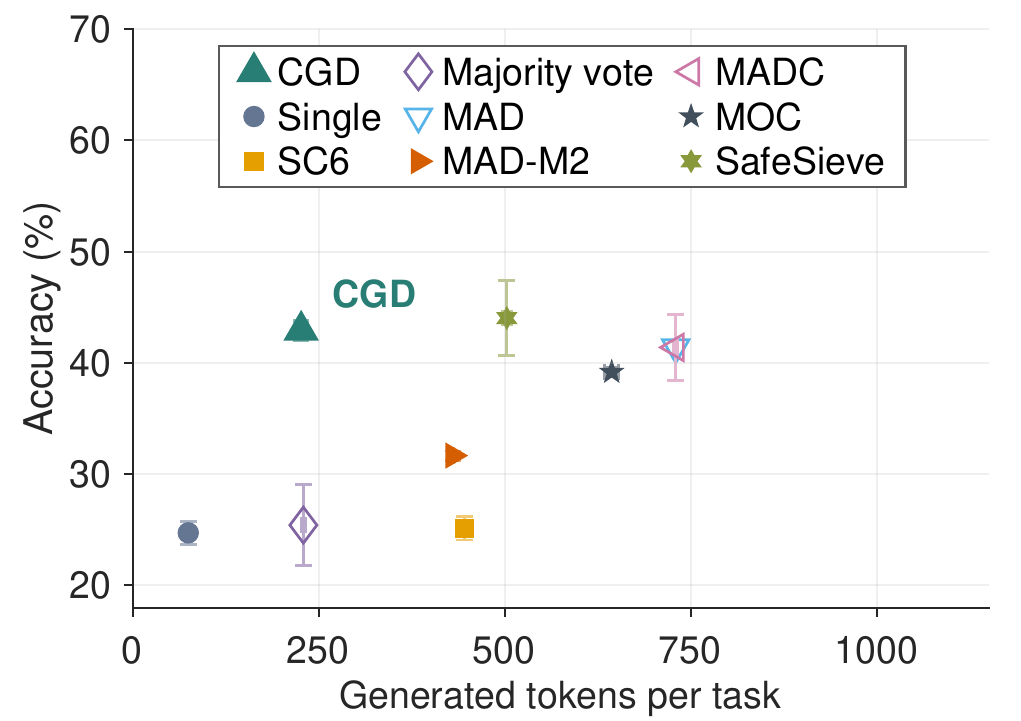} &
\includegraphics[width=0.49\linewidth]{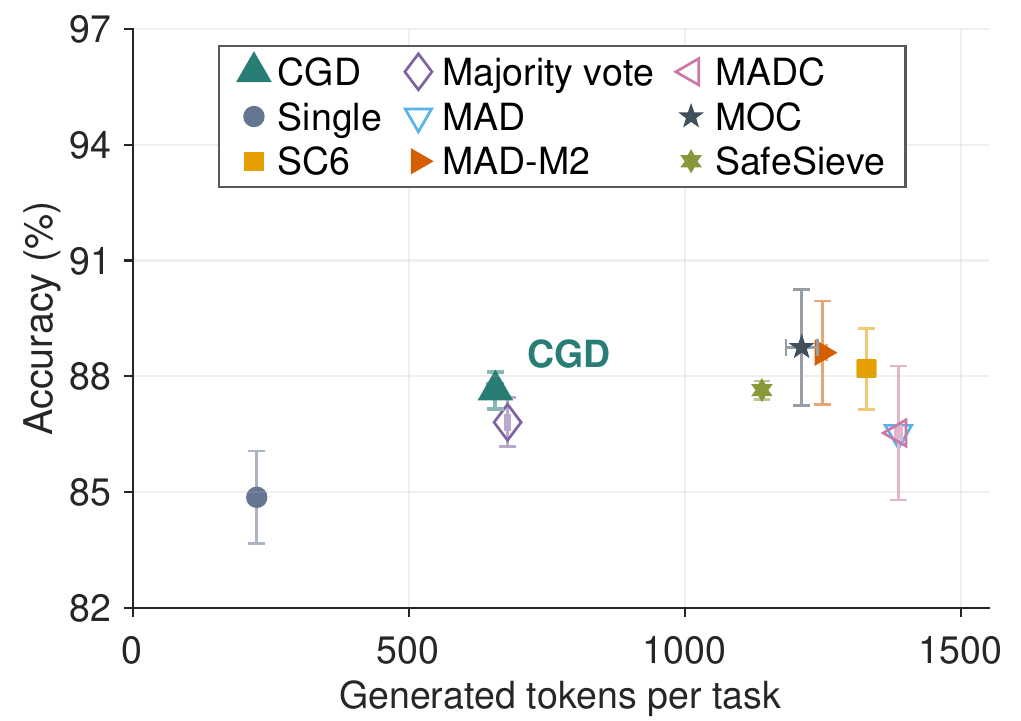} \\
{\footnotesize (a) Gemma-3-4B / 2Wiki} & {\footnotesize (b) Gemma-3-4B / GSM8K} \\[0pt]
\includegraphics[width=0.49\linewidth]{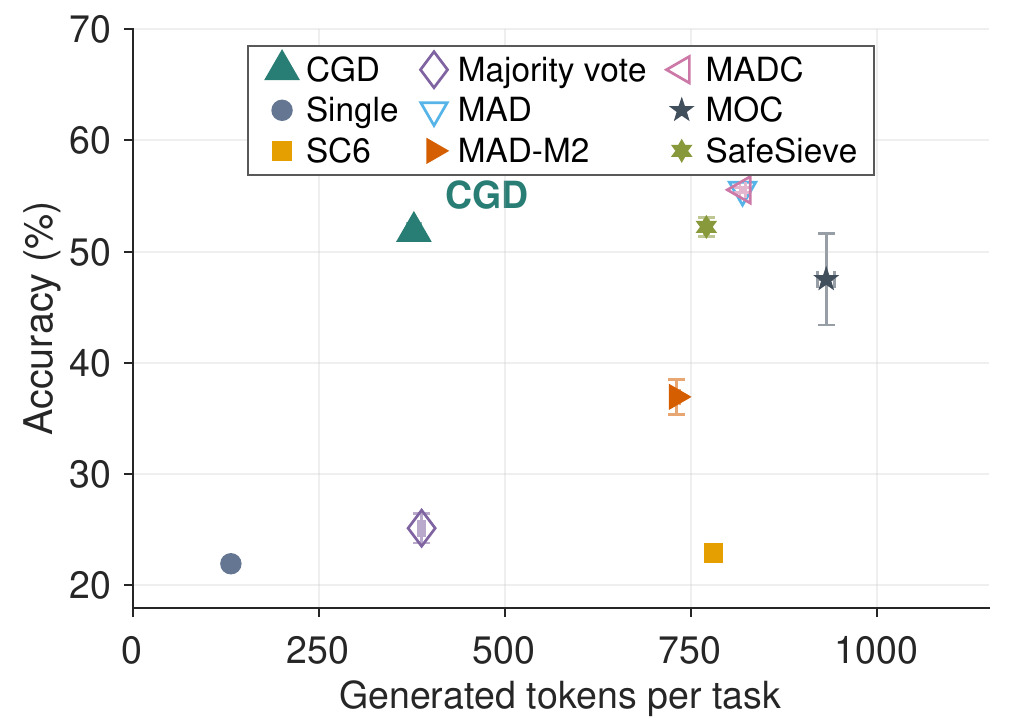} &
\includegraphics[width=0.49\linewidth]{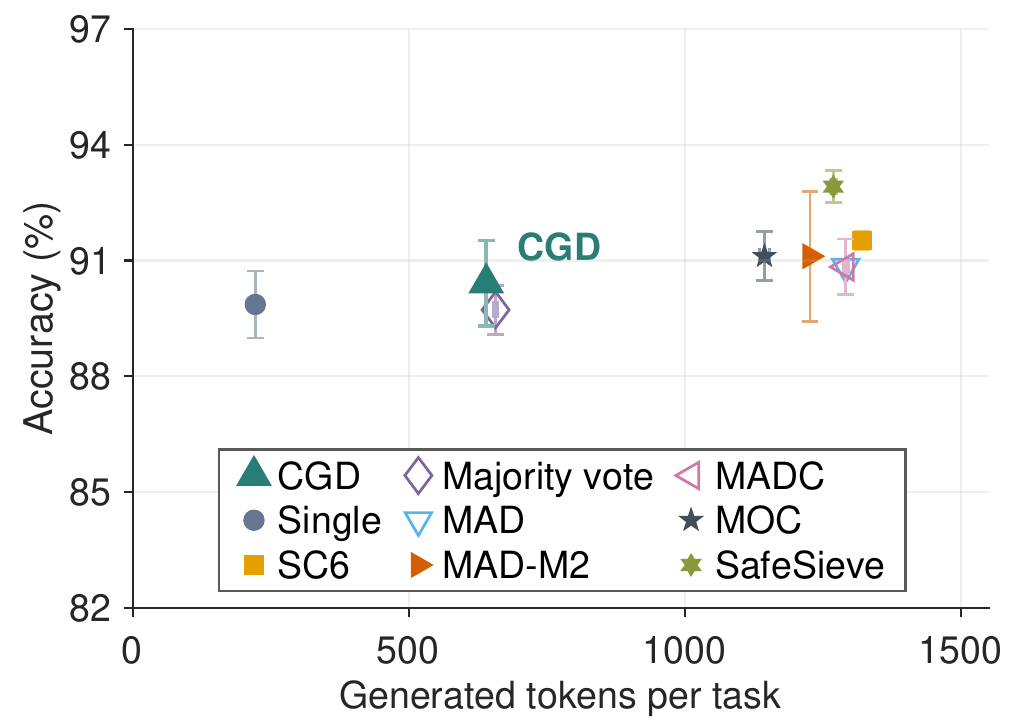} \\
{\footnotesize (c) Qwen3-4B / 2Wiki} & {\footnotesize (d) Qwen3-4B / GSM8K} \\[0pt]
\includegraphics[width=0.49\linewidth]{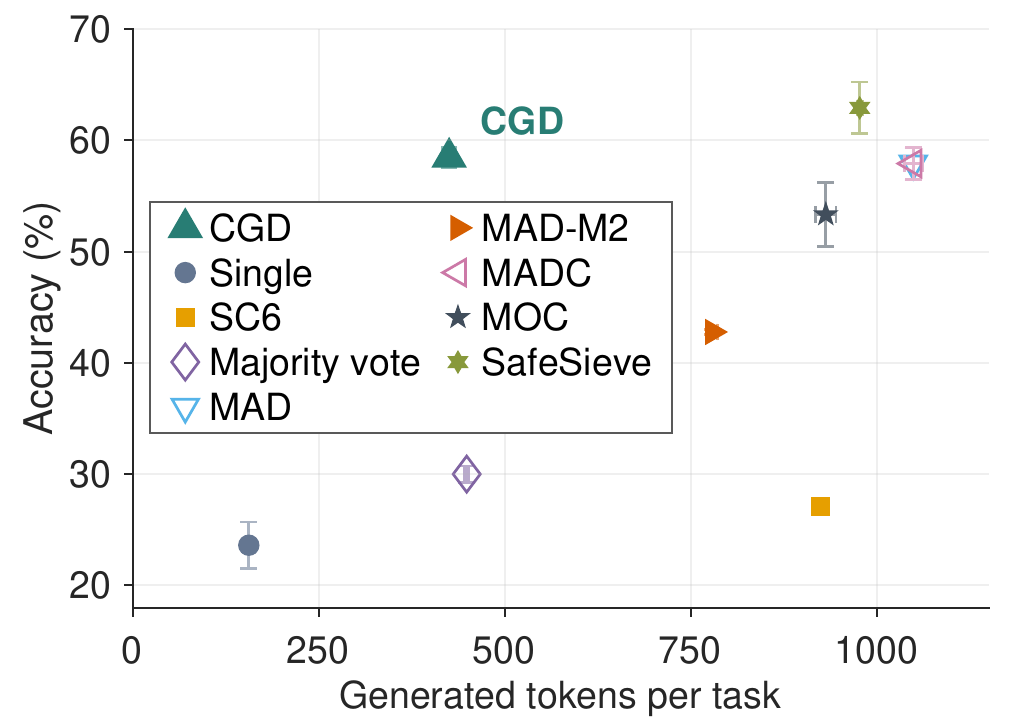} &
\includegraphics[width=0.49\linewidth]{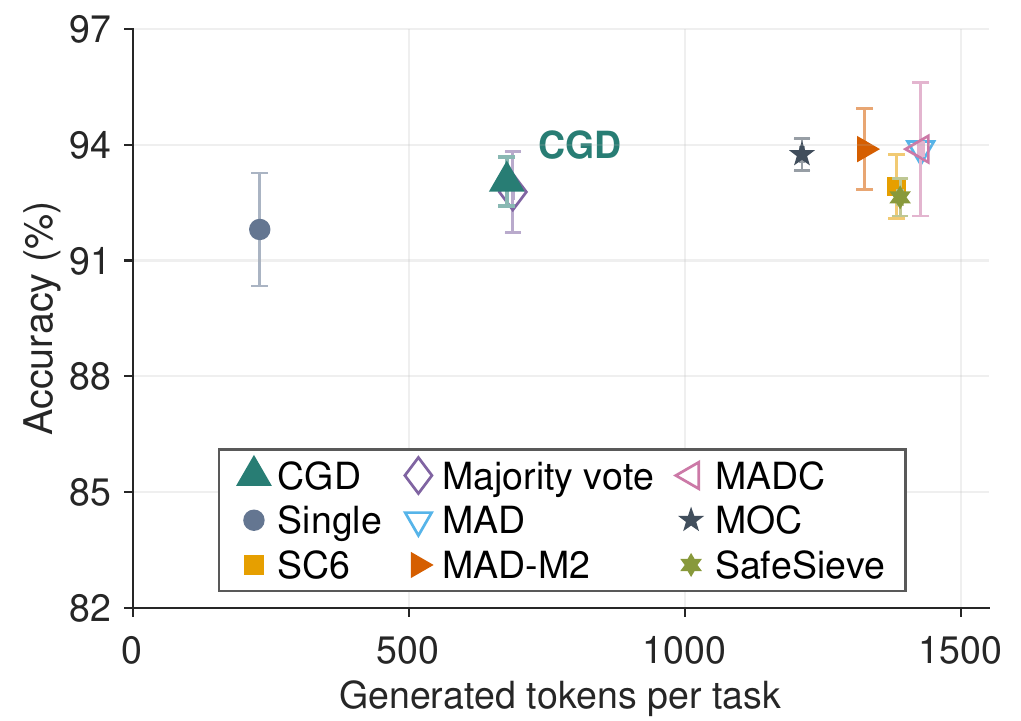} \\
{\footnotesize (e) Qwen3-8B / 2Wiki} & {\footnotesize (f) Qwen3-8B / GSM8K}
\end{tabular}
\caption{Free-generation accuracy and generated tokens over 720 runs per model and dataset. Points show means and error bars show standard deviations across seeds. MAD and MADC have overlapping results.}
\label{fig:additional_D5_1}
\end{figure}
\subsection{Accuracy and generation cost}\label{sec:additional-costs}
\paragraph{Two-answer comparisons.}
We compare collaboration approaches using the A/B answer scores in Appendix~\ref{app:protocol-scoring}.
Table~\ref{tab:2wiki-answer-scoring} shows that CGD exceeds the strongest baseline by 2.92\%, 2.92\%, and 1.67\% on Gemma-3-4B, Qwen3-4B, and Qwen3-8B, respectively.
\paragraph{Accuracy and generated tokens.}
To compare accuracy and generation cost across models, we present the original 2Wiki and GSM8K runs with complete token records in Fig.~\ref{fig:additional_D5_1}.
For CGD, this figure uses a separate set of runs from Table~\ref{tab:four-dataset-accuracy-tokens}, with accuracy and token counts calculated from the same recorded answers.
We evaluate 240 tasks per dataset using each approach's original generation settings and record generated tokens across all calls, including messages, intermediate responses, and final answers.
Across all six settings, CGD achieves higher average accuracy than every baseline that generates fewer tokens on average.
\FloatBarrier
\endgroup

\clearpage
\subsection{Case study of message selection}\label{sec:additional-examples} \paragraph{Correcting reasoning with selected messages.}
To show how the selected messages affect the receiver's answer, we present a GSM8K example in Fig.~\ref{fig:additional_D6_1}.
Agent 1 fails to subtract the bees that have already returned, while Agent 2 omits the later departures from its final calculation.
CGD delivers both messages, and Agent 3 calculates the correct answer, 75.
With no message or just one message, Agent 3 answers incorrectly.
This example shows that messages ending in incorrect answers can still improve the receiver's answer when delivered in combination.

\begin{figure}[!t]
\centering
\includegraphics[width=\linewidth]{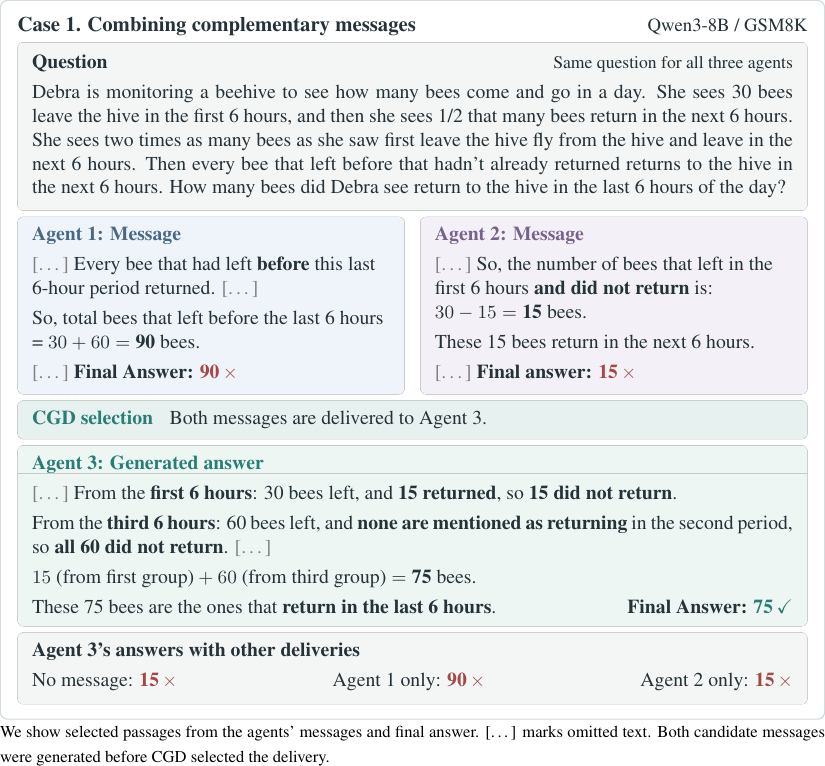}
\caption{A GSM8K example with Qwen3-8B in which joint message delivery corrects the receiver's answer.}
\label{fig:additional_D6_1}
\end{figure}
\FloatBarrier

\clearpage
Fig.~\ref{fig:additional_D6_3} shows a different correction on GSM8K with Qwen3-4B.
Both senders calculate travel times of 36 and 40 minutes, but only Agent 1 subtracts them to obtain the waiting time.
CGD delivers this message, and the receiver answers 4 instead of 40.
The selected message supplies the final calculation missing from the receiver's answer without communication.
\begin{figure}[!t]
\centering
\includegraphics[width=\linewidth]{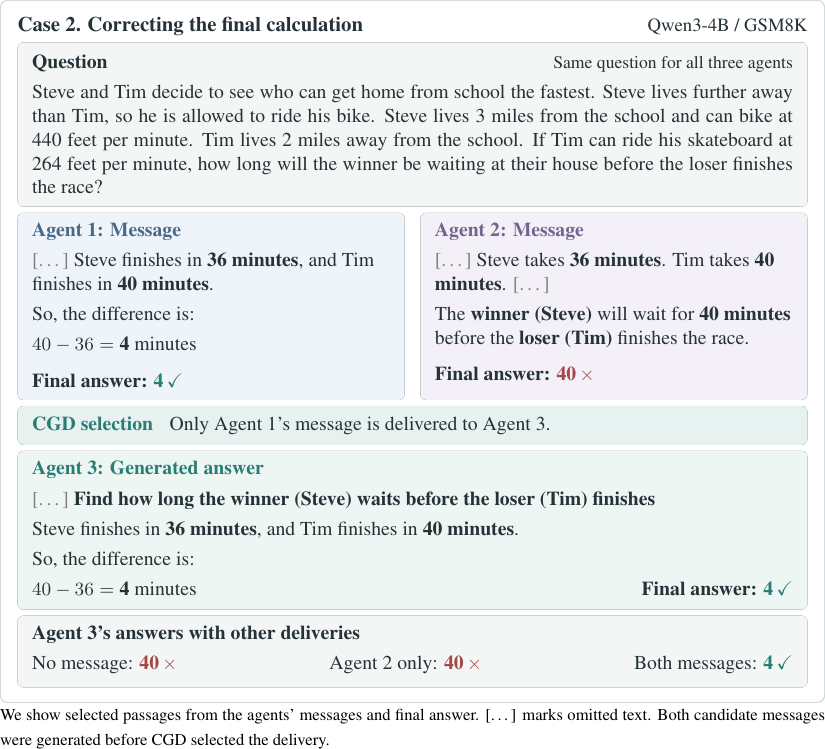}
\caption{A GSM8K example with Qwen3-4B in which the selected message supplies the receiver's missing calculation.}
\label{fig:additional_D6_3}
\end{figure}
\FloatBarrier
\clearpage
\paragraph{Excluding misleading information.}
We present a 2Wiki example in Fig.~\ref{fig:additional_D6_2} where excluding one message corrects the receiver's answer.
Agent 1 mistakes Robert Stewart for Frances Vane's husband, while Agent 2 identifies Charles Vane and his place of death, London.
CGD delivers only Agent 2's message, and Agent 3 answers London correctly.
Delivering both messages instead leads to the incorrect answer, 1864.
This example shows that adding a misleading message can turn a correct answer into an incorrect one.

\begin{figure}[!t]
\centering
\includegraphics[width=\linewidth]{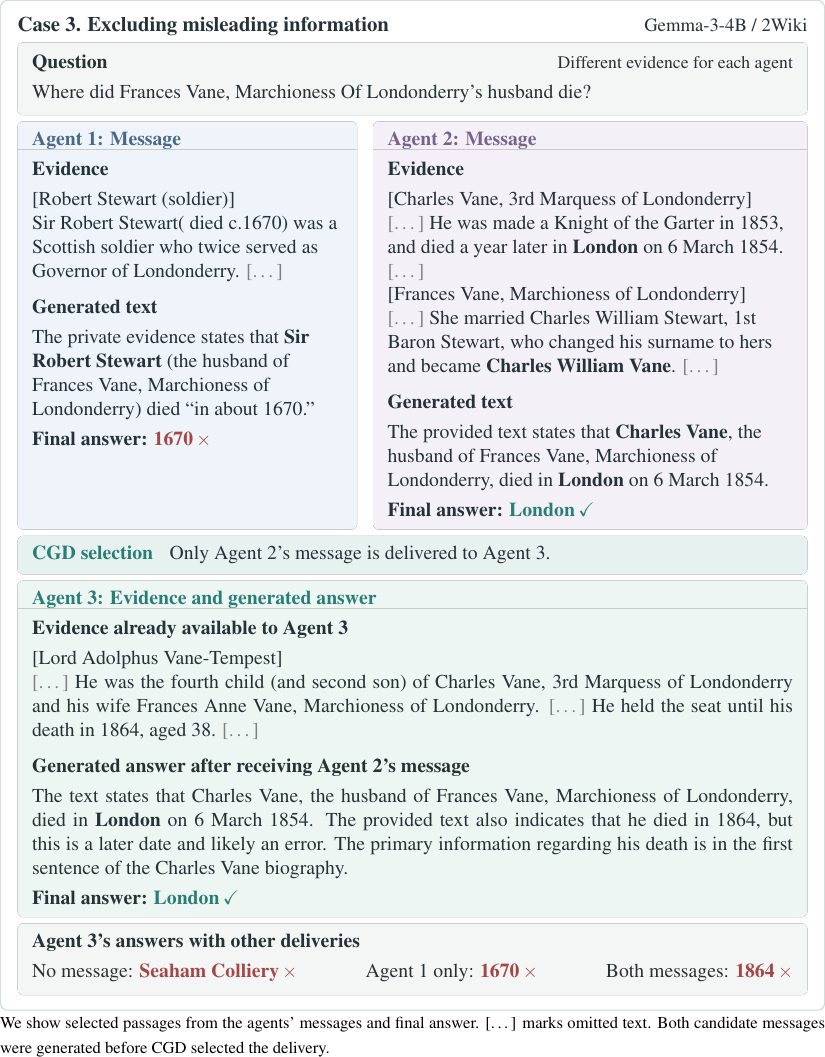}
\caption{A 2Wiki example with Gemma-3-4B in which excluding a misleading message corrects the receiver's answer.}
\label{fig:additional_D6_2}
\end{figure}
\FloatBarrier

\clearpage
Fig.~\ref{fig:additional_D6_4} illustrates a shared calculation error on GSM8K with Gemma-3-4B.
Both senders answer 6, and delivering both messages leaves the receiver with the same error.
CGD delivers only Agent 1's message, after which the receiver restores the volume before cooking and answers 12.
Agreement between senders does not ensure that delivering both messages improves the receiver's answer.
\begin{figure}[!t]
\centering
\includegraphics[width=\linewidth]{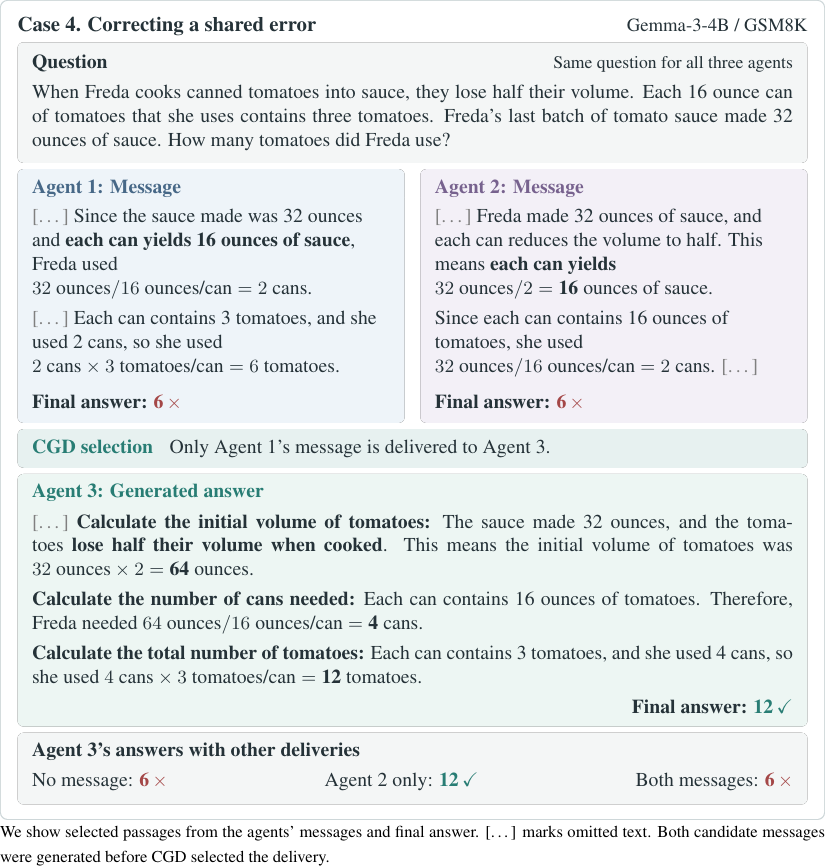}
\caption{A GSM8K example with Gemma-3-4B in which single-message delivery corrects an error shared by both senders.}
\label{fig:additional_D6_4}
\end{figure}
\FloatBarrier

\section{Discussions}
\label{app:limitations}
Our activation replacements span several layers and prompt positions. Identifying a smaller set of activations that preserves the measured message effects remains open.
CGD selects among messages that have already been generated and reads receiver activations for each available message combination. Its running cost includes message generation, these activation reads, and final-answer generation. Scoring every message combination may become expensive as the number of messages grows.

\end{document}